%% file: dycausal.tex
\documentclass[10pt]{article} 
\usepackage[preprint]{tmlr}

\input{math_commands.tex}

\usepackage{hyperref}
\usepackage{url}
\usepackage{algorithm}
\usepackage{algorithmic}
\usepackage{mathtools}
\usepackage{amsmath}
\usepackage{amssymb}
\usepackage{amsthm}
\usepackage{url}
\usepackage{graphicx}
\usepackage{newtxmath}
\usepackage{bm}
\usepackage{enumitem}
\usepackage{tablefootnote}
\usepackage{bbding}  
\usepackage{pifont}

\usepackage{booktabs}

\usepackage{float}
\usepackage{makecell}
\usepackage{xcolor}
\usepackage{subfig}
\newtheorem{theorem}{Theorem}
\newtheorem{lemma}{Lemma}

\newtheorem{assumption}{Assumption}
\newtheorem{proposition}{Proposition}
\newtheorem{corollary}{Corollary}
\def\bbP{\mathbb{P}}
\def\sfp{\mathsf{f}}
\def\sfs{\mathsf{s}}

\title{Score-matching-based Structure Learning for Temporal Data on Networks}

\author{\name Hao Chen \email chen\_hao1@sjtu.edu.cn \\
      \addr School of Mathematical Sciences, Shanghai Jiao Tong University
      \and
      \name Kai Yi \email kyi@mrclmb.cam.ac.uk \\
      \addr University of Cambridge}

\def\month{MM}  
\def\year{YYYY} 
\def\openreview{\url{https://openreview.net/forum?id=XXXX}} 

\begin{document}

\maketitle

\begin{abstract}
Causal discovery is a crucial first step in establishing causality from empirical data and background knowledge. Numerous algorithms have been developed for this purpose. Among them, the score-matching method has demonstrated superior performance across various evaluation metrics, particularly for the commonly encountered Additive Nonlinear Causal Models. However, current score-matching-based algorithms are primarily designed to analyze independent and identically distributed (i.i.d.) data. More importantly, they suffer from high computational complexity due to the pruning step required for handling dense Directed Acyclic Graphs (DAGs). To enhance the scalability of score matching, we have developed a new parent-finding subroutine for leaf nodes in DAGs, significantly accelerating the most time-consuming part of the process: the pruning step. This improvement results in an efficiency-lifted score matching algorithm, termed Parent Identification-based Causal structure learning for both i.i.d. and temporal data on networKs, or PICK. The new score-matching algorithm extends the scope of existing algorithms and can handle static and temporal data on networks with weak network interference. Our proposed algorithm can efficiently cope with increasingly complex datasets that exhibit spatial and temporal dependencies, commonly encountered in academia and industry. The proposed algorithm can accelerate score-matching-based methods while maintaining high accuracy in real-world applications.
\end{abstract}

\section{Introduction}
\label{sec:intro}

Causal discovery is the problem of identifying causal relationship from empirical data, with diverse applications spanning fields such as transcription regulation in genomics \citep{maathuis2010predicting}, systems biology \citep{sachs2005causal}, and e-commerce \citep{sharma2015estimating}.  
Statistically, the main objective of causal discovery is to infer the underlying Directed Acyclic Graph (DAG) structure from empirical data, whether observational, interventional, or a combination of both \citep{li2023causal, spirtes2000causation, spirtes2001anytime, jaber2020causal}. Classical causal discovery methods can be roughly divided into constraint-based and score-based methods. Recently, a method based on the idea of ``score matching'' \citep{score-matching} was developed. A dense DAG is first generated by score-matching and then pruned by regularized regression techniques \citep{CAM-method}. The pruning step needs to detect parents of leaf nodes in a causal DAG, but this step can be computationally expensive. In this paper, we develop a novel algorithm based on the idea of variance comparison, which will be made clear later. It turns out that the method can significantly reduce the computational cost while maintaining high accuracy, in particular when the structural equation model associated with the causal DAG is nonlinear.

 
To discover the DAG structure using score-matching, one first estimates the score function by using Stein's lemma (aka Gaussian integration-by-parts) \cite{stein1972bound}, and then produces a dense DAG whose leaf nodes are identified by comparing variances of the corresponding scores. The resulting dense graph is then pruned by CAM pruning developed in \citet{CAM-method}.
The pruning step, however, is a computational bottleneck as its time complexity is of the cubic order in the number of nodes \citep{montagna2023scalable}.

To improve the computational efficiency of state-of-the-art score-matching methods, we develop a new algorithm, called PICK, for Parent Identification-based Causal structure learning for both i.i.d. and temporal data on networKs. 
In PICK, we develop a parent identifying subroutine that can significantly alleviate the computational burden of many existing causal discovery algorithms. More specifically, PICK can learn causal DAGs within non-linear additive noise models (ANM), without the linearity assumption which can be overly restrictive in real-world scenarios. Moreover, the PICK is provably consistent under mild assumptions.

The motivation of the subroutine lies in the observation that the variance of the score function\footnote{In our method, the score function of any node $i$ is defined as the $i$-th element of the gradient of the log-likelihood function with respect to the sample point.} with respect to a certain node is comprised of both the variability of the node itself and that of its children, see Figure~\ref{fig:intuition}. Removing a leaf node leads to the reduction of the variance of the score function with respect to its parent, resulting in decreased variance, while the score function of other non-parent nodes remains unaffected. 
This then provides us with a method to identify the parent nodes of a specific leaf node. A brief example of the main algorithm procedure is provided in Figure~\ref{fig:framework}. 

Besides, temporal network data have gain much attention in many fields. The main challenge of handling such data structure is that the exogeneity assumption of the root nodes (with respect to the causal DAG) no longer holds in time series settings. This is because root node can be influenced by nodes from previous time steps, necessitating the design of a new method for identifying leaf nodes. 



In practice, data can manifest as either \emph{static} or \emph{temporal}, depending on the tasks and objectives. 
In both cases, different sample points may exhibit dependencies, most commonly in the form of network interference \citep{relational-data-causal-discovery,relational-data-causal-discovery2,fan2023directed}. Therefore, we evaluate our algorithm on both synthetic and real datasets. In static i.i.d. data, PICK is approximately 10 times faster than the existing score-matching-based methods, while still maintaining high statistical accuracy. In the case of temporal data with network interference, the recovered adjacency matrices encoding causal relationship by PICK achieve superior accuracy within acceptable running time, compared to baselines.

The main contributions of the paper are summarized as follows.
\begin{enumerate}
    \item 
    We present a novel parent-finding subroutine PICK for DAG structure learning, with the goal of improving the computational efficiency compared to traditional score-matching-based methods for both static and temporal data.
    \item We prove that PICK consistently identifies the underlying DAG structure; in particular, our assumption for consistency is generally weak and can be applicable to a broad range of scenarios in practice.
    \item PICK is shown, in both numerical experiments and real data analysis, to be a computationally-efficient algorithm with high accuracy.

\end{enumerate}




\begin{figure}[htbp] 
  \centering           
  \subfloat[Algorithm Framework]   
  {      \label{fig:framework}\includegraphics[width=0.47\linewidth]{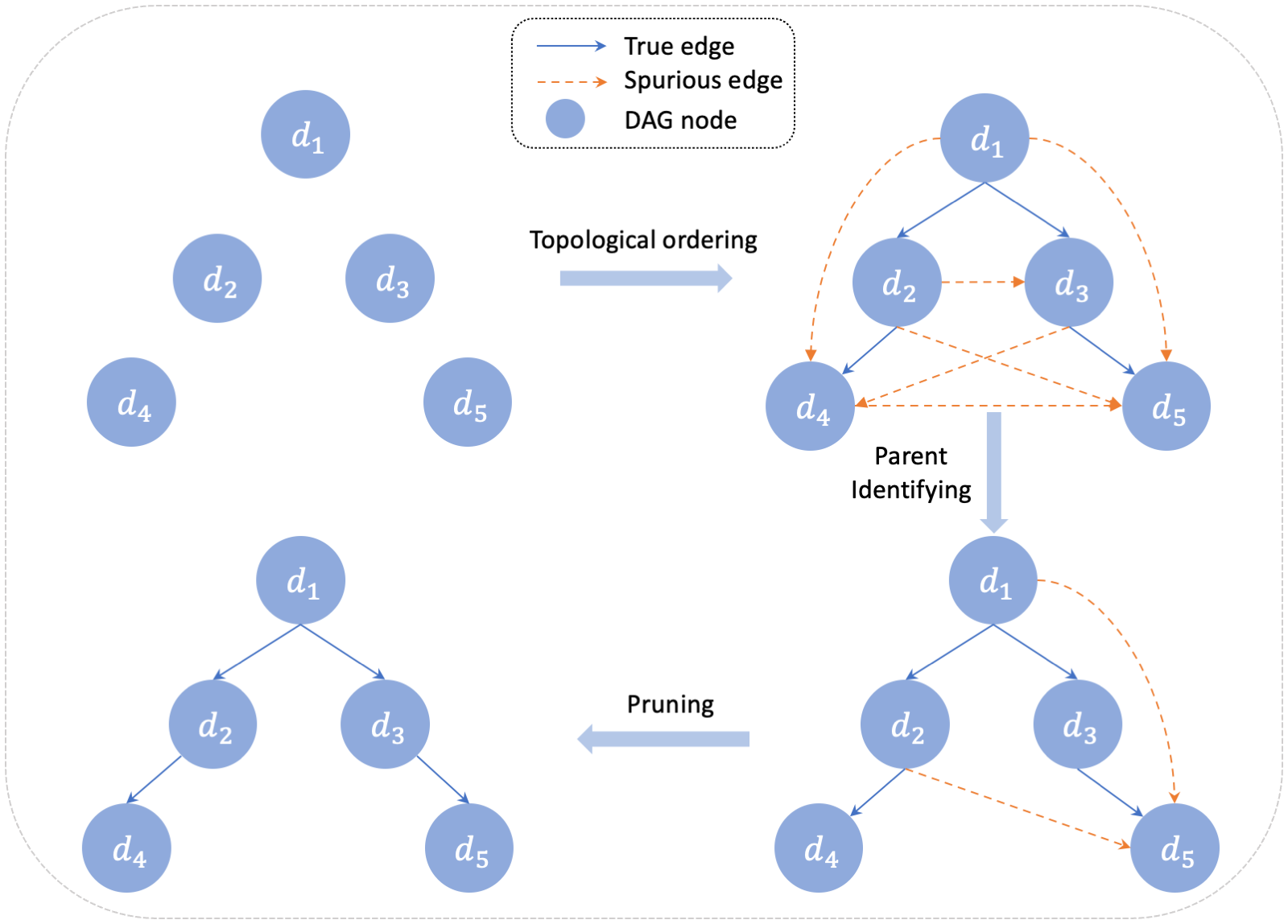}
  }
  \subfloat[Parent Identifying Intuition]
  {      \label{fig:intuition}\includegraphics[width=0.47\linewidth]{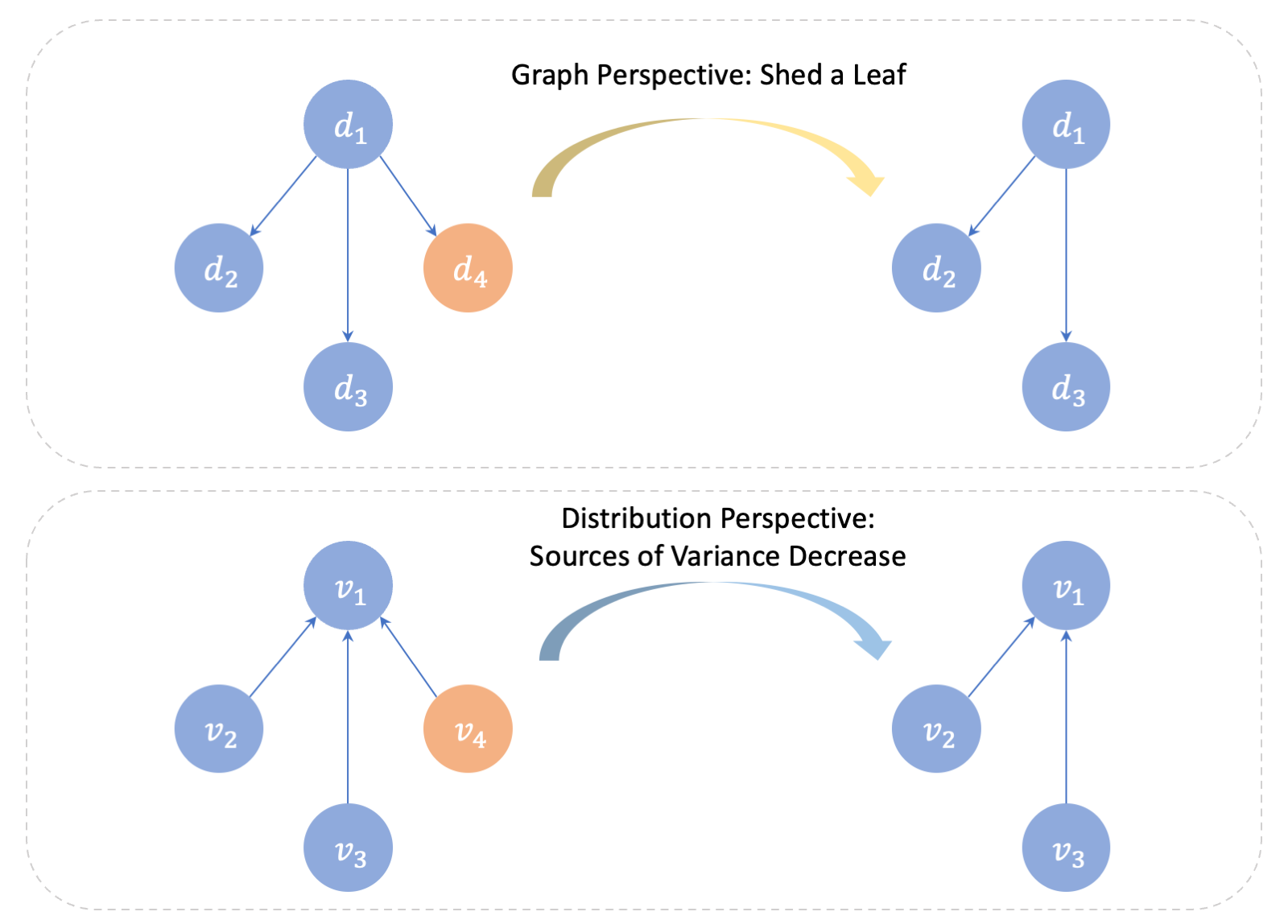}
  }
  \caption{The left provides a brief overview of the main framework of the PICK algorithm. The topological ordering and parent identification procedures are executed within a loop, with each iteration identifying one leaf node and its corresponding parent nodes. The right illustrates the intuition behind our parent identification subroutine. Notably, in the causal DAG (upper panel), although node \(d_1\) is the parent (cause) of nodes \(d_2\), \(d_3\), and \(d_4\), the variance of its score function is influenced by its three child nodes.}   
  \label{fig:main idea}          
\end{figure}

\section{Related works}
Bayesian Networks (BN), or probabilistic Directed Acyclic Graphs (DAG) models are widely adopted in the statistics, theoretical computer science, and machine learning community to model complex dependencies or causal structures in high-dimensional datasets \citep{pearl2003bayesian, ott2003finding, pearl1988probabilistic, koller2009probabilistic, wainwright2008graphical, spirtes2000causation, squires2023causal, bakshi2024structure}. 
Causal discovery provides an effective method for identifying DAG among variables through observations.
Most contemporary causal discovery methods, notably constraint-based algorithms such as PC and FCI, are tailored for data analysis under the classical independent and i.i.d. sampling scheme \citep{heckerman1995learning, geiger1994learning, spirtes1995causal}. These methods have been proven to be point-wise consistent in the infinite sample limit \citep{robins2003uniform, uhler2013geometry}. Given that causal discovery poses an NP-hard combinatorial optimization challenge \citep{chickering1996learning}, numerous heuristic alternatives have emerged, including score-based methods like Greedy Equivalence Search (GES) and its extensions \citep{chickering2002optimal, hasan2023kgs, JMLR:continuous-anm, HuaZha-generalized-score}. More recently, the machine learning and optimization communities have reconceptualized the discrete optimization problem in causal discovery as a continuous one by employing an innovative (semi-)algebraic and smooth characterization of DAG \citep{zheng2018dags, lachapelle2019gradient, sober-structure-learning, bhattacharya2021differentiable}. These efforts benefit from recent advances in smooth optimization techniques \citep{beck2017first}. Furthermore, Reinforcement Learning (RL)-based methods have also been increasingly applied to causal discovery \citep{zhu2019causal, wang-rl}.

Dynamic Bayesian Networks (DBNs) has been applied to various fields including clinical disease prognosis (\citep{DBN-clinical,DBN-lateral-sclerosis}), neuroscience (\citep{DBN-neurosci}), and etc. Hence, various kinds of methods have been proposed to learn these models. In \citet{pamfil2020dynotears}, the structure learning problem is turned into a continuous optimization problem in which the adjacency matrix of DAG is parameterized and finally obtained through optimization methods. Constraint-based methods have also been applied in learning DBNs (\citep{SVAR-condition-independ}). A framework solving causal discovery with randomly missing time series data is proposed in \citet{yuxiao2023cuts}. In this work, a Delayed Supervision Graph Neural Network (DSGNN) is used to predict the latent missing data. 

Graph-structured data has garnered some recent attention from areas in both science and industry, due to its flexibility in modeling dependent or structured data, such as images, small molecules or protein sequences \citep{DBLP:journals/corr/KipfW16,Kim2022MolNetAC,VisGNN,FlexibleGNNProtein}. Constraint-based causal discovery methods for static graph data have been proposed in pioneering works such as \citet{relational-data-causal-discovery} and \citet{relational-data-causal-discovery2}. In these works, conditional independence tests are used to orient bivariate causal relationship. Moreover, graph structured temporal data are also very common in real-world scenarios. For instance, in a social network, one subscriber's preference of a restaurant could be affected by both restaurant information and the preference of subscriber's friends. Therefore, causal discovery methods for such kind of data are needed. 
In \citet{fan2023directed}, a causal discovery method for temporal data on networks is proposed in which the causal DAG is parameterized and obtained by minimization of loss. In this paper, the causal relationship between several commercial factors are estimated.

In \citet{huang2020causal}, a framework termed as Constraint-based causal Discovery from heterogeneous/NOnstationary Data (CD-NOD) is proposed  for causal discovery from heterogeneous or non-stationary data with data generating process changing across domains or over time. In this work, constraint-based causal skeleton recovery procedure is firstly proposed and a method for determining causal directions are designed by exploiting the independent changes in data distributions. Furthermore, a nonparametric approach to extract latent low-dimensional representation is developed by exploiting KPCA techniques \citep{scholkopf1998nonlinear} which investigates eigenvalue decomposition problems.

\section{Preliminaries and Problem Setup}
\label{sec:prelim}

\paragraph{Notation}
Before proceeding, we collect some notation frequently used throughout the paper. For any positive integer $T$, we let $[T] \coloneqq \{1, \cdots, T\}$. Hereafter, we simply interpret $[T]$ as time points. We denote $X^{(t)}_{j}$ as the corresponding random variable for node $j$ at time step $t$; when a second subscript is attached, i.e. $X_{i, j}^{(t)}$, it means the corresponding random variable for the $i$-th observation unit from the given dataset. Given two functions $f(x)$ and $g(x)$, we write $f(x) = \Theta(g(x))$ if there exist absolute constants $c_1, c_2 > 0$ such that $c_1 g (x) \leq f(x) \leq c_2 g (x)$ for all $x$ sufficiently large.

\paragraph{Problem Setup}
\label{sec:setup}

Suppose that we have, at our disposal, a dataset comprised of a sequence of observations $\bfO_{T} \equiv \bfO \equiv \{O^{(1)}, \cdots, O^{(T)}\}$ measured repeatedly over a time horizon of length $T$, where for $t \in [T]$, $O^{(t)} \coloneqq (X^{(t)}, A^{(t)})$ consists of two sets of information, with $X^{(t)} \in \bbR^{n \times d}$ recording the values of $d$ features across $n$ different units at time $t$, and $A^{(t)} \in \{0, 1\}^{n \times n}$ encoding the (binary) adjacency matrix describing the dependency structure over the $n$ units at time $t$. Here we follow the convention that $A_{i, j}^{(t)} = 1$ if at $t$ the $i$-th and $j$-th units are connected and $A_{i, j}^{(t)} = 0$ if otherwise. For network graph at time step $t$ d, we invoke the following assumption
\begin{assumption}\label{asmp:sparse graph}
         (Graph sparsity) Denote the edge set of $G_t$ as $E_t$ and we assume that $\#E_t=\mathcal{O}(n)$.
\end{assumption}

We further assume that for each unit, the observations are random draws from a common probability distribution $\bbP$ with probability density function $\sfp$ that (1) Markov factorizes according to a DAG $G \equiv G (V, E)$ with $V$ the set of $d$ nodes corresponding to the $d$ features and $E$ the set of directed edges encoding direct causal effects from the incoming to the outgoing node with time lag $p_l$. We let $W \in \{0, 1\}^{d \times d}$ denote the adjacency matrix corresponding to the DAG $G$. $W$ naturally induces a topological ordering over the $d$ features, denoted as $\Pi = (\pi_{1}, \cdots, \pi_{d})$, a permutation of $[d]$. To describe the time-lag effects, we define adjacency matrices $P^{(k)}$ for $k = 1, \cdots, p$, where $P_{i, j}^{(k)} = 1$ if feature $j$ at time $t - k$ has a causal influence on feature $i$ at time $t$. With slight abuse of terminologies, we interpret $W$ as the intra-snapshot graph and $P^{(k)}$'s as the inter-snapshot graphs. We further let $s (x) \equiv \nabla \sfp (x)$ be the score function of $\sfp$ with respect to the random variable $X$, and $J_i (x) \coloneqq \frac{\partial s_i (x)}{\partial x_i}$ be the diagonal element of the Jacobian of the score function.

We further assume the following Structural Vector Autoregressive (SVAR) model \citep{demiralp2003searching} that generates the observations:
\begin{equation}
\label{eq:generation}
X^{(t)}_j = f_j \left(X^{(t)}_{\pa (j)}, \hat{X}^{(t - 1)}_{\pa (j)}, \dots, \ \hat{X}^{(t - p)}_{\pa (j)}\right) + z_{j},
\end{equation}
where $z_{j}$ denotes the exogenous Gaussian white noise (hence we assume causal sufficiency), $\pa (j) \equiv \pa_{G} (j)$ denotes the parent set of node $j$, including $j$ itself, with respect to the DAG $G$, and $f_{j}$ is a possibly nonlinear function that stays invariant in time. In the above model, it is noteworthy that we introduced a new set of variables $\hat{X}^{(t)}$ that aggregates information from the neighbor samples via sample-wise adjacency matrix $A$. In particular, we let $\hat{X}^{(t)} \coloneqq \hat{A}^{(t)} X^{(t)}$, where $\hat{A}^{(t)} = D^{-\frac{1}{2}} \Tilde{A} D^{-\frac{1}{2}}$, $\Tilde{A}=A + I$ and $I$ is the identity matrix. Here, $D$ is a diagonal matrix with $D_{ii}=\sum_{j}\Tilde{A}_{ij}$. We adopt such a formulation by following \citet{fan2023directed}, which is commonly used in modeling network or graph-structured data. To our knowledge, this particular aggregation rule can be traced back to \citet{kipf2016semi}. The reason that we use $\tilde{A} = A + I$ instead of $\tilde{A}$ is that the status of a variable in the past, in general, also influences its future status. The purpose of further normalizing the aggregation by $D^{-1 / 2}$ is to bring the aggregated variables back to the same scale.

\section{The PICK Algorithm}\label{sec:model}

\subsection{Illustration of the main idea under i.i.d. sampling scheme}

Before introducing DAGs estimation algorithm for non-i.i.d. data, we first consider the dynamic situation without network interference to illustrate some useful properties of the leaf and parent nodes. These properties are the key to our algorithmic development. In particular, without network interference, SVAR model \eqref{eq:generation} reduces to 
\begin{equation}
\label{eq:generation iid}
X^{(t)}_j=f_j \left(X^{(t)}_{\pa (j)}, X^{(t - 1)}_{\pa (j)}, \dots, X^{(t - p)}_{\pa (j)}\right) + z_{j},    
\end{equation} 
where all the aggregated variables $\hat{X}_{\pa (j)}^{(t - k)}$ in the past are replaced by the non-aggregated version $X_{\pa (j)}^{(t - k)}$, a consequence of the i.i.d. assumption.

We now present the following result, in a similar spirit to Lemma 1 in \citet{score-matching}, which offers critical insight for developing the aforementioned leaf-finding subroutine.

\begin{lemma}\label{lemma:hess-leaf}
Let $\Bar{X}^{(t)} \coloneqq (X^{(t)},\dots,X^{(t-p)})$. For any node $i$, we have:

(\romannumeral1) Node $i$ is a leaf $\iff$ $\forall{x}$ in the sample space of $X_{i}$, $\frac{\partial \sfs_i^{(t)}}{\partial x_i^{(t)}}(x)\equiv c$ for some constant $c$ that is independent of $x$, or equivalently $\Var_{\Bar{X}^{(t)}}(\frac{\partial \sfs_i^{(t)}}{\partial x^{(t)}_i}(\Bar{x}^{(t)})) \equiv 0$;

(\romannumeral2) Node $i$ is a leaf and $j \in \pa (i)$ $\iff$ $\sfs_i^{(t)}(x)$ depends on $x_j$, or equivalently $\Var_{\Bar{X}^{(t)}}(\frac{\partial \sfs_i^{(t)}}{\partial x_j}(\Bar{X}^{(t)}))$$\neq 0$.
\end{lemma} 
Lemma~\ref{lemma:hess-leaf} simply tells us that one can determine if a node is a leaf by computing its score function and the derivative of the score function, and then testing if the derivative is a constant.

Upon identifying the leaf nodes, the next step is to detect the parent set of the leaves. Lemma~\ref{lemma:hess-leaf} (ii) cannot be directly used for parent-set detection. Fortunately, the next theorem suggests a practically useful strategy by comparing the variances of certain ``partial scores'', before and after leaving the current leaf node out.

\begin{theorem}\label{thm:parent node}
Assume that the link function in \eqref{eq:generation iid} is smooth and its partial derivative is not identically equal to zero.
    Denote the joint probability density function of $\Bar{X}^{(t)}$ induced by $\sfp$ as $\sfp^{(t)}$ and $\sfs_{j,t^\prime}^{(t)} (\bar{x}^{(t)}) \coloneqq \frac{\partial\log \sfp^{(t)}}{\partial x_i^{(t^{\prime})}}$.
    For each leaf node $l$ at time step $t$, any node $j$ at $t^{\prime}$ with $t^{\prime} \in \{t,t-1,\dots,\ t-p\}$ is a parent of $l$ at $t$ if and only if 
    \begin{equation}
        \mathrm{Var}(\sfs_{j,t^\prime}^{(t)}(\Bar{X}^{(t)}))>\mathrm{Var}(\sfs_{j,t^\prime}^{(t)}(\Bar{X}^{(t)}_{\backslash\{l\}}))
    \end{equation}
    where $\sfs_{j,t^\prime}^{(t)} (\bar{x}_{\setminus \{l\}}^{(t)})$ is defined in the same way as $\sfs_{j,t^\prime}^{(t)} (\bar{x}^{(t)})$ except that node $l$ at $t$ is removed.
\end{theorem}


Theorem~\ref{thm:parent node} has the following important implication: given a leaf node, the partial score with respect to its parent should be greater than the same partial score except for excluding the given leaf. The above two results serve as the foundations of the structure learning algorithm (the PICK algorithm) to be introduced in the next section.

In the static setting, Theorem~\ref{thm:parent node} can be simplified drastically as follows with the same assumption.
\begin{corollary}\label{cor:parent node static}
For any node $i$ and leaf node $l$, $i$ is a parent node of leaf $l$ if and only if 
\begin{equation}
\Var(\sfs_i(X))>\Var(\sfs_i(X_{\backslash\{l\}})).
\end{equation}
\end{corollary}
This corollary is a direct consequence of Theorem~\ref{thm:parent node} and can be used to design scalable causal discovery algorithms for data in static i.i.d. setting (the same data generation process as in \citet{score-matching}).



\subsection{The algorithm}

We are now ready to present our new algorithm PICK that learns the underlying causal structure given temporal data in the presence of network interference among observation units. 

First, we illustrate how to estimate the score function and the diagonal elements of its Jacobian based on the (generalized) Stein's identity (or integration-by-part) \citep{stein1972bound,stein2004use,stein-hessian}. Recall that $\bar{X}^{(t)} \equiv (X^{(t)},{X}^{(t-1)}, \dots, {X}^{(t-p)})$\footnote{In our implementation, we replace $\bar{X}^{(t)}$ with $(X^{(t)},A^{(t-1)}{X}^{(t-1)},\dots,\ A^{(t-p)}{X}^{(t-p)})$ to obtain more stable performance}, with $\bar{\sfp}^{(t)}$ being the corresponding p.d.f and $s (\bar{x}^{(t)})$ being its score function. 
By Stein's identity \citep{stein1972bound}, given any sufficiently regular (multi-valued) test function $h^{(t)} (\cdot)$ such that $\lim_{x \to \infty} h^{(t)} (x) \bar{\sfp}^{(t)} (x) = 0$, one has
\begin{equation}
\label{stein}
\mathbb{E} [h^{(t)} (\bar{X}^{(t)}) \sfs^{(t)} (\bar{X}^{(t)})^{\top} + \nabla h^{(t)} (\bar{X}^{(t)})] \equiv 0.
\end{equation} 
Denote $H^{(t)} \coloneqq (h^{(t)}(\bar{X}_1^{(t)}),h^{(t)}(\bar{X}_2^{(t)}),\dots,\ h^{(t)}(\bar{X}_n^{(t)}))^{\top}$ as the matrix that concatenates the row vector $h^{(t)} (\bar{X}^{(t)})^{\top}$ over $n$ units. Following \cite{score-matching}, identity \eqref{stein} motivates the following ridge-regularized estimator $G^{(t)}$ of the score functions:
\begin{equation}\label{eq:score func estimate}
G^{(t)} = -(H^{(t)} H^{(t)\top} + \eta I)^{-1} H^{(t)} \sum_{k = 1}^{n} \nabla h^{(t)} (\bar{X}_{i}^{(t)}).
\end{equation}
Furthermore, by the second-order Stein's identity, we have
\begin{equation}
   \mathbb{E} \left[ q^{(t)}(\bar{X}^{(t)})\nabla^2\log \Bar{\sfp}^{(t)}(\bar{X}^{(t)}) \right]\equiv \mathbb{E} \left[ \nabla^2q^{(t)}(\bar{X}^{(t)})\right]
    -\mathbb{E} \left[q^{(t)}(\bar{X}^{(t)})\nabla\log \Bar{\sfp}^{(t)}(\bar{X}^{(t)})\nabla\log \Bar{\sfp}^{(t)}(\bar{X}^{(t)})^{\top} \right]
\end{equation}
where $q^{(t)}(x)$ is any test function such that $\lim_{x\to\infty}q^{(t)}(x)\Bar{\sfp}^{(t)}(x)=0$ and that $\mathbb{E}(\nabla^2q(x))$ exists. 
Similarly, we denote $Q^{(t)}=(q^{(t)}(\bar{X}_1^{(t)}),q^{(t)}(\bar{X}_2^{(t)}),\dots,\ q^{(t)}(\bar{X}_n^{(t)}))^{\top}$ as the matrix that aggregates the test functions over $n$ units, and then we can obtain the following ridge-regularized estimator of the Hessian 
\begin{align*}
J^{(t)}\equiv - \, \mathrm{diag}(G^{(t)}G^{(t)^\top})
+(Q^{(t)}Q^{(t)\top}+\eta I)^{-1}Q^{(t)}\sum\limits_{k=1}^n\nabla_{\mathrm{diag}}^2h(\bar{X}^{(t)}_k).
\end{align*}

Following results on ridge regression \citep{hoerl1970ridge,silva2015strong}, the above estimator is consistent for i.i.d. data. Given temporal data with network interference, Proposition~\ref{thm:score var convergence} below shows that the above method is also consistent under extra mild assumptions.

\begin{proposition}\label{thm:score var convergence}
Denote $Y^{(t)} \coloneqq (Y^{(t)}_{i}, i = 1, 2, \dots, n)$ as a set of $n$ independent random variables, each sharing the same marginal distribution as $X^{(t)}_{i}$. 
Under Assumption~\ref{asmp:sparse graph}, we could obtain that the sample variance estimator of score function, $\hat{\Var}(G^{(t)}) \coloneqq \sum_{i = 1}^{n} \frac{(G_i^{(t)}-\bar{G}^{(t)})^2}{n}$ in \eqref{eq:score func estimate} is consistent to its counterpart plugged in $Y^{(t)}$. 
\end{proposition}

The proof of Proposition~\ref{thm:score var convergence} is deferred to online Appendix \ref{sec:pf and theoretical analysis}. Here we only make a few remarks. It is straightforward to see that the score and Hessian estimators are consistent under certain mild regularity conditions. We are thus left to investigate how neighborhood interference affects variance estimation. Both assumptions imposed in Proposition~\ref{thm:score var convergence} are not very stringent. 
The first assumption could be satisfied as most neighborhood interference would manifest in the correlation between variables instead of marginal distribution.
Secondly, the sparsity of a graph in the real life scenario such as social network connection is a very common property.

Due to space limitation, we defer the full algorithm to Appendix (see Algorithms~\ref{alg:dynamicDAG} and \ref{alg:dynamic-pruning}) and only present the core algorithm for temporal data in Algorithm \ref{alg:PICK-t}. After obtaining the estimated intra-snapshot adjacency matrix $W^{(t)}$ and inter-snapshot adjacency matrix $P^{(t)}$ at each time step $t$, we compute the average results for all time steps and apply a thresholding operation to determine if an edge exists in the final step of the algorithm. 
As our parent finding method could also be applied to statically-observed data, we also apply this idea to develop a similar algorithm for static data (see Algorithm~\ref{alg:staticDAG}). Numerical experiments demonstrate its superior finite-sample performance and computational efficiency compared to several other causal discovery methods.

\begin{algorithm}
	\caption{PICK-t}
	\label{alg:PICK-t}
\hspace*{0.02in} {\bf Input:}\\
$\hspace{1cm}\Bar{X}^{(t)}\coloneqq(X^{(t)},\ {X}^{(t-1)},\dots,\ {X}^{(t-p)})$, $(A^{(t)},\ \dots,\ A^{(t-p)})$\\
\hspace*{0.02in} {\bf Output:} Estimated $\hat{W}^t$ and $\hat{P}^{t}$ at time $t$
	\begin{algorithmic}[1]
        \STATE $W \gets (p+1) d \times d$ zero matrix
        \FORALL{$i \in\ \{1,2\dots,d\} $}
        \STATE Estimate score function\\ $\sfs_{\rm nodes}=\nabla p_{\rm nodes}(x^{(t)},x^{(t-1)},\dots,x^{(t-p)})$ and its diagonal hessian value $H$, where $H_i=\frac{\partial \sfs_{i}}{\partial x_i}$
        \STATE Assign $\sfs_{\rm nodes}$ to $\mathrm{newSCORE}$
        \STATE Remove the last $dp$ columns of $H$
       \STATE $l\gets\mathop{\arg\min}\limits_{i}(var(H_i))$
        \STATE $\mathrm{ParentSet}$ $\gets$ $\{i:\Var (\mathrm{newScore}_i)<\Var (\mathrm{oldScore}_i)\}$
       \STATE Set  coordinates of $\mathrm{ParentSet}$ in the $l$-th column of $W$ to $1$.
       \STATE Remove the leaf column in $X$
       \STATE Remove the leaf column in $\mathrm{newScore}$ and assign it to $\mathrm{oldScore}$

    \ENDFOR
        \STATE Prune $W$ to remove spurious edges and obtain $\hat{W}^{(t)}$ and $\hat{P}^{(t)}$ 
	\end{algorithmic}  
\end{algorithm}

\paragraph{Pruning Spurious Edges}

As in \citet{score-matching,montagna2023scalable,montagna2023causal}, we need an additional pruning step after parent identification and topological ordering step in the algorithm to cut spurious edges for better performance.

For static data, we simply use the CAM-pruning step in \citet{score-matching}. For temporal data with network interference, we provide a pruning algorithm based on Chebyshev inequality, as shown in Algorithm \ref{alg:dynamic-pruning}. The convergence analysis is provided in Theorem \ref{thm:pruning-converge}. In the actual implementation, however, we adopt general additive model as a procedure selecting appropriate parent nodes following the CAM-pruning step proposed in \citet{CAM-method}. 
This strategy has been shown to work well from the results of our experiments.



\paragraph{Convergence of PICK}\label{sec:theoretical-analysis}
First, we would present the consistency property of predicted inter-snapshot and intra-snapshot adjacency matrix by the following theorem. For simplicity, we only discuss the situation of $p=1$.
\begin{theorem}\label{thm:dynamic-consistency}
     Let the inter-snapshot and intra-snapshot DAG predicted by Algorithm \ref{alg:dynamicDAG} be $\hat{P}$ and $\hat{W}$ and we have
    \begin{equation}
        \lim\limits_{n\to \infty}\mathbb{P}(\hat{W}=W)=1,\quad \lim\limits_{n\to \infty}\mathbb{P}(\hat{P}=P)=1.
    \end{equation}
\end{theorem}


\begin{corollary}\label{cor:static consistency}
    Let the topological ordering and DAG predicted by Algorithm \ref{alg:staticDAG} as $\hat{\pi}$ and $\hat{W}$ and we have
    \begin{equation}
        \lim\limits_{n\to \infty}\mathbb{P}(\hat{\pi}\in \Pi)=1,\quad \lim\limits_{n\to \infty}\mathbb{P}(\hat{G}= G)=1.
    \end{equation}
\end{corollary}

\paragraph{Algorithmic complexity}
We discuss the complexity of static situation first. Considering an input with $n$ samples and $d$ nodes, the overall complexity of our algorithm is $\Theta(dn^3+nd^2)$. On one hand, our algorithm could discover parent nodes in each loop with the estimation of score function which is computed during the the estimation of diagonal hessian matrix. Therefore the computing complexity of finding topological order is the same as SCORE \citep{score-matching}. On the other hand, experimentally and theoretically, the bottleneck of SCORE \citep{score-matching} is the first step of pruning approach, preliminary neighborhood search (PNS), amounts to complexity of $\Theta(nd^3)$. Due to the assumption of sparsity, the complexity of this procedure could be reduced to $\Theta(nd)$. 

For the dynamical situation, assuming an input with time step $T$, time-lagged time step number $p_l$ and sample number $n$ and node number $d$ for each time step, the complexity is $\Theta(T(pdn^3+p^2d^2))$.

\paragraph{Identifiability Discussion}
Identifiability is an important problem in SVAR models \citep{Kilian2017} and it illustrates that, given a structural equation model, the causal relationship graph could be idetified with observation data input. The identifiability of structure learning problem on temporal data has been studied in \citet{pamfil2020dynotears}. 
Here, we briefly discuss the identifiability of adjacency matrix $W$ that describes the causal relationships given temporal data as the input. Assume that we have an oracle to estimate the score function and its corresponded partial derivative, we can then find one leaf node for DAG following Lemma~\ref{lemma:hess-leaf}. After pruning this leaf node, we are able to discover its parents following Theorem~\ref{thm:parent node}. Repeating the above procedure, we can discover the ground truth topological ordering and the parents of each leaf node, which together recovers both the inter-snapshot and intra-snapshot causal graphs.

\section{Experiments}
\label{sec:experiment}

In this section, we apply PICK to both synthetic and real-world datasets to evaluate its finite-sample performance. 

\subsection{Synthetic data}
As discovering the underlying data generation machanisms of real-world scenarios is often difficult, it is extremely hard to obtain the ground truth causal structure. To validate the correctness and efficiency of our algorithm, we follow the setting in \citet{zheng2018dags,fan2023directed,pamfil2020dynotears} and conduct numerical experiments on synthetic data with known ground truth to simulate real-world scenarios. 
\paragraph{Dataset} 
We generate three types of synthetic data to evaluate the performance of our algorithm across various scenarios. The process of creating the necessary synthetic data involves two steps. First, we generate $W$, $P$, and $A$. Here, $W$ represents a directed acyclic graph at time $t$, generated using the Erd\H{o}s-Rényi model \citep{erdos1959random}.
For intra-snapshot graph $P$, we use the ER model \citep{erdHos1960evolution}. We set the sample size $n=1000$, and to generate the network graph with around $10n$ edges by setting the edge generation probability to $0.01$.
Next, we generate $(X_1,\dots,X_T)$ with two types of nonlinear functions including $f(x_1,\dots, x_k)=\sum\limits_{i=1}^k\sin x_i$ and nonlinear function generated by sampling from Gaussian processes using the RBF kernel with bandwidth 1. For noise term, we utilize Gaussian noise. Moreover, to compare PICK against baselines with a wide range
of the number of variables $d$, we vary $d\in\{10,15,20,25,30,35\}$ and set sample size $n=1000$. The length of time-series $T$ is set to 10. 

\paragraph{Baseline}
Since our focus is on nonlinear causal models, we use the GraphNOTEARS algorithm developed by \citet{fan2023directed} as one of our baselines. Consistent with the approach in \citet{fan2023directed}, we also select the following two methods as benchmarks.
\begin{itemize}
    \item NOTEARS \citep{zheng2018dags} + LASSO: This method includes two major steps. $W$ and $P^{(1)},\dots,\ P^{(p)}$ are estimated using NOTEARS and LASSO respectively. 
    
    $W$ is obtained through solving this optimization problem: Minimize $\mathcal{L}(W)=\frac{1}{2n}\sum_t||X^{(t)}-X^{(t)}W||_F^2+\lambda_W||W||_1$ s.t. $W$ is acyclic.

    The LASSO step is minimizing $\mathcal{L}(P)=\frac{1}{2n}\sum_t||X^{(t)}-\sum\limits_{k=1}^pA^{(t-k)}X^{(t-k)}P^{(t-k)}||_F+\lambda_P\sum\limits_{k=1}^p||P^{(t-k)}||_1$.
    \item DYNOTEARS \citep{pamfil2020dynotears} This method ignores the interactions among samples and obtains $W$ and $P$ by minimizing $\mathcal{L}(W,P^{(1)},\dots,P^{(p)})=\frac{1}{2n}\sum_t||(X^{(t)}-X^{(t)}W-\sum\limits_{k=1}^pX^{(t-k)}P^{(t-k)})||_F^2+\lambda_P\sum\limits_{k=1}^p||P^{(t-k)}||_1+\lambda_W||W||_1$ such that $W$ is acyclic.
\end{itemize}


\paragraph{Metrics}
We apply false discovery rate (FDR), structural Hamming distance (SHD) and true positive rate (TPR) to evaluate the predicted directed acyclic graph. SHD counts the number of missing, falsely detected or reversed edges. FDR computes the ratio between total number of predicted edges and falsely detected or reversed edges. TPR is the ratio between the total number of edges of the ground truth graph and the number of correctly predicted edges.
For FDR and SHD, lower is better and for TPR, higer is better. The results for SHD are shown in Figure~\ref{fig:dynamic-shd-intra-inter-snapshot-sin} while other results are deferred to the online Appendix. For each experiment, we repeated 10 times and report the average result together with the standard errors.\footnote{In some cases, the standard error is so small that the error bands around the lines in the figure are too narrow to be visible. This issue also occurs in the experiment figures for the static data.}

\paragraph{Implementation Details}
To utilize multiple-timestep data, we predict inter-snapshot and intra-snapshot causal relationship matrix denoted as $P_t$ and $W_t$ in each timestep $t$ with data from timestep $t-p$ to $t$ and compute the average results for all time steps. As elements of causal relationship matrix must be 1 or 0, we would use a hard thresholding method \citep{hardThre} with any element greater than threhold $\tau$ set to 1 and otherwise 0. In our experiments, we choose $\tau=0.4$.

\paragraph{Performance Evaluation}
We compare the performance of our algorithm and baselines for different DAG settings with different node and edge numbers and show the results in Figure~\ref{fig:dynamic-fdr-inter-snapshot-sin} and Figure~\ref{fig:dynamic-fdr-intra-snapshot-sin}. 
The results confirm that our algorithm can estimate both inter-snapshot and intra-snapshot causal relationships with high precision and within a reasonable time frame. This may be because all baseline methods assume a linear structural equation model, and only GraphNOTEARS takes sample interaction information into account.

\begin{figure}[htbp] 
  \centering           
  \subfloat[ER1-W]   
  {      \label{fig:dy-sin-shd-subfig1}\includegraphics[width=0.3\linewidth]{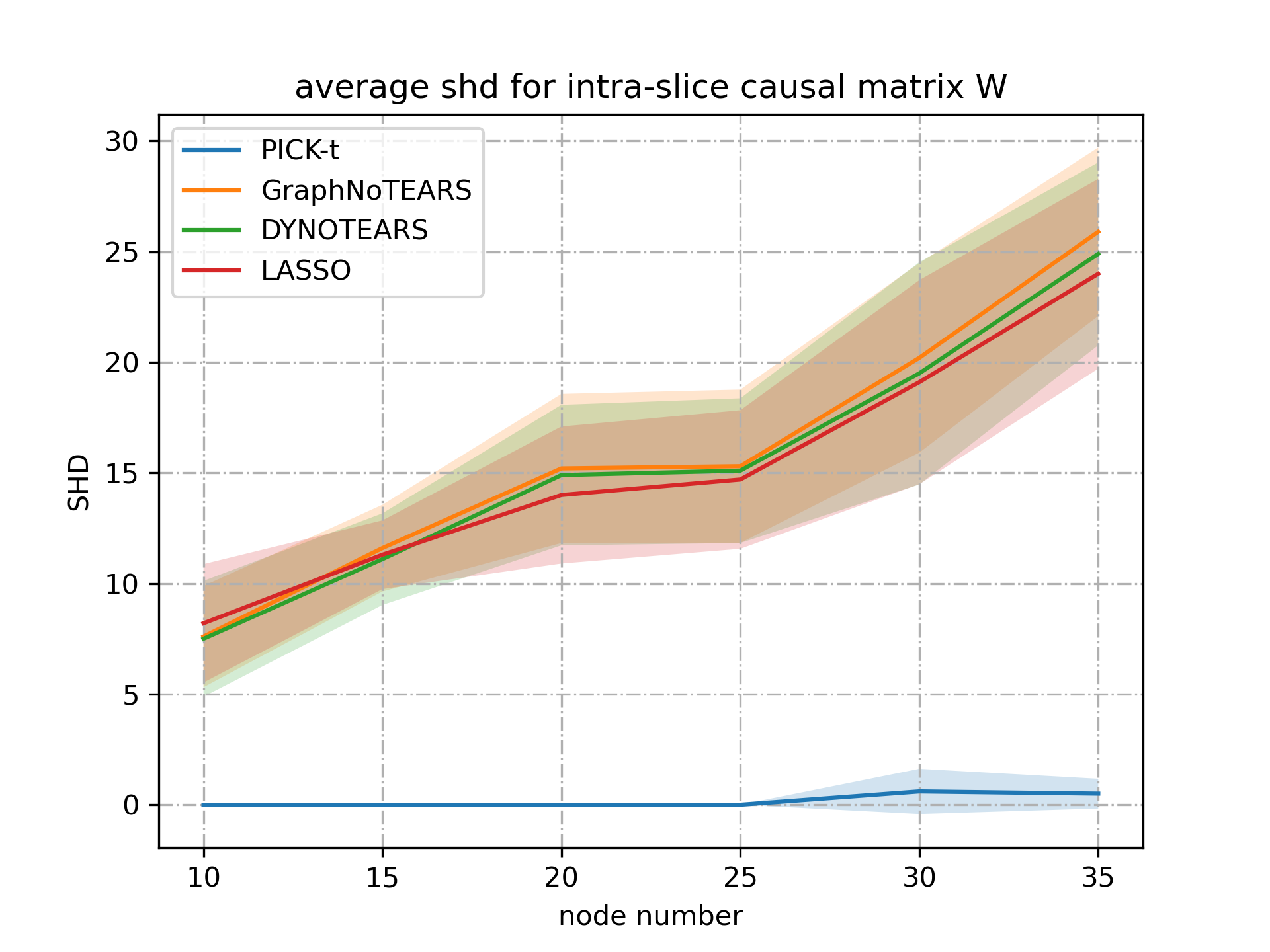}
  }
  \subfloat[ER2-W]
  {      \label{fig:dy-sin-shd-subfig2}\includegraphics[width=0.3\linewidth]{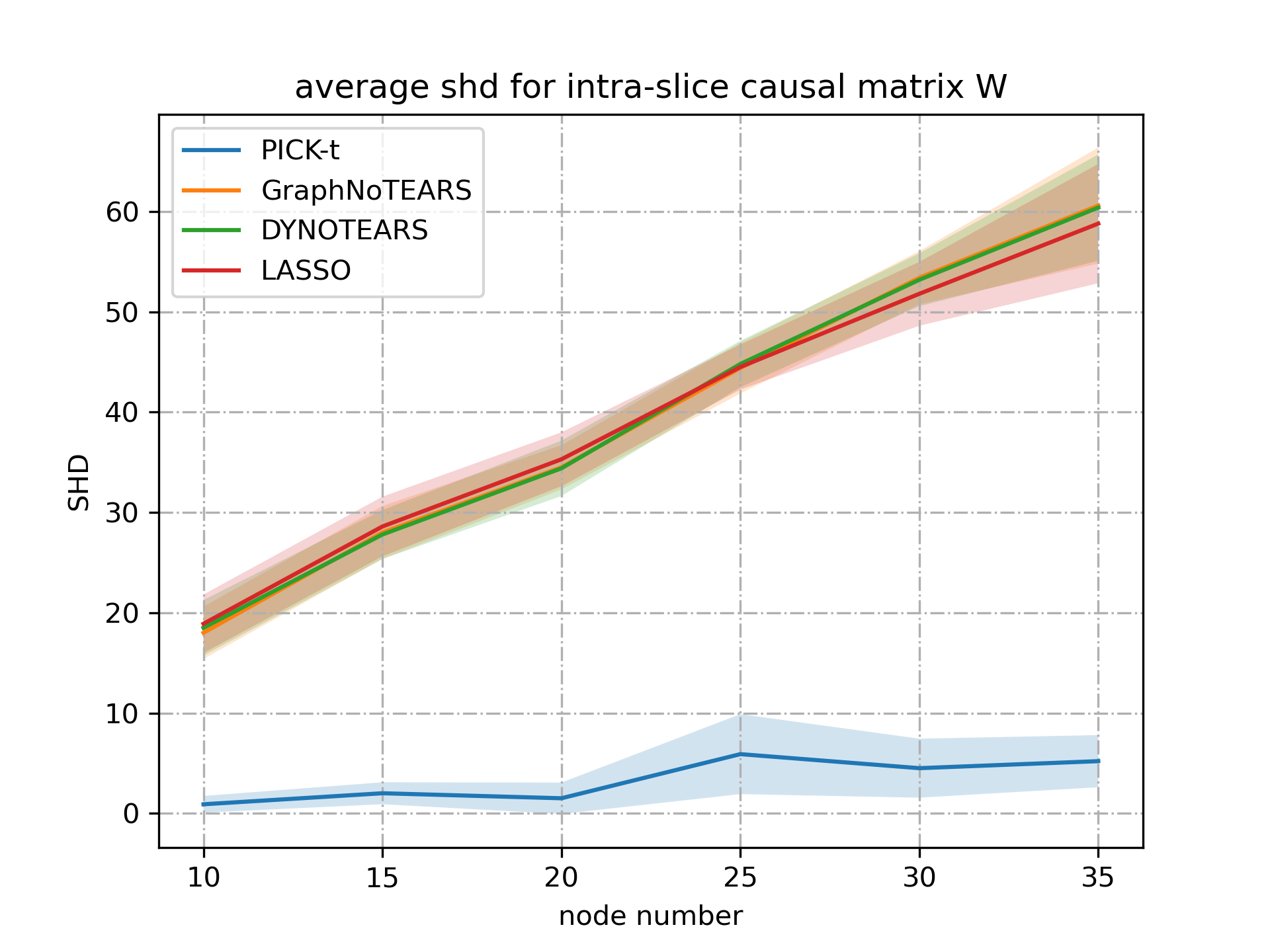}
  }
    \subfloat[ER4-W]
  {      \label{fig:dy-sin-shd-subfig3}\includegraphics[width=0.3\linewidth]{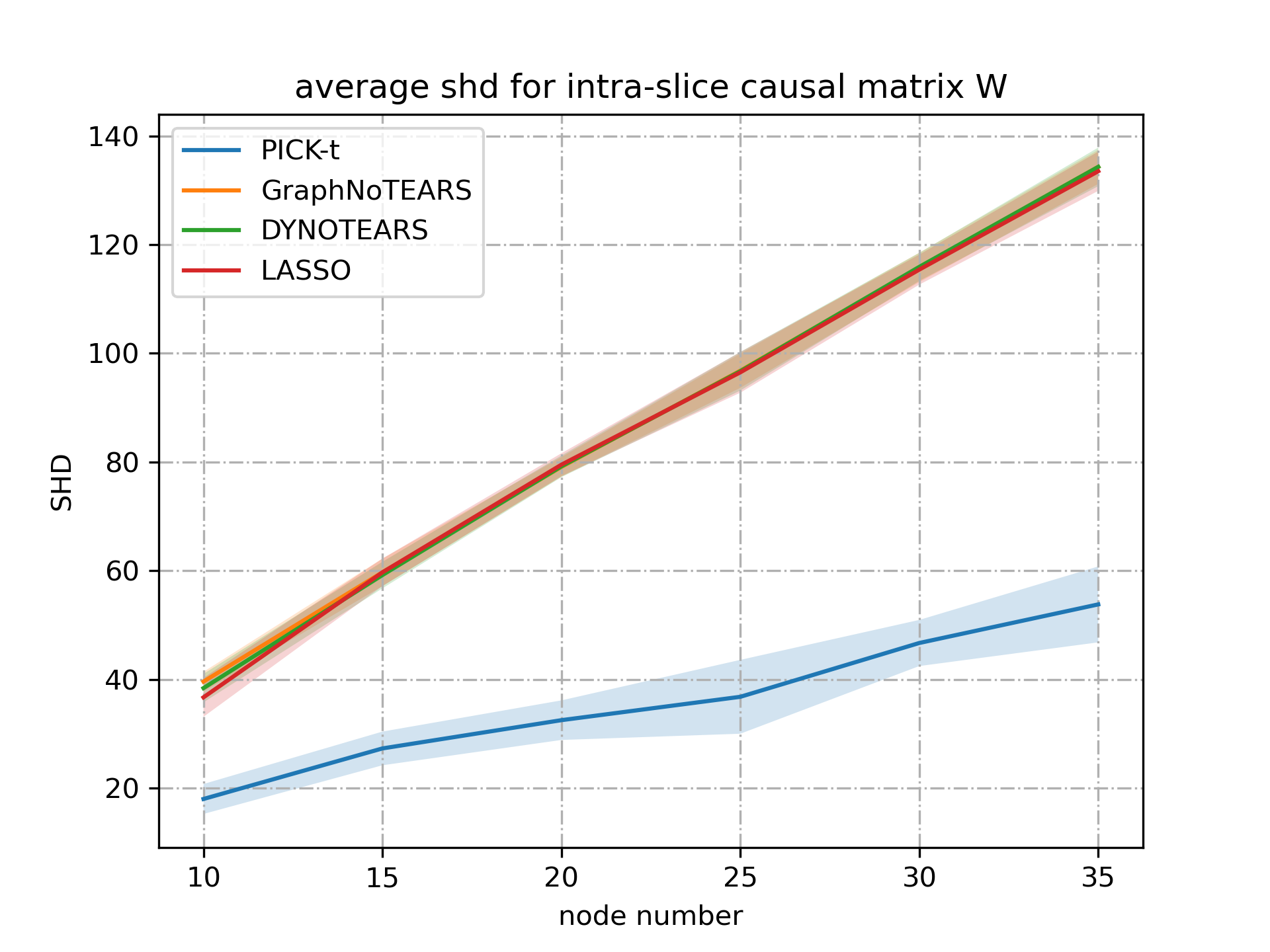}
  }

    \subfloat[ER1-P]   
  {      \label{fig:dy-sin-shd-p-subfig1}\includegraphics[width=0.3\linewidth]{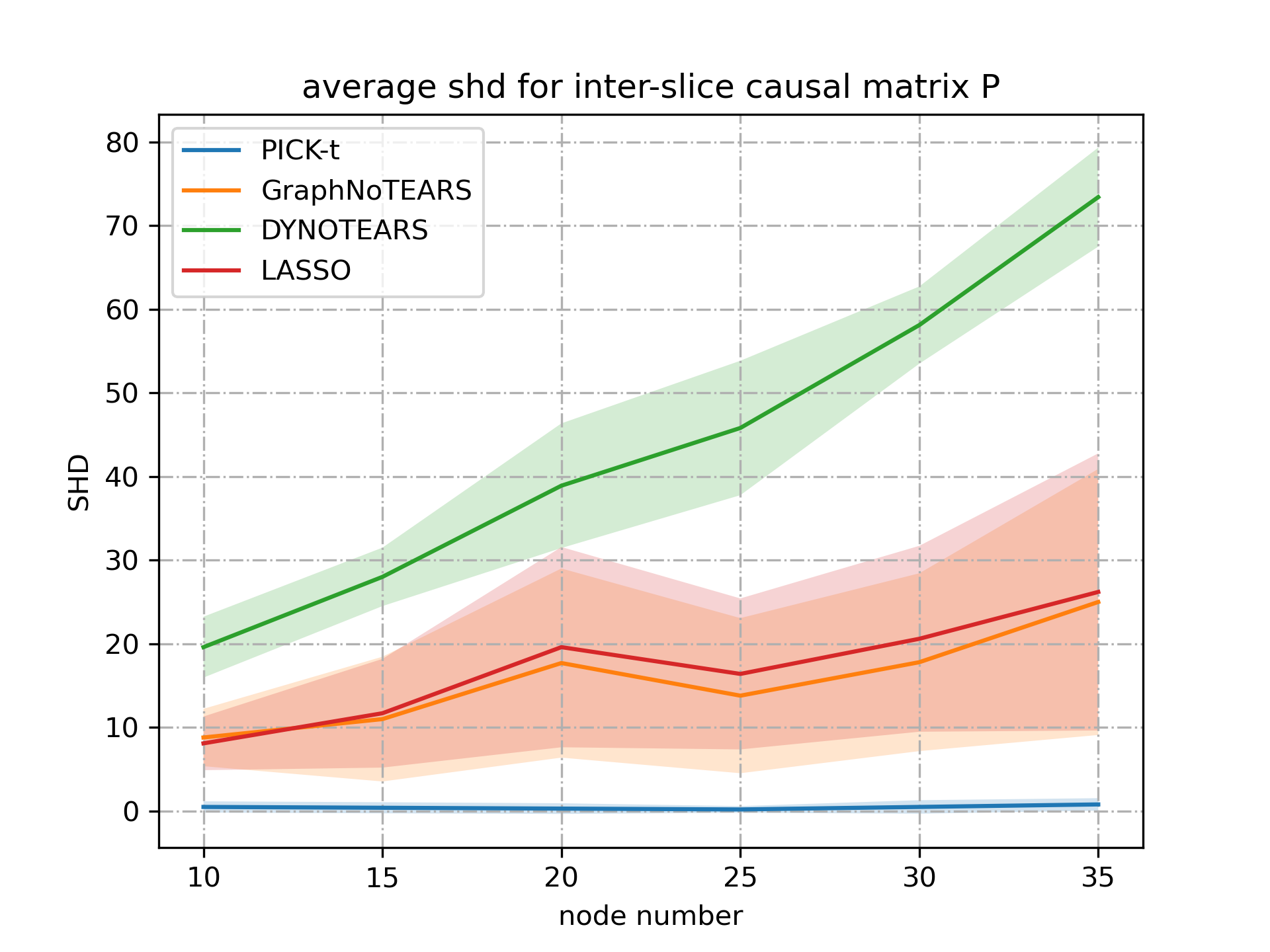}
  }
  \subfloat[ER2-P]
  {      \label{fig:dy-sin-shd-p-subfig2}\includegraphics[width=0.3\linewidth]{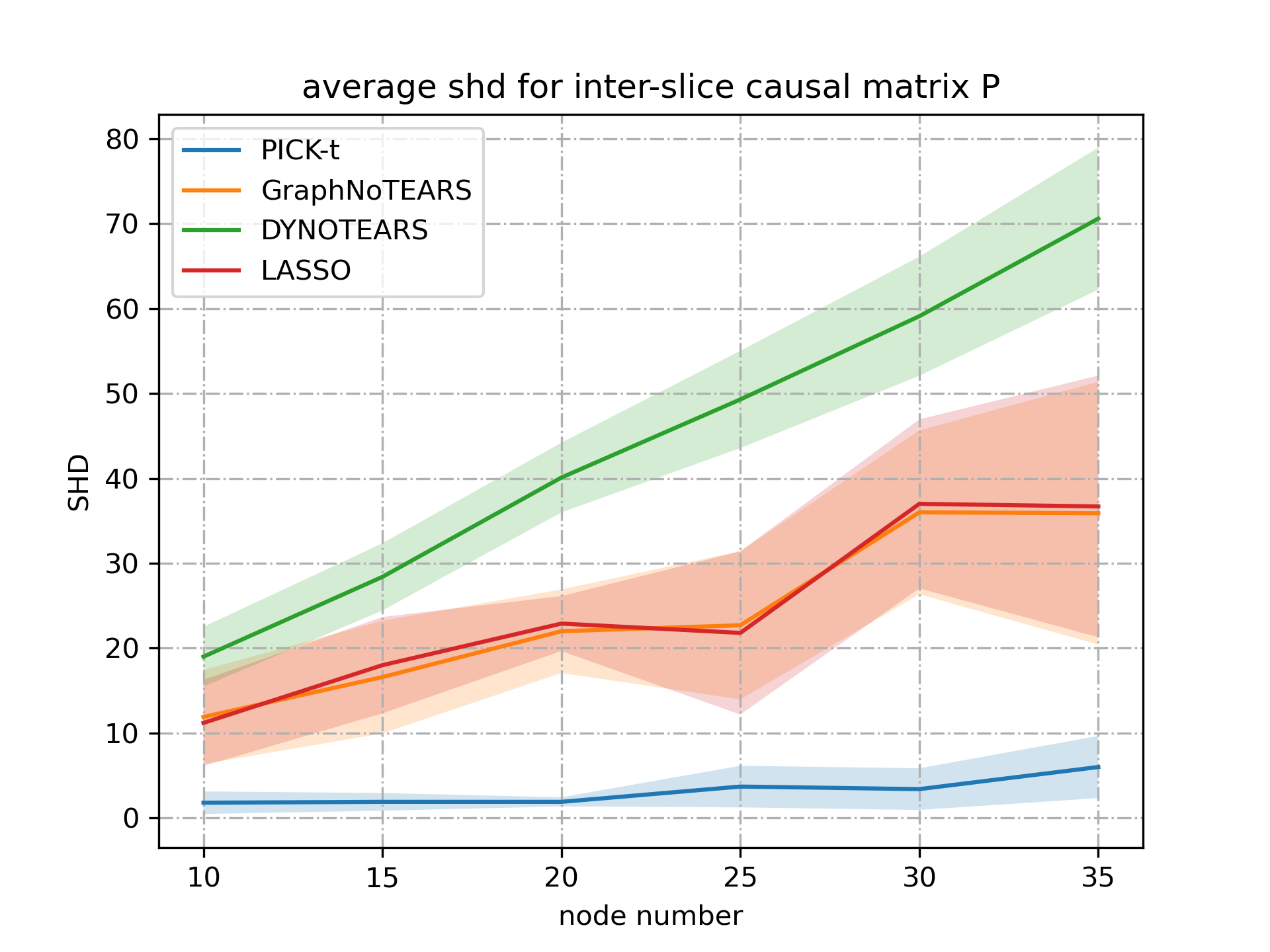}
  }
    \subfloat[ER4-P]
  {      \label{fig:dy-sin-shd-p-subfig3}\includegraphics[width=0.3\linewidth]{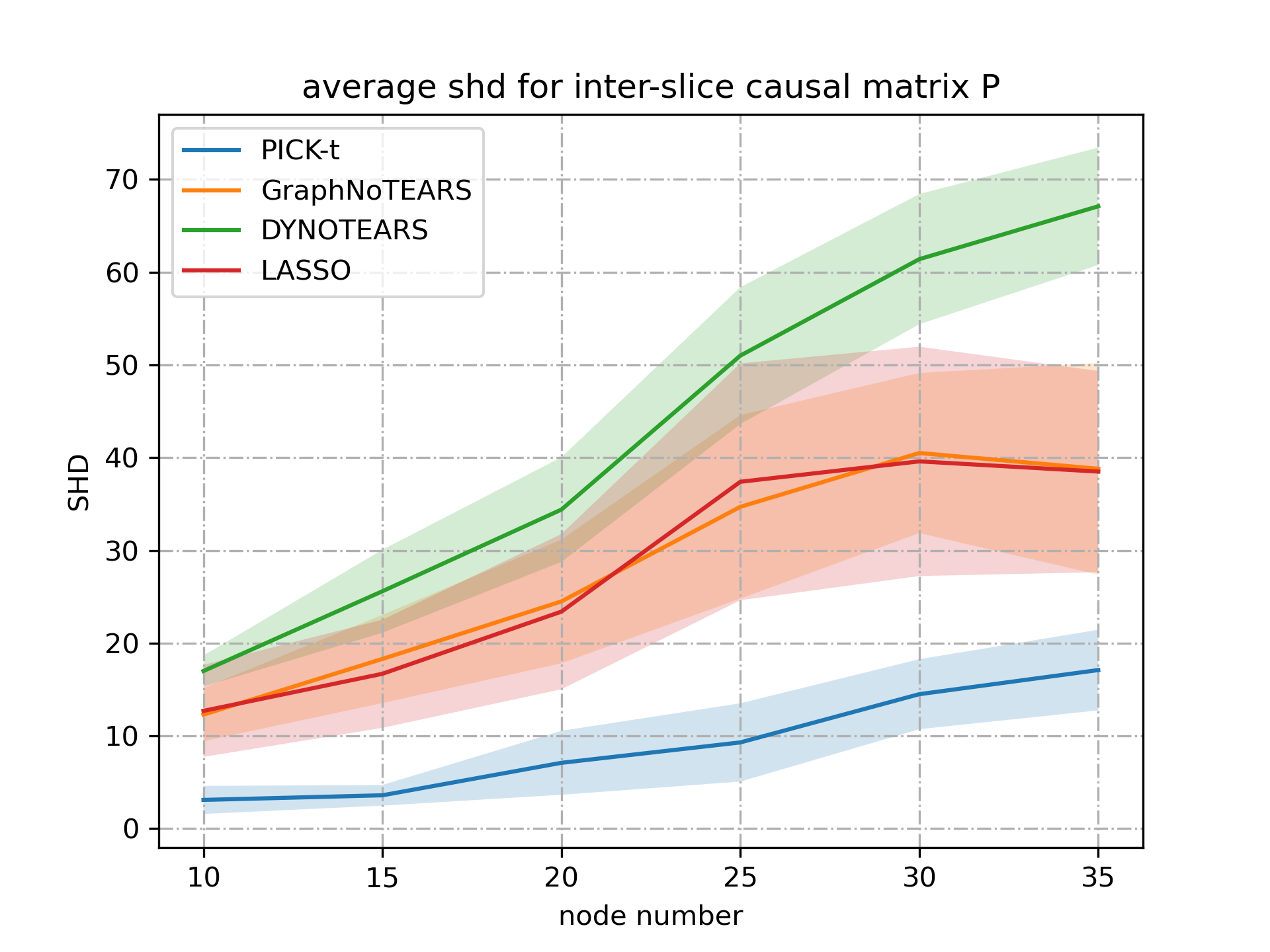}
  }
  \caption{SHD results of PICK-t and baselines for predicted intra-snapshot and inter-snapshot causal graph with link function $f_{i}^{(t)}(x_i)=\sum\limits_{j\in\pa(i)}\sin{x_j}$.}   
  \label{fig:dynamic-shd-intra-inter-snapshot-sin}          
\end{figure}

\paragraph{Static data}
We also test our algorithm in synthetic static data generated from a nonlinear additive data as additional experiments to show the efficiency. The nonlinear function $f_i$ for any node $i$ is generated by sampling Gaussian processes with a unit bandwitdth RBF kernel. The causal DAGs are generated by Erd\H{o}s R\'{e}nyi model \citep{erdHos1960evolution}. Different graph sparsity for DAGs is also considered in our experiment by setting average edge numbers to $d$, $2d$ and $4d$ where $d$ is the node number. For each method, we report the average structural Hanming Distance (SHD) and running time. The sturctural Hanming Distance computes the number of missing, falsely detected and reverse edges between the predicted and ground truth causal graph. We compared the performance of our method with CAM \citep{CAM-method}, GraNDAG, DAS \citep{montagna2023causal}, SCORE \citep{score-matching} with smaller node numbers. For greater node numbers, we only report the average SHD and running time for DAS and our method as the other two methods are too time consuming. For greater node number experiments, to achieve better efficiency, we utilize GAM pruning, leveraging $p$-value of GAM regression. This procedure is shown in Algorithm~\ref{alg:dynamic-pruning} in Appendix. 

The comparison of precision and time cost between our algorithm and baselines for different DAG settings are shown in Figure~\ref{fig:static-shd} and Figure~\ref{fig:static-time}. We could see that our algorithm has superior performance and much less time cost than the baselines. As the other three methods are too time consuming, we compare the computing time and SHD with DAS method for larger node number $d$ (high dimension) scenarios. 
\begin{figure}[htbp] 
  \centering           
  \subfloat[ER1]   
  {      \label{fig:sta-gp-shd-subfig1}\includegraphics[width=0.3\linewidth]{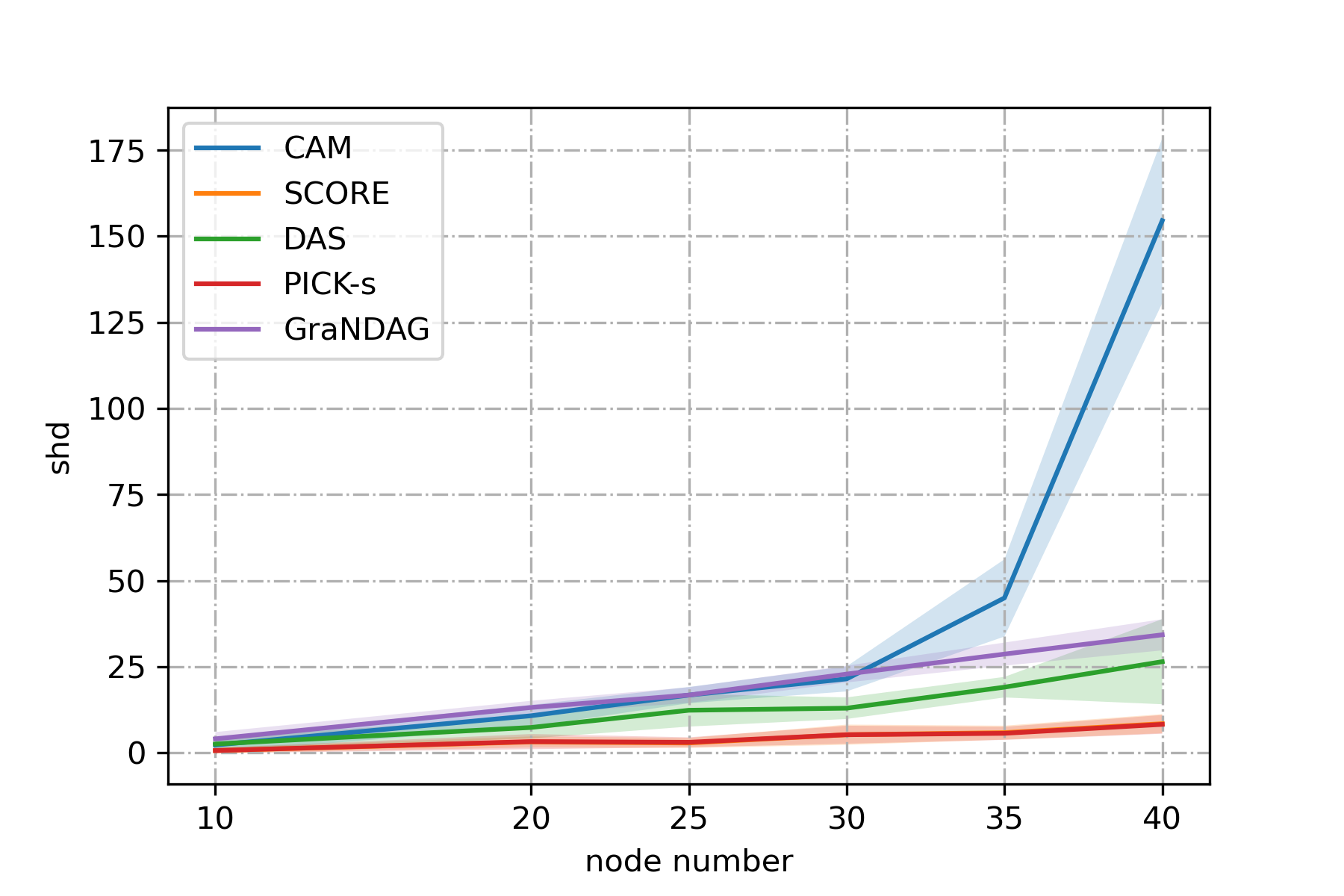}
  }
  \subfloat[ER2]
  {      \label{fig:sta-gp-shd-subfig2}\includegraphics[width=0.3\linewidth]{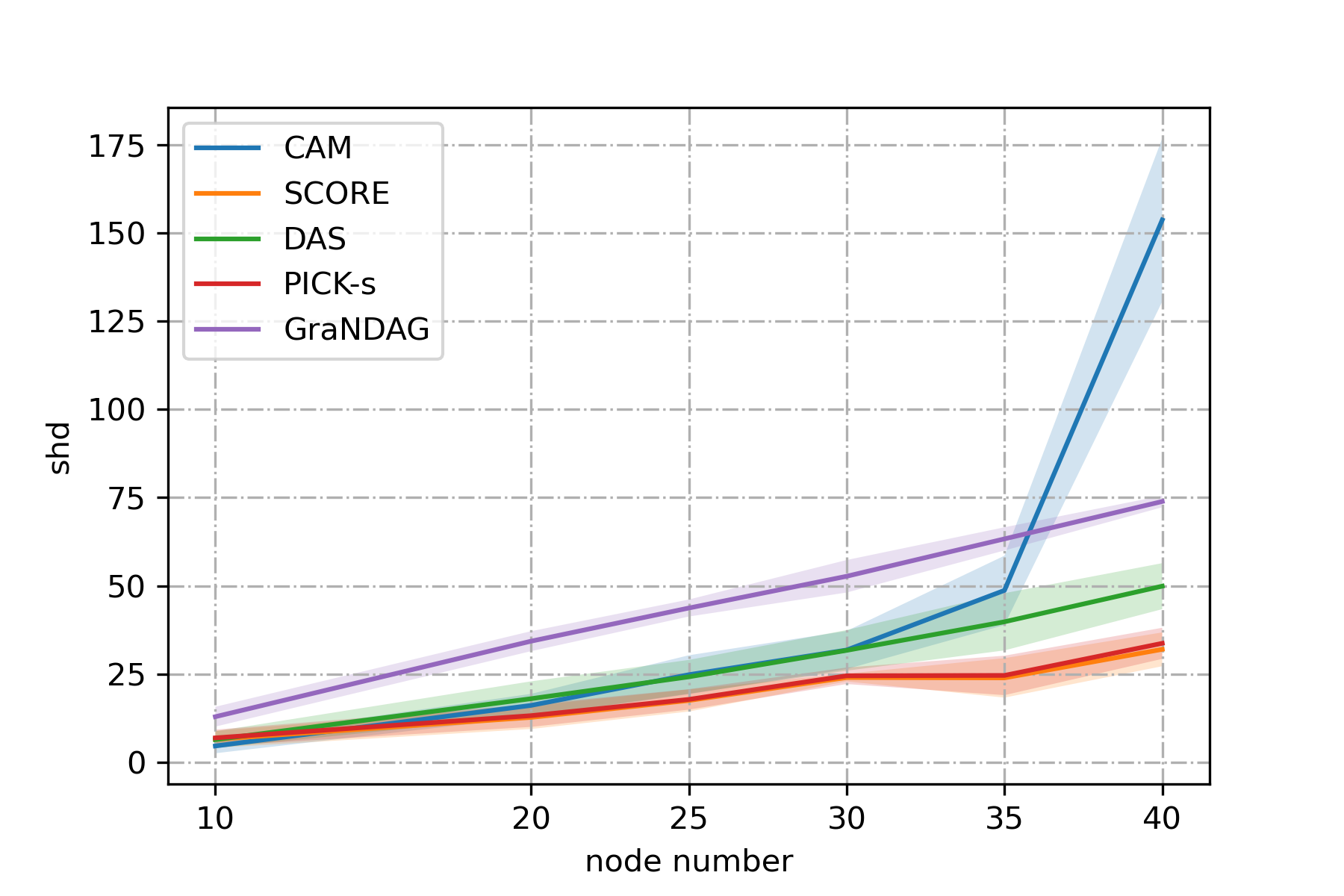}
  }
    \subfloat[ER4]
  {      \label{fig:sta-gp-shd-subfig3}\includegraphics[width=0.3\linewidth]{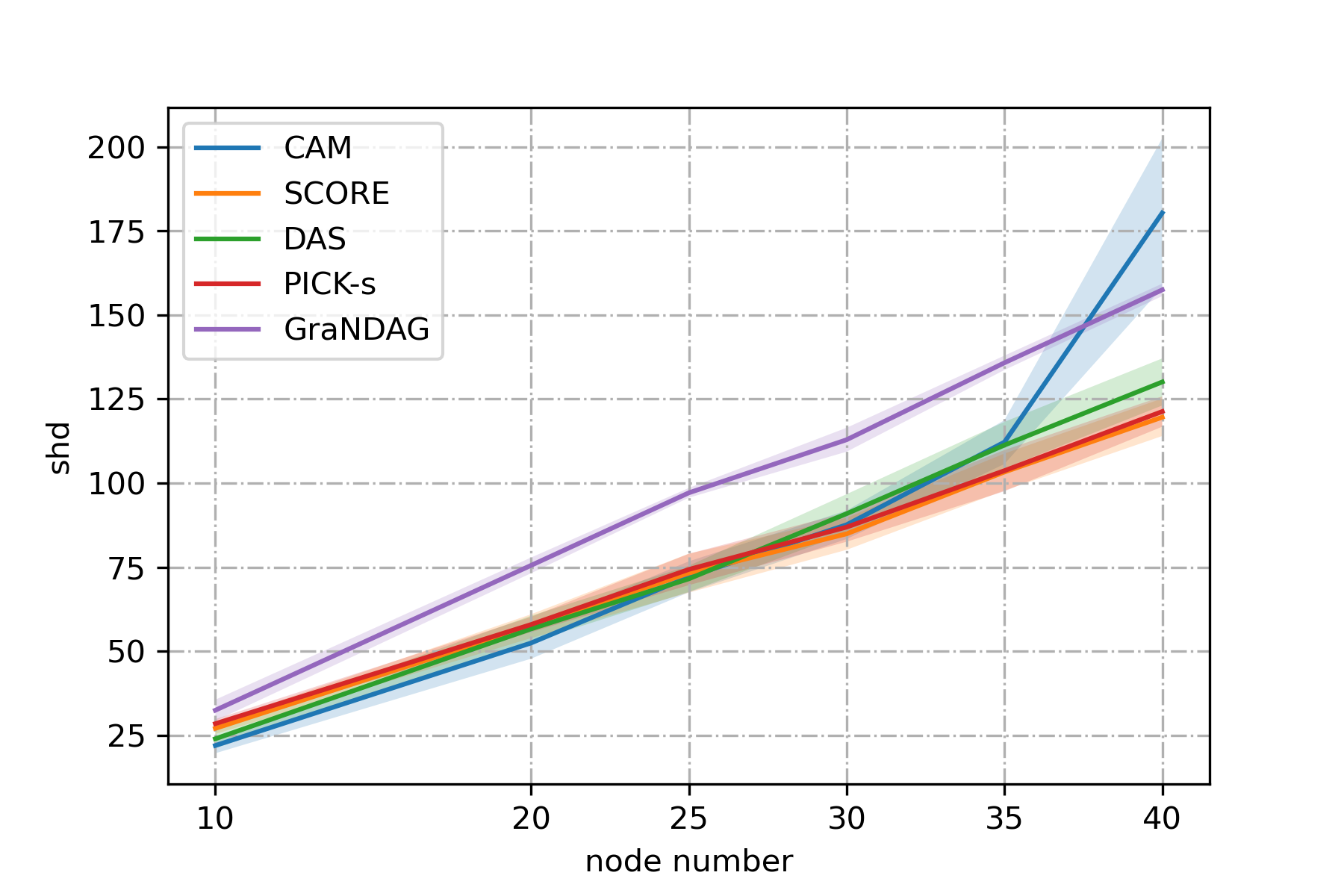}
  }

  \subfloat[ER1]   
  {      \label{fig:sta-gp-shd-b-subfig1}\includegraphics[width=0.3\linewidth]{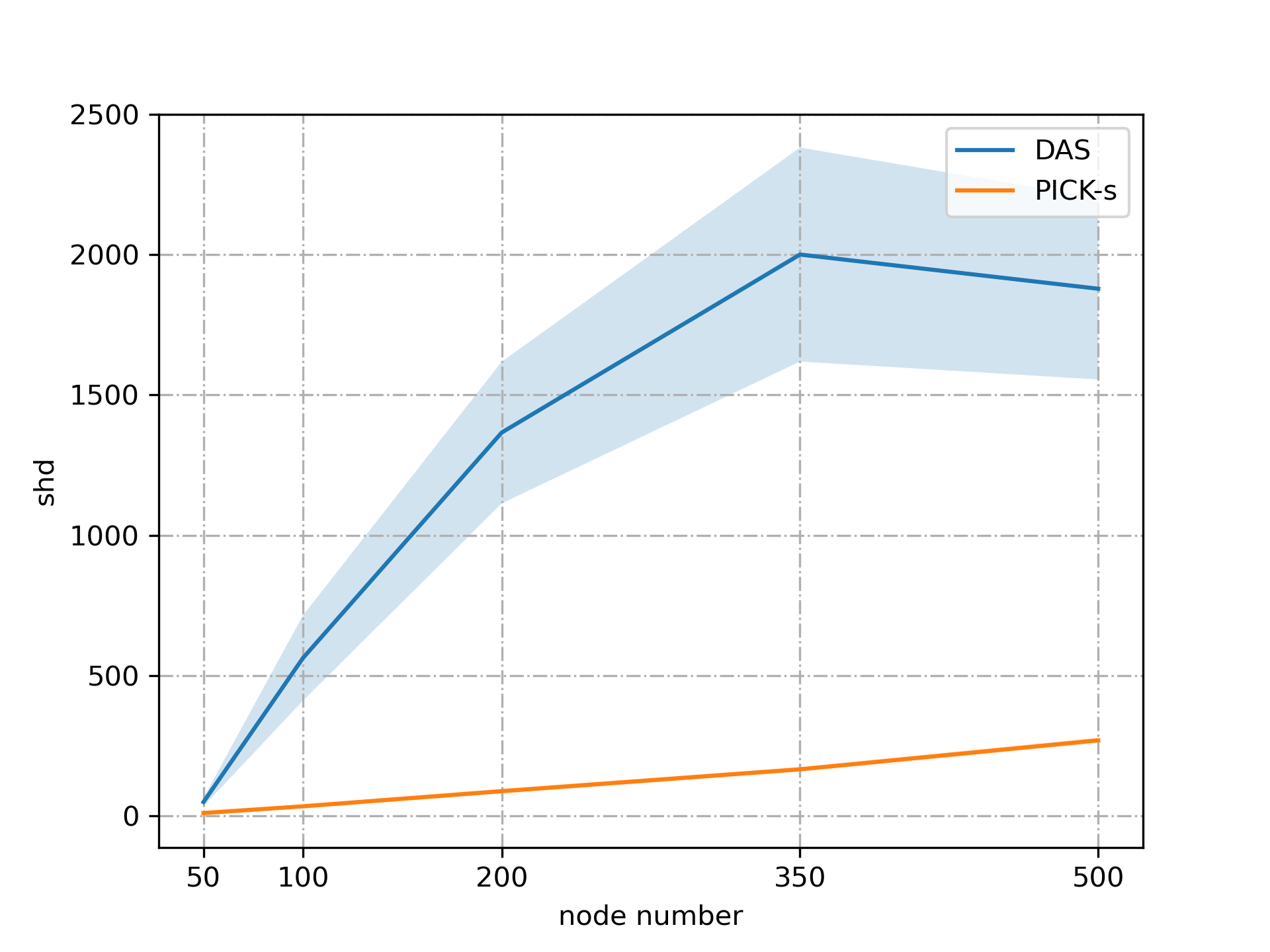}
  }
  \subfloat[ER2]
  {      \label{fig:sta-gp-shd-b-subfig2}\includegraphics[width=0.3\linewidth]{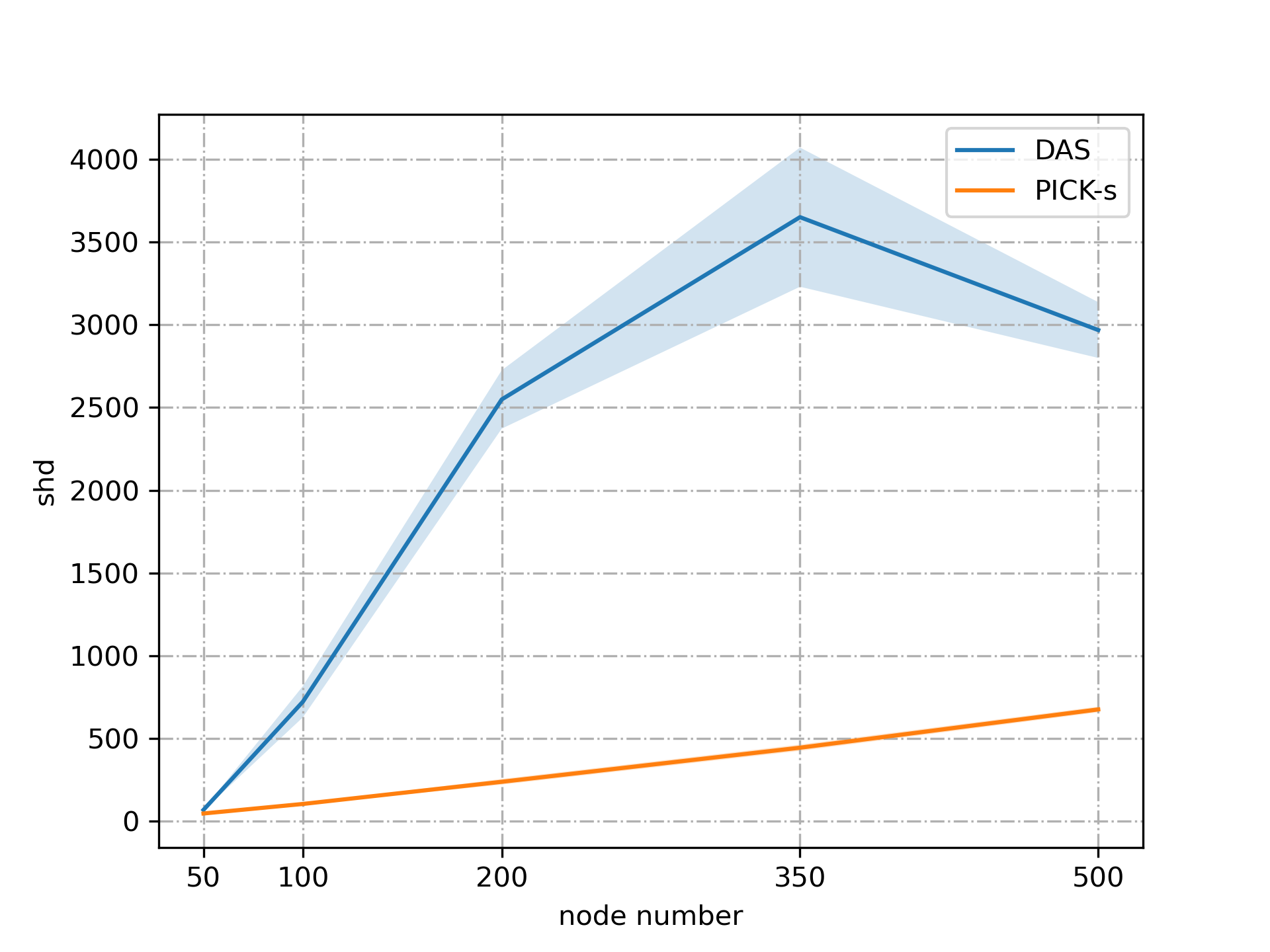}
  }
    \subfloat[ER4]
  {      \label{fig:sta-gp-shd-b-subfig3}\includegraphics[width=0.3\linewidth]{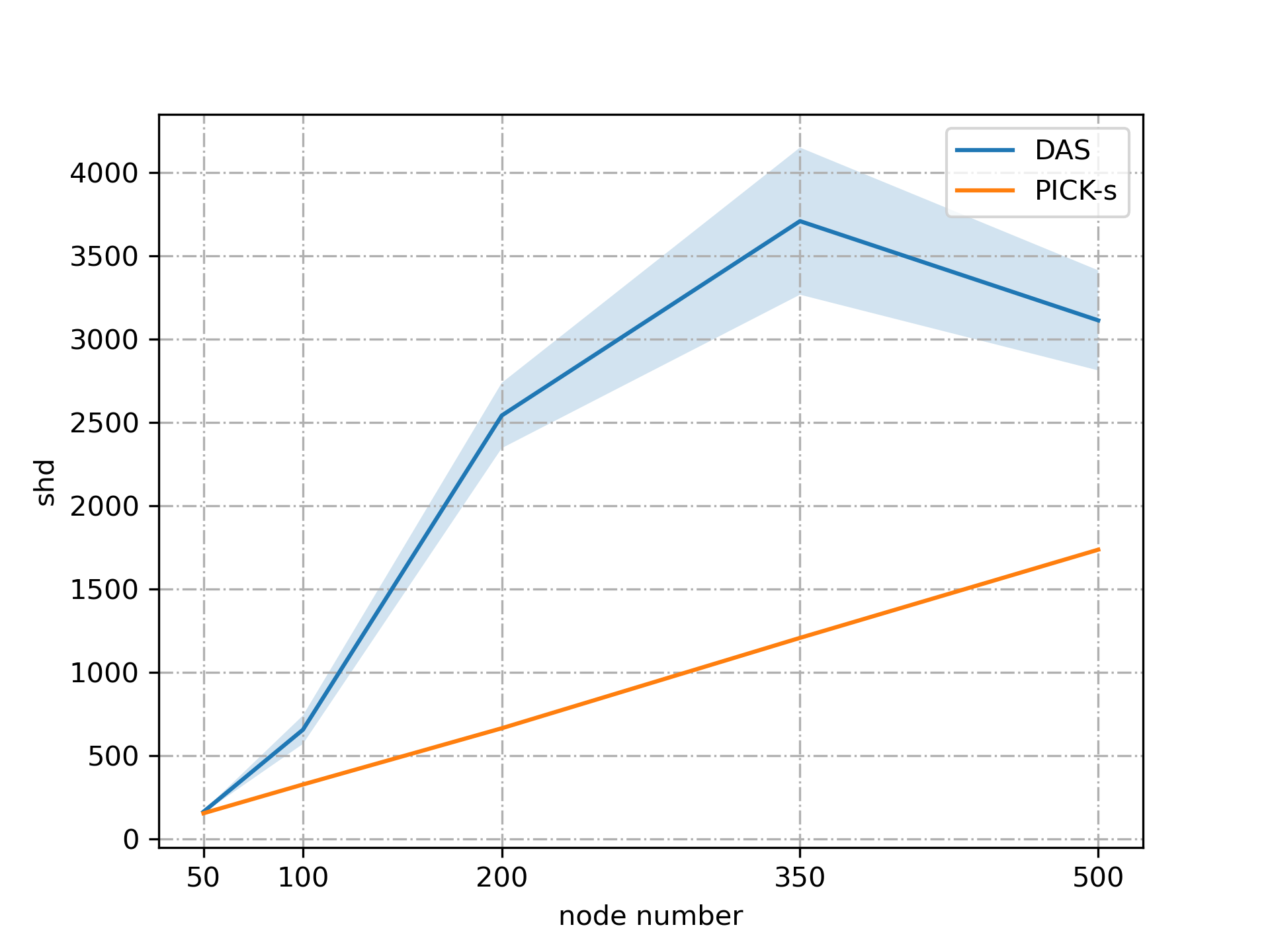}
  }
  \caption{SHD for predicted and ground truth causal graph with link function $f_{i}^{(t)}$ generated by sampling Gaussian process with a unit bandwidth RBF kernel. The upper and lower rows show the results for low dimension and high dimension respectively.}   
  \label{fig:static-shd}          
\end{figure}

\begin{figure}[htbp] 
  \centering           
  \subfloat[ER1]   
  {      \label{fig:sta-gp-time-subfig1}\includegraphics[width=0.3\linewidth]{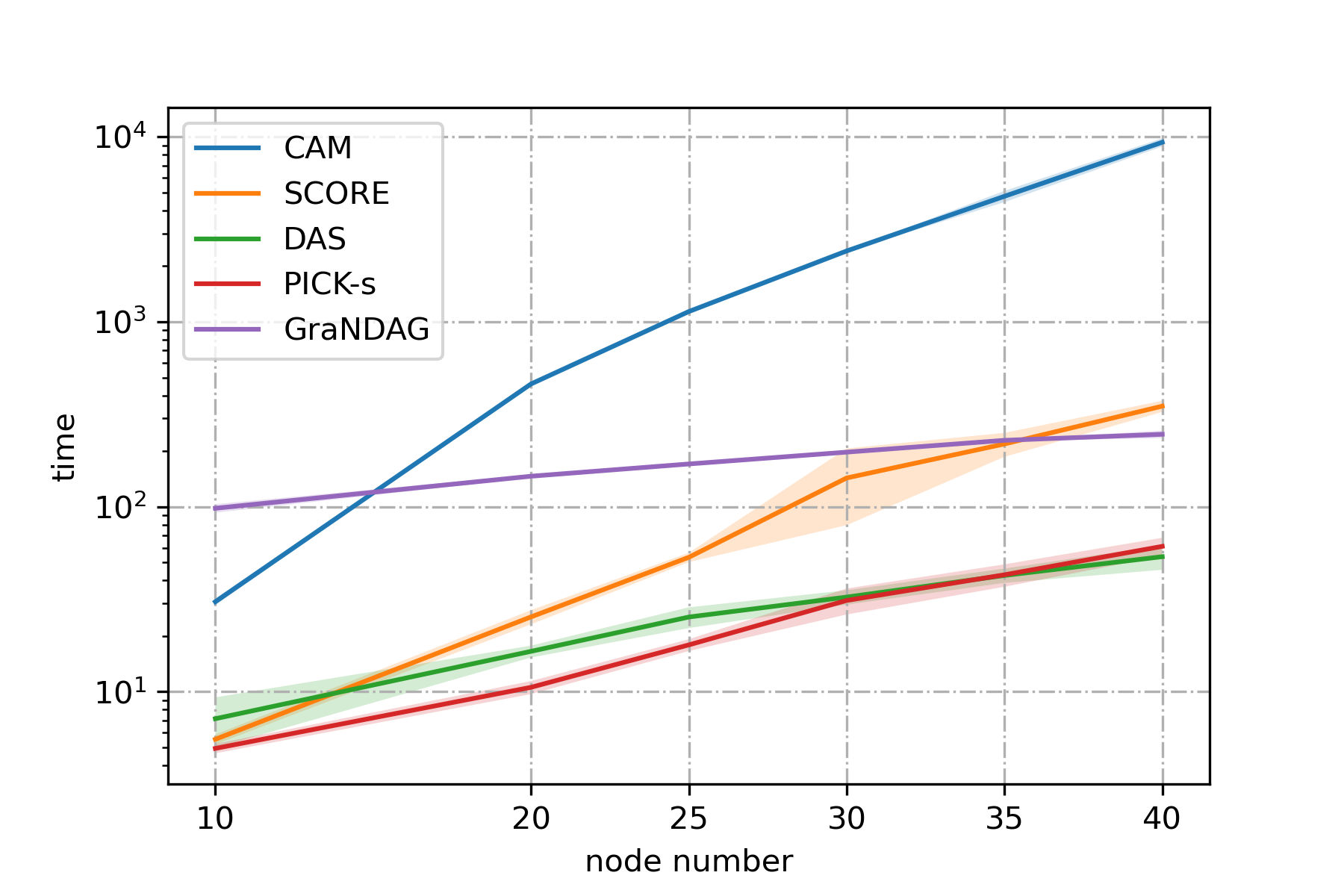}
  }
  \subfloat[ER2]
  {      \label{fig:sta-gp-time-subfig2}\includegraphics[width=0.3\linewidth]{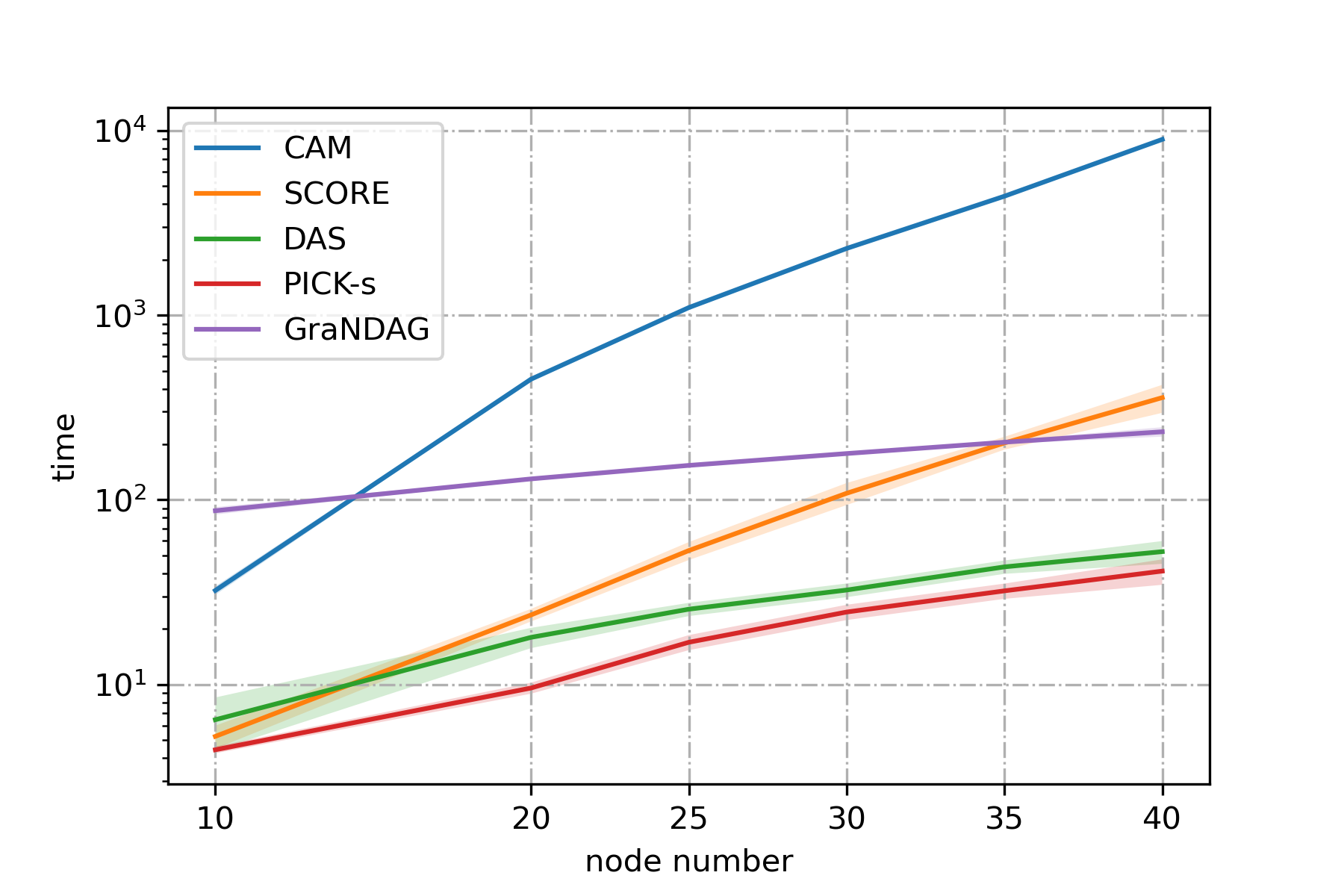}
  }
    \subfloat[ER4]
  {      \label{figsta-gp-time-subfig3}\includegraphics[width=0.3\linewidth]{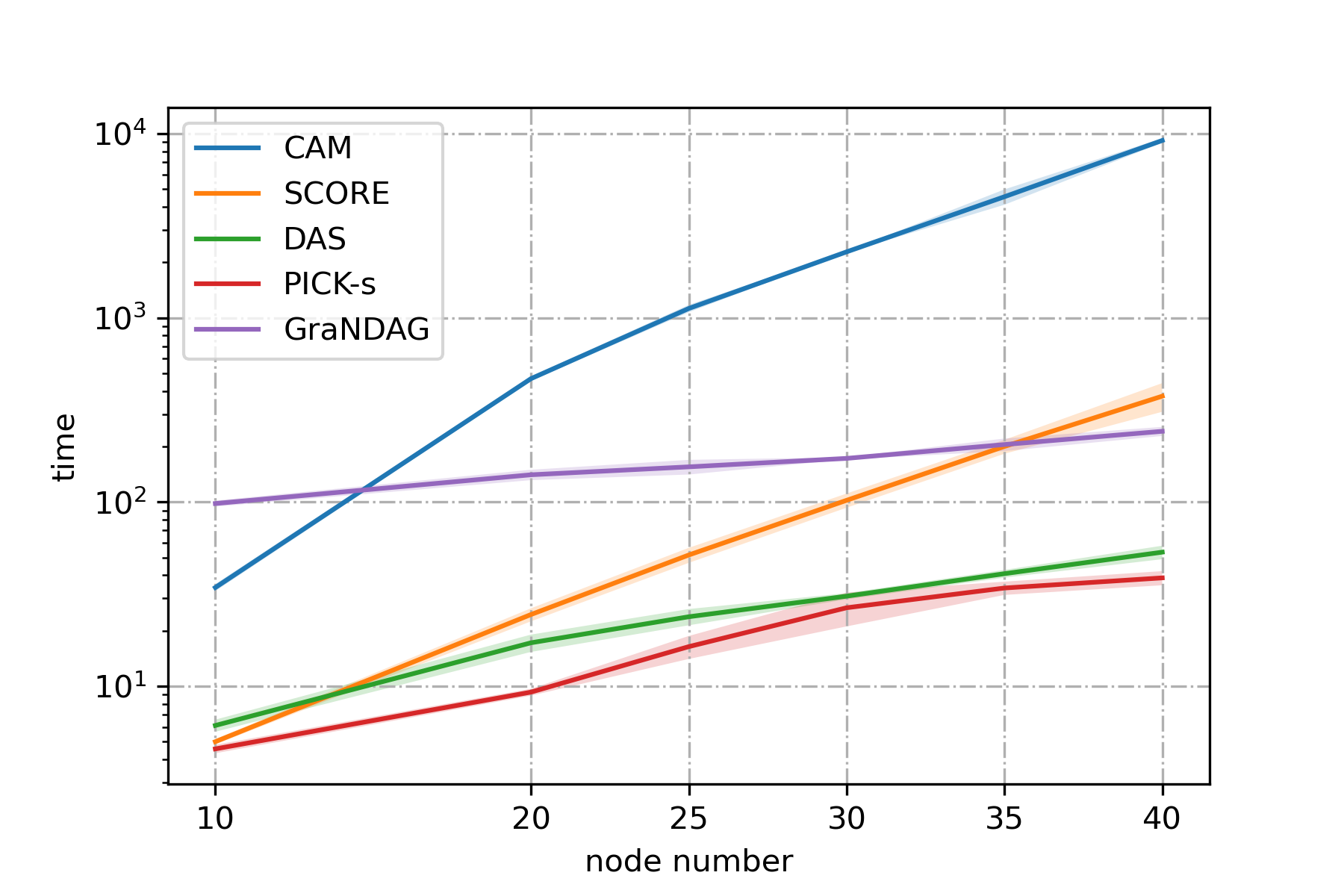}
  }

    \subfloat[ER1]   
  {      \label{fig:sta-gp-time-b-subfig1}\includegraphics[width=0.3\linewidth]{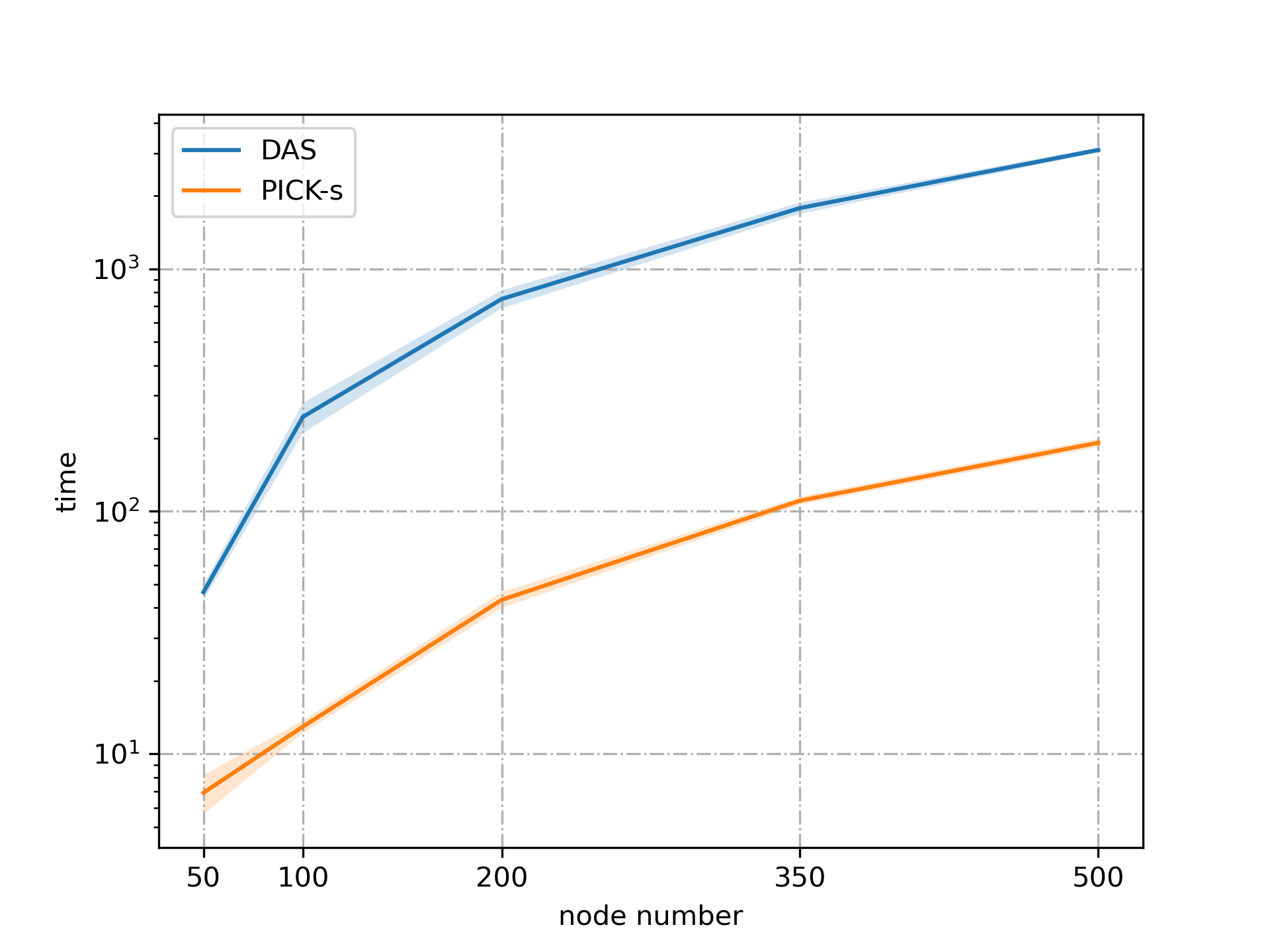}
  }
  \subfloat[ER2]
  {      \label{fig:sta-gp-time-b-subfig2}\includegraphics[width=0.3\linewidth]{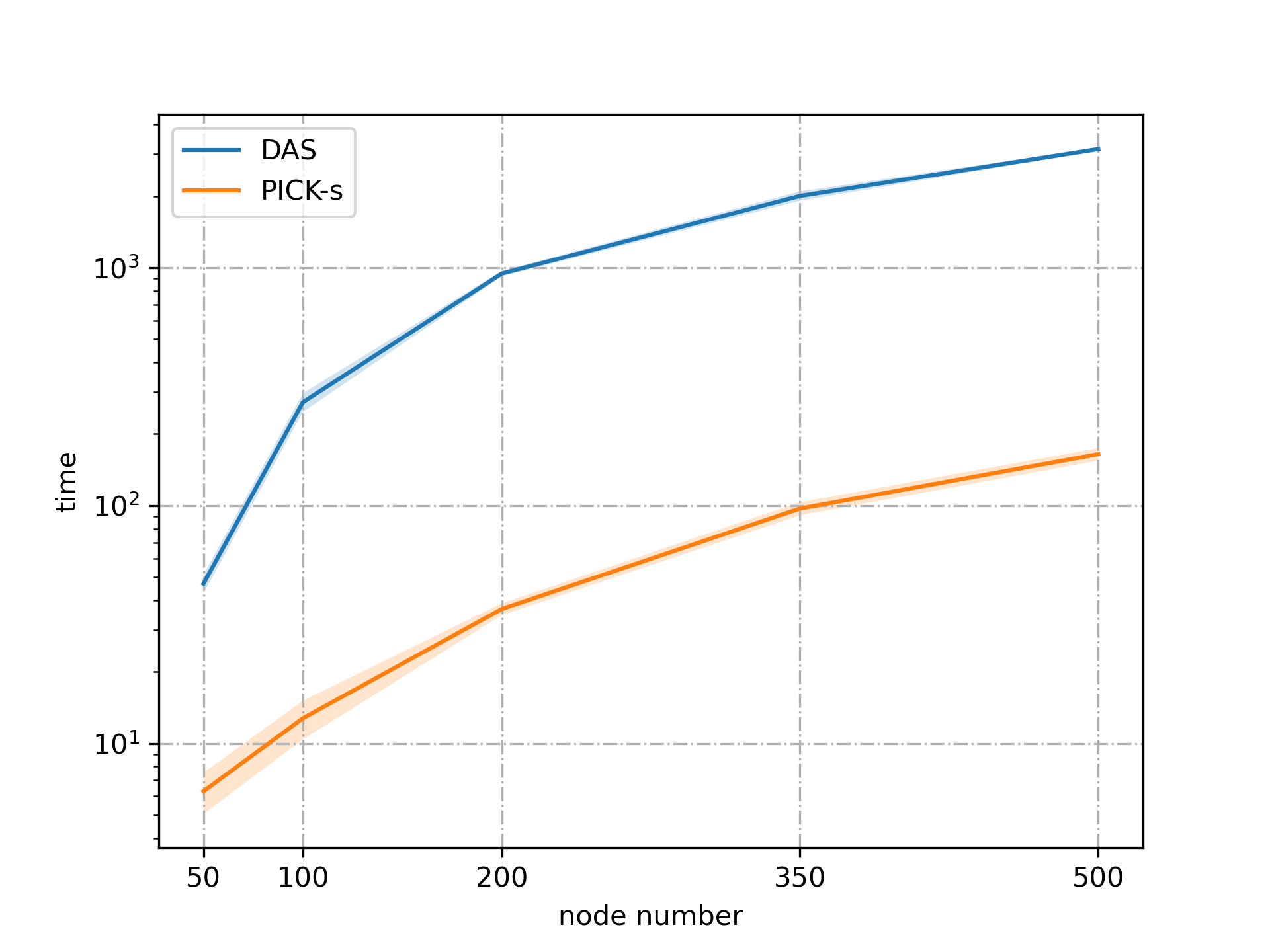}
  }
    \subfloat[ER4]
  {      \label{fig:sta-gp-time-b-subfig3}\includegraphics[width=0.3\linewidth]{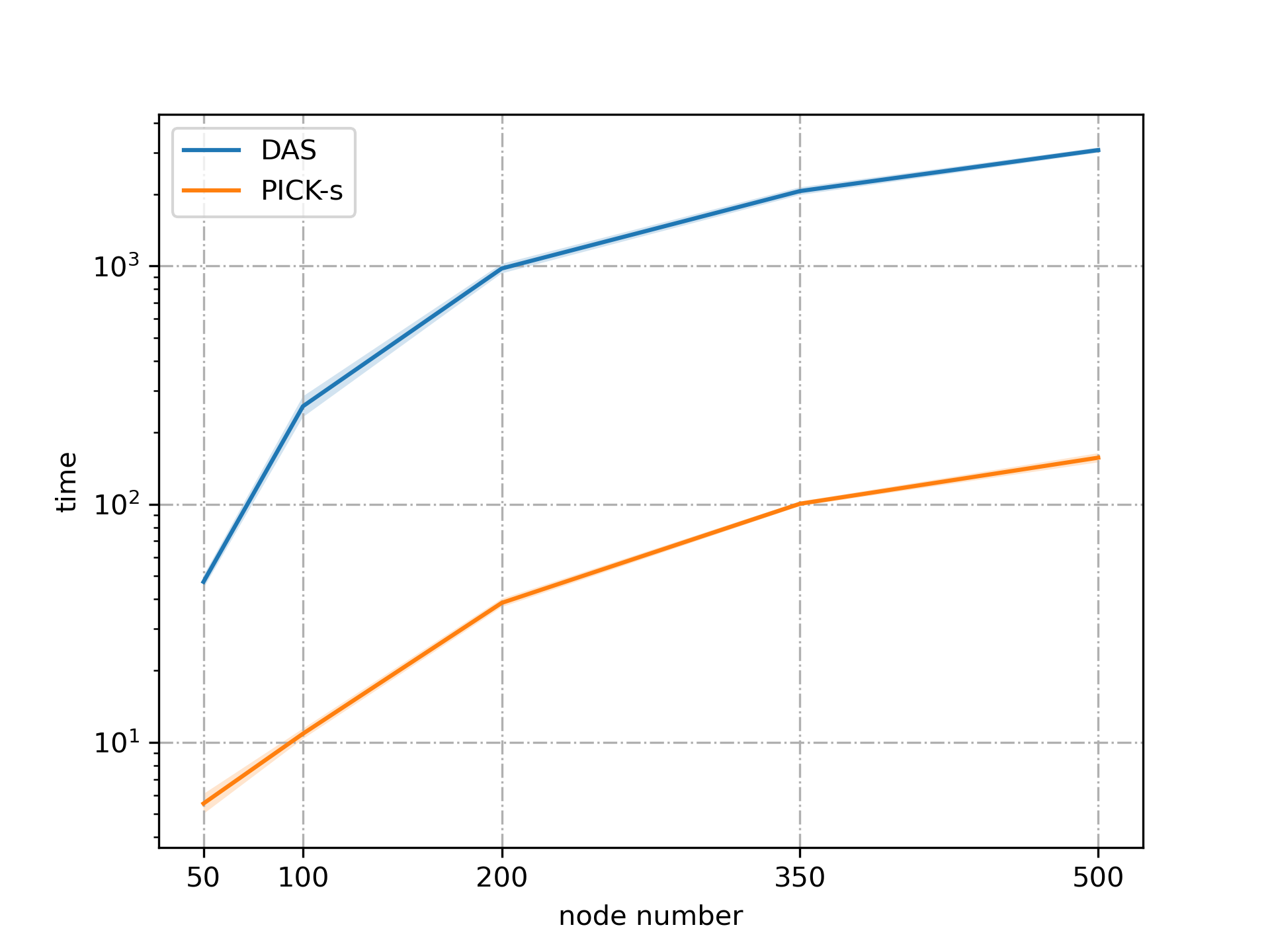}
  }
  \caption{Running time for predicted  and ground truth causal graph with link function $f_{i}^{(t)}$ generated by sampling Gaussian process with a unit bandwidth RBF kernel. The upper and lower rows show the results for low dimension and high dimension respectively.}   
  \label{fig:static-time}          
\end{figure}

\subsection{Real data}
\subsubsection{Temporal data}
We also apply our method to the ``review and business data in 2021'' from the \textbf{yelp} dataset \citep{yelp}. For this dataset, we construct a business graph by distance as restaurants with close distance might influence each other \citep{fan2023directed}. 

We assume that each month is a time step and choose the restaurant data in 2021 in the city of Philadelphia as it is one of the cities that has the most restaurants in \textbf{yelp} dataset. With more restaurants that appears in reviews of each month in 2021 we could ensure that there is no missing data. Similarly, we use the average star rating in reviews of each month as star data and the quantity of reviews as customer flow data. 
For this scenario, we focus on the causal relationship of restaurant rating, opening hours and customer flow. A toy Structural Causal Model (SCM) introduced in \citet{anderson2012learning} includes a directed causal edge from star rating to customer flow. Moreover, restaurant opening hours are another significant factor influencing customer flow. We use these findings as ground truth to evaluate our predicted intra-snapshot causal graph. The evaluation results, presented in Table~\ref{tab:dynamic-real-data}, demonstrate that the causal graph predicted by our algorithm is the most accurate.

\begin{table}[htbp]
\centering
\caption{SHD and edge discovery results for \textsc{PICK-t} compared with baselines
 on business graph in Philadelphia}
\label{tab:dynamic-real-data}
\resizebox{\linewidth}{!}{
\begin{tabular}{c c cc}
\toprule
Method &SHD& star rating $\to$ customer flow&open hour $\to$ customer flow  \\
\midrule
\textsc{PICK-t}(Ours) &0 &\Checkmark&\Checkmark\\
\textsc{GraphNOTEARS} & 2&\ding{55} &\ding{55}\\
\textsc{DyNOTEARS}&2&\ding{55}&\ding{55}\\
\textsc{LASSO}&2&\Checkmark&\ding{55}\\
\bottomrule
\end{tabular}
}
\end{table}

\subsubsection{Static data}
For the performance evaluation of the static model, we compare the algorithms on a commonly used real-world dataset \textbf{Schs} \citep{sachs2005causal} and 10 pseudo-real-world datasets sampled from \textbf{SynTReN} generator \citep{van2006syntren}. We repeat experiment in \textbf{Schs} dataset for 10 times to report the average and standard error of running time and the SHD, together with results for \textbf{SynTReN} in Table~\ref{tab:real-data-static}. The results demonstrate that our algorithm ahieves superior accuracy and incurs the lowest running cost in real-world scenarios.  

\begin{table}[htbp]
\caption{SHD and runtime for \textsc{PICK-t} compared with baselines on \textbf{Schs} and \textbf{SynTReN} datasets}
\label{tab:real-data-static}
\centering
\resizebox{\linewidth}{!}{
\begin{tabular}{ccccc}
\toprule
&\multicolumn{2}{c}{\textbf{Schs}}&\multicolumn{2}{c}{\textbf{SynTReN}} \\
\cmidrule(lr){2-3}\cmidrule(lr){4-5}
&SHD&Time (s)&SHD&Time (s) \\
\midrule
\textsc{CAM} &$14$&$59.81\pm2.35$&$69.00\pm13.87 $&$513.04\pm11.64$\\
\textsc{SCORE}&$\bf 12$&$5.94\pm0.75$&$48.10\pm8.96 $&$19.21\pm4.03$\\
\textsc{DAS}&$14$&$11.85\pm6.04$&$66.00\pm8.40 $&$12.85\pm2.99$\\
GranDAG&$15$&$125.34\pm3.50$&$29.30\pm8.04 $&$145.24\pm1.35$\\ 
\textsc{PICK-s} (Ours) &  $13$&$\bf 3.77\pm0.42$&$\bf 22.40\pm4.51 $& $\bf 2.52\pm0.10$\\

\bottomrule
\end{tabular}
}
\end{table}

\section{Conclusion}
\label{sec:conclusion}
In this work, we propose a new subroutine for identifying parents of leaf nodes with the variance of score value. This subroutine has the potential to greatly reduce computation time of the pruning step. Furthermore, through partial score matching, we combine the score-matching-based algorithm developed in \citet{score-matching} and our new parent-identification subroutine to recover the causal structure when analyzing temporal data with network interference, going one step beyond the classical static, i.i.d. setup. 
Finally, in both synthetic and real datasets, our proposed method exhibits improved finite-sample performance over several existing benchmarks based on various evaluation metrics with low computational cost. 
\bibliography{ref}
\bibliographystyle{tmlr}

\appendix
\section{A brief description of CAM method} 

Based on generative additive model selection, this method is mainly used to remove spurious edges of dense graph induced by topological order.

The first step of CAM method is PNS, a neighbors selection method based on \citet{meinshausen2006high}. In the score matching based methods, PNS is preformed by fitting for each leaf node $X_j$ versus all the other variables left and choose $K$ possible predictor variables as possible parents of $X_j$. This method is implemented  with a boosting method in \citet{buhlmann2003boosting,buhlmann2007boosting}. The total computation complexity of PNS is $\mathcal{O}(dt\mathrm{r}(n,d))$ with $t$ models needed to be fit for each leaf node and $\mathrm{r}(n,d)$ the model fitting complexity, amounting to $\mathcal{O}(nd^2)$ when choosing the Iteratively Reweighted Least Squares method in \citet{minka2003comparison}.

Following PNS, the CAM-pruning is a final pruning method by applying hypothesis tesing for model selection and could thus be used to reduce the number of false positives.

\section{Deferred algorithms}
\label{sec:deffer alg}

In this section, we present the full algorithm that return causal relationship matrices for temporal and static data input in Algorithm~\ref{alg:dynamicDAG} and Algorithm~\ref{alg:staticDAG}. The pruning procedure based on GAM is illustrated in Algorithm~\ref{alg:dynamic-pruning}.

\begin{algorithm}
	\caption{dynamicalDAG}
	\label{alg:dynamicDAG}
\hspace*{0.02in} {\bf Input:} $(X^{(T)},\dots,\ X^{(1)})$, $(A^{(T)},\ \dots,\ A^{(1)})$\\
\hspace*{0.02in} {\bf Output:} Estimated $\hat{W}^{(t)}$ and $\hat{P}^{(t)}$ at time $t$

	\begin{algorithmic}[1]
        \STATE WList=[]
        \STATE PList=[]
        \FORALL{$t \in\ \{T\dots,T-p\} $}
        \STATE $Xt\gets(X^{(t)},\ \hat{X}^{(t-1)},\dots,\ \hat{X}^{(t-p)})$
        \STATE $At\gets(A^{(t)},\ A^{(t-1)},\dots,\ A^{(t-p)})$
       \STATE $Wt,Pt=\mathrm{PICK-t}(Xt,At)$
       \STATE WList$\gets$WList$+Wt$, PList$\gets$PList$+Pt$      
    \ENDFOR
        \STATE $W=\vmathbb{1}_{\geq \tau_w}(\mathrm{average(WList)})$ 
        \STATE $P=\vmathbb{1}_{\geq \tau_p}(\mathrm{average(PList)})$ ($\tau_w$ and $\tau_p$ are threshold hyperparameters)
	\end{algorithmic}  
\end{algorithm}

\begin{algorithm}
	\caption{PICK-s}
	\label{alg:staticDAG}
\hspace*{0.02in} {\bf Input:} $X$\\
\hspace*{0.02in} {\bf Output:} Estimated $\hat{W}^t$ and $\hat{P}^{t}$ at time $t$

	\begin{algorithmic}[1]
        \STATE ActiveNode=$\{0,1,\dots,d\}$
        \STATE LeafNode$\gets$$\{\}$
        \STATE WDense$\gets$$d\times d$ zero matrix
        \FORALL{$i \in\ \{1,2\dots,d\} $}
        \STATE $\mathrm{Hess},\mathrm{newScore}=\mathrm{hess}(X_{\mathrm{ActiveNode}})$ and leave out the last $p-r$ columns of $H$
       \STATE leaf$\gets$argmin(H.var(dim=1))
       \STATE ParSet$\gets$$\{i:\Var(\mathrm{newScore})_i<\Var(\mathrm{oldScore})_i\}$
       \STATE Set leaf column ParSet row of WDense $1$.
       \STATE LeafNode.append(ActiveNode(leaf))
       \STATE ActiveNode.remove(ActiveNode(leaf))
       \STATE $\mathrm{oldSCORE}$<-$\mathrm{newSCORE}$ with $\mathrm{leaf}$ column removed
       \STATE remove the ${\rm ActiveNode(leaf)}$-th column in $X$      
    \ENDFOR
        \STATE LeafNode.reverse()
        \STATE Get estimated $\hat{W}$ by pruning algorithm
	\end{algorithmic}  
\end{algorithm}

\begin{algorithm}
	\caption{Pruning Method}
	\label{alg:dynamic-pruning}
\hspace*{0.02in} {\bf Input:} covariates $X\coloneqq(X_1,\dots,X_k)$, labels $Y$, threshold $\alpha$\\
\hspace*{0.02in} {\bf Output:} selected variables
	\begin{algorithmic}[1]
        \STATE $\hat{\beta} = \texttt{gam} (Y \sim X)$
        \FORALL{$i \in\ \{1,2,\dots,k\} $}
        \IF{$\frac{\hat{\beta}_i^2}{\Var \hat{\beta}_i}>\alpha$}
        \STATE $X_i$ is selected   
    \ENDIF
        \ENDFOR
	\end{algorithmic}  
\end{algorithm}

\section{Proofs and theoretical analysis}
\label{sec:pf and theoretical analysis}

\subsection{Proof of Lemma~\ref{lemma:hess-leaf}}
\begin{proof}
Without loss of generality, we assume that $p = 1$. To start with, we denote the joint probability density function of $(X^{(t)},X^{(t-1)})$ and the marginal probability density function of $X^{(t-1)}$ as $g^{(t)}(x^{(t)},x^{(t-1)})$ and $h^{(t-1)}(x^{(t-1)})$, respectively. We first decompose the joint probability density function as 
\begin{equation}\label{eq:bayes decomposition}
   g^{(t)}(x^{(t)},x^{(t-1)})=h^{(t-1)}(x^{(t-1)})p^{(t,t-1)}(x^{(t)}|x^{(t-1)})
\end{equation}
where $p^{(t,t-1)}(x^{(t)}|x^{(t-1)})$ is the conditional density function. We realize that 
\begin{equation}
    \frac{\partial\log g^{(t)}}{\partial x^{(t)}_i}(x^{(t)},x^{(t-1)})=\frac{\partial\log p^{(t,t-1)}}{\partial x^{(t)}_i}(x^{(t)}|x^{(t-1)})
\end{equation}
from the generation process we have
\begin{equation}
    p^{(t,t-1)}(x^{(t)}|x^{(t-1)})=\prod\limits_{i=1}^d\frac{1}{\sigma_i\sqrt{2\pi}}\exp\left\{-\frac{(x_i^{(t)}-f^{(t)}_i((x_{\pa(i)}^{(t)},x_{\pa(i)}^{(t-1)})))^2}{2\sigma_i^2}\right\}.
\end{equation}
For simplicity of notation, $f_i^{(t)}$ will hereafter be cited within the text as EG.
Therefore, we could find that,
\begin{equation}\label{eq:score function-leaf}
    \frac{\partial\log g^{(t)}}{\partial x^{(t)}_i}(x^{(t)},x^{(t-1)})=-\frac{x_i^{(t)}-f_i((x_{\pa(i)}^{(t)},x_{\pa(i)}^{(t-1)}))}{2\sigma_i^2}
\end{equation}
when $i$ is a leaf node
and that 
\begin{equation}\label{eq:score function-nonleaf}
\begin{split}
        \frac{\partial\log g^{(t)}}{\partial x^{(t)}_i}(x^{(t)},x^{(t-1)})&=-\frac{x_i^{(t)}-f_i((x_{\pa(i)}^{(t)},x_{\pa(i)}^{(t-1)})))}{\sigma_i^2}\\
        &+\sum\limits_{j\in\ch(i)}\frac{x_j^{(t)}-f_j((x_{\pa(j)}^{(t)},x_{\pa(j)}^{(t-1)}))}{\sigma_j^2}\frac{\partial f_j}{\partial x^{(t)}_i}((x_{\pa(j)}^{(t)},x_{\pa(j)}^{(t-1)}))
    \end{split}
\end{equation}
when $i$ is not a leaf node. Hence, we could compute the variance of second-order partial derivative respectively as 
\begin{equation}
    \Var \left(\frac{\partial^2\log g^{(t)}}{\partial x^{(t)2}_i}(x^{(t)},x^{(t-1)})\right)=0
\end{equation}
and 
\begin{equation}\label{eq:var of non-leaf}
    \Var \left(\frac{\partial^2\log g^{(t)}}{\partial x^{(t)2}_i}(x^{(t)},x^{(t-1)})\right)\geq\frac{1}{\sigma_i^2}+\sum\limits_{j\in\ch(i)}\frac{1}{\sigma_j^2}\mathbb{E}\left(\frac{\partial^2f_j}{\partial x_i^{(t)2}}\right)^2.
\end{equation}
As we have assumed that $f_j$ is not a linear function for any node $j$, $\mathbb{E}\left(\frac{\partial^2f_j}{\partial x_i^{(t)2}}\right)^2$ is strictly positive, which grant that the variance in \eqref{eq:var of non-leaf} is strictly positive. Denote $\frac{\partial^2\log g^{(t)}}{\partial x^{(t)2}_i}(x^{(t)},x^{(t-1)})$ as $J_i$, and then we could have node $i$ is a leaf if and only $\Var(J_i)=0$.

Similarly, we could prove the second part of this lemma.
\end{proof}

\subsection{Proof of Theorem~\ref{thm:parent node}}
\begin{proof}
We consider the situation for $t^\prime=t$ and $t^\prime\in\{t-1,\dots,t-p\}$. Without loss of generality, we assume $p=1$. 

When $t^\prime=t$, we could obtain $s_{j,t}^{(t)}\coloneqq  \frac{\partial\log g^{(t)}}{\partial x^{(t)}_i}(x^{(t)},x^{(t-1)})$ by \eqref{eq:score function-nonleaf}. From the generation procedure, we could know that $X_i^{(t)}-f_i((X_{\pa(i)}^{(t)},X_{\pa(i)}^{(t-1)}))\equiv Z_i^{(t)}$ and $Z_i^{(t)}\sim\mathcal{N}(0,\sigma_i^2)$ for any node $i$ and time step $t$ and that $Z_i^{(t)}$ are independent to each other. Besides, $Z_i^{(t)}$ is independent with $X_{i^\prime}^{(t^\prime)}$ for any $t^\prime\in \{t,t-1\}$ and non-descendant node $i^\prime$. It follows that
\begin{equation}\label{eq:cov1-parent-node}
    \mathrm{Cov}\left(\frac{Z_i^{(t)}}{\sigma_i^2},\frac{Z_j^{(t)}}{\sigma_j^2}\frac{\partial f_j}{\partial x^{(t)}_i}((X_{\pa(j)}^{(t)},X_{\pa(j)}^{(t-1)}))\right)=0
\end{equation}
and that 
\begin{equation}\label{eq:cov2-parent-node}
    \mathrm{Cov}\left(\frac{Z_{j_1}^{(t)}}{\sigma_{j_1}^2}\frac{\partial f_{j_1}}{\partial x^{(t)}_i}((X_{\pa(j_1)}^{(t)},X_{\pa(j_1)}^{(t-1)})),\frac{Z_{j_2}^{(t)}}{\sigma_{j_2}^2}\frac{\partial f_{j_2}}{\partial x^{(t)}_i}((X_{\pa(j_2)}^{(t)},X_{\pa(j_2)}^{(t-1)}))\right)
\end{equation}
for any $j_1\neq j_2\in\ch(i)$.
Henceforth, we could present its variance in the following equation
\begin{equation}\label{eq:var-score-t}
    \Var(s_{i,t}^{(t)}((X^{(t)},X^{(t-1)})))=\frac{1}{\sigma_i^2}+\sum\limits_{j\in\ch(i)}\frac{1}{\sigma_i^2}\mathbb{E}\left(\frac{\partial f_j}{\partial x^{(t)}_i}((X^{(t)}_{\pa(j)},X^{(t-1)}_{\pa(j)}))\right)^2.
\end{equation}
As we assume that $f_j$ is not linear, we could know that $\mathbb{E}\left(\frac{\partial f_j}{\partial x^{(t)}_i}((X^{(t)}_{\pa(j)},X^{(t-1)}_{\pa(j)}))\right)^2$ would be positive. Therefore, if leaf node $l$ is a child of node $i$, after removing node $l$, the variance of $s_{i,t}^{(t)}$ would strictly decrease. If node $i$ is not a parent, the variance of $s_{i,t}^{(t)}$ would remain unchanged. 

When $t^{\prime}=t-1$, we would present $s_{j,t^{\prime}}^{(t)}$ by \eqref{eq:bayes decomposition}
\begin{equation}
    \frac{\partial\log g^{(t)}}{\partial x_i^{(t-1)}}(x^{(t)},x^{(t-1)})=\frac{\partial\log p^{(t,t-1)}}{\partial x^{(t-1)}_i}(x^{(t)}|x^{(t-1)})+\frac{\partial\log h^{(t-1)}}{\partial x^{(t-1)}_i}(x^{(t-1)})
\end{equation}

From model assumption, we know that 
\begin{equation}
    \frac{\partial\log p^{(t,t-1)}}{\partial x^{(t-1)}_i}(x^{(t)}|x^{(t-1)})=\sum\limits_{j\in \ch_t (i)}\frac{\partial\log p^{(t,t-1)}(x^{(t)}_j|x^{(t-1)})}{\partial x_i^{(t-1)}}\frac{\partial f_j}{\partial x^{(t)}_i}((X_{\pa(j)}^{(t)},X_{\pa(j)}^{(t-1)})),
\end{equation}
and that
\begin{equation}
    \frac{\partial\log p^{(t,t-1)}(x^{(t)}_j|x^{(t-1)})}{\partial x_i^{(t-1)}}=s^{Z}_{j,t}{\partial x_i^{(t-1)}}
\end{equation}
where $s^{Z}_{j,t}$ denotes as the score function of noise $Z_j^{(t)}$.
Therefore, we could have
\begin{equation}
     \frac{\partial\log g^{(t)}}{\partial x_i^{(t-1)}}(X^{(t)},X^{(t-1)})=\frac{\partial\log h^{(t-1)}}{\partial x^{(t-1)}_i}(X^{(t-1)})+\sum\limits_{j\in\ch_t(i)}s_{j,t}^Z\frac{\partial f_j}{\partial x^{(t)}_i}((X_{\pa(j)}^{(t)},X_{\pa(j)}^{(t-1)})).
\end{equation}
As $Z_j^{(t)}$ is independent with $X^{(t-1)}$,$X_{\pa(j)}^{(t)}$ and $X_{\pa(j)}^{(t-1)}$, we could have 
\begin{equation}\label{eq:cov3-parent-node}
    \mathrm{Cov}\left(\frac{\partial\log h^{(t-1)}}{\partial x^{(t-1)}_i}(X^{(t-1)}),s_{j,t}^Z\frac{\partial f_j}{\partial x^{(t)}_i}((X_{\pa(j)}^{(t)},X_{\pa(j)}^{(t-1)}))\right)=0
\end{equation}
Together with \eqref{eq:cov1-parent-node} and \eqref{eq:cov2-parent-node}, we could offer the variance of $s_{i,t-1}^{(t)}$
\begin{equation}\label{eq:var-score-t-1}
\Var(s_{i,t-1}^{(t)}((X^{(t)},X^{(t-1)})))=\Var\left(\frac{\partial\log h^{(t-1)}}{\partial x^{(t-1)}_i}(X^{(t-1)})\right)
        +\sum\limits_{j\in\ch_t(i)}\frac{1}{\sigma_i^2}\mathbb{E}\left(\frac{\partial f_j}{\partial x^{(t)}_i}((X^{(t)}_{\pa(j)},X^{(t-1)}_{\pa(j)}))\right)^2
\end{equation}
where $\ch_t(i)$ means the child at time step $t$ of node $i$. Similarly, we could conclude that if and only if leaf node $l$ is a child of node $i$, after removing node $l$, the variance of $s_{i,t}^{(t)}$ should strictly decrease. 

\end{proof}

The proof of Corollary~\ref{cor:parent node static} follows easily from the proof of Theorem~\ref{thm:parent node}, and is hence omitted.

\subsection{Proof of Proposition~\ref{thm:score var convergence}}
\begin{proof}
\label{pf:thm-score-var-convergence}
    By the marginal distribution assumption in Proposition~\ref{thm:score var convergence} and Assumption~\ref{asmp:sparse graph}, we could know that only $o(n^2)$ terms in $\sum\limits_{1\leq i\leq j\leq n}\mathbb{E}(||X^{(t)}_i-Y^{{(t)}}_i||_2^2)$ would be nonzero and each of them is bounded, which leads to 
    \begin{equation}\label{eq:cor bound}
        \sum\limits_{1\leq i\leq j\leq n}\mathbb{E}(||(X^{(t)}_i-Y^{{(t)}}_i)||_2^2)=o(n^2).
    \end{equation}
   Therefore, we have $\sum\limits_{1\leq i\leq j\leq n}\mathbb{E}(||\Bar{X}^{(t)}_i-\Bar{Y}^{{(t)}}_i||_2^2)=o(n^2)$. For simplicity, we substitute $\Bar{X}^{(t)}$ and $\Bar{Y}^{(t)}$ by $X^{(t)}$ and $Y^{(t)}$ in the proof.

    For the first part, we denote the hessian function of node $i$ as $h_i$ and then the variance of the hessian for each node is expressed as 
    \begin{equation}        \sigma(h_i({X}^{(t)}))=\frac{\sum\limits_{k=1}^nh_i(X^{(t)}_k)^2}{n}-\left(\frac{\sum\limits_{k=1}^nh_i(X^{(t)}_k)}{n}\right)^2
    \end{equation}
    In the score matching and Jacobian estimation method, we use RBF kernel as our test function inner product which provides smoothness and boundedness of $h_i$. By Lemma~\ref{lemma:function converge}, we have
    \begin{equation}        \lim\limits_{n\to\infty}\mathbb{P}\left(\left|\frac{\sum\limits_{k=1}^nh_i(X^{(t)}_k)^2}{n}-\frac{\sum\limits_{k=1}^nh_i(Y^{(t)}_k)^2}{n}\right|\right)=0
    \end{equation}
    and that
    \begin{equation}        \lim\limits_{n\to\infty}\mathbb{P}\left(\left|\frac{\sum\limits_{k=1}^nh_i(X^{(t)}_k)}{n}-\frac{\sum\limits_{k=1}^nh_i(Y^{(t)}_k)}{n}\right|\right)=0
    \end{equation}
    which follows that
    \begin{equation}
        \lim\limits_{n\to\infty}\mathbb{P}(|\sigma(h_i(X^{(t)}))-\sigma(h_i(Y^{(t)}))|\leq\epsilon)=1
    \end{equation}
    for any node $i$ and positive constant $\epsilon$.
Besides, identity \eqref{stein} motivates the following estimator of the score function:   
\begin{equation}\label{eq:mc1}
    -\frac{1}{n}\sum\limits_{i=1}^nh(\Bar{x}^{(t)}_k)+err=\frac{1}{n}\sum\limits_{i=1}^n\nabla h(\Bar{x}^{(t)}_k)
\end{equation}
and similarly we have
\begin{equation}\label{eq:monte-carlo2}
\begin{split}
    \frac{1}{n}\sum\limits_{k=1}^nq^{(t)}(\Bar{x}^{(t)}_k)\mathrm{diag}(\nabla^2\log \Bar{p}^{(t)}(\Bar{x}^{(t)}_k))^{\top}&+err=\frac{1}{n}\sum\limits_{k=1}^n(\nabla^2_{\mathrm{diag}}q^{(t)}(\Bar{x}^{(t)}_k)-q^{(t)}(\Bar{x}^{(t)}_k)\\&-q^{(t)}(\Bar{x}_k^{(t)})\mathrm{diag}(\nabla\log \Bar{p}^{(t)}(\Bar{x}^{(t)}_k)\nabla\log \Bar{p}^{(t)}(\Bar{x}^{(t)}_k)^\top)).
  \end{split}  
\end{equation}
    Then, we parameterize the hessian and score function and estimate them by optimization from \eqref{eq:mc1} and \eqref{eq:monte-carlo2} using kernel trick. Hence the convergence to real variance of estimated Jacobian and score function for data $\{Y_k^{(t)}\}_{k=1,\dots,n},\forall t\in\{1,\dots,T\}$ is obvious.
\end{proof}

\subsection{Proof of Theorem~\ref{thm:dynamic-consistency}}
Before the proof of Theorem~\ref{thm:dynamic-consistency}, we present a lemma.
\begin{lemma}\label{lemma:function converge}
With $X^{(t)}$ and $Y^{(t)}$ as defined in Proposition~\ref{thm:score var convergence} ($t\in\{1,2,\dots,T\}$), we have 
    \begin{equation}
        \lim\limits_{n\to \infty}\mathbb{P} \left( \left| \frac{1}{n}\sum\limits_{k=1}f(X^{(t)}_k)-f(Y^{(t)}_k) \right| \geq \epsilon \right) = 0
    \end{equation}
    for any smooth bounded function $f$ that there exists a constant $B$ such that $|f|\leq B$, $|f^\prime|\leq B$ and $|f^{\prime\prime}|\leq B$ and any positive constant $\epsilon$.
\end{lemma}

\begin{proof}
    From smoothness condition, we could know that there exists an positive constant $L$ such that
    \begin{equation}
        ||f(X^{(t)}_k)-f(Y^{(t)}_k)||\leq L||X^{(t)}_k-Y^{(t)}_k)||_2.
    \end{equation}
    Then by Chebyshev inequality, we could have
    \begin{equation}
        \mathbb{P}(|\frac{1}{n}\sum\limits_{k=1}f(X^{(t)}_k)-f(Y^{(t)}_k)|\geq \epsilon)\leq \frac{L^2}{n^2\epsilon^2}\mathbb{E}(\sum\limits_{1\leq k_1,k_2\leq n}||X^{(t)}_{k_1}-Y^{(t)}_{k_1}||_2||X^{(t)}_{k_2}-Y^{(t)}_{k_2}||_2).
    \end{equation}
    By \eqref{eq:cor bound}, we know that 
    \begin{equation}
        \lim\limits_{n\to\infty}\mathbb{P}(|\frac{1}{n}\sum\limits_{k=1}f(X^{(t)}_k)-f(Y^{(t)}_k)|\geq \epsilon)=0.
    \end{equation}
\end{proof}
Now, we start the proof of Theorem~\ref{thm:dynamic-consistency}.
\begin{proof}

Denote the topological ordering of intra-snapshot graph at time step $t$ for $\Tilde{X}^{(t)}$ as $\hat{\pi}_t$ and set of all ground truth topological order as $\Pi_t$. 
     From the estimation procedure, we could find that $J^{(t)}\xrightarrow{\mathbb{P}}J^{(t)}$. 
     Denote the leaf in $j$th loop as $l_j$ and by the results in Proposition~\ref{thm:score var convergence} we have
    \begin{equation}
    \mathbb{P}(\mathrm{chi}(l_j)=\emptyset)=1.
    \end{equation}
    Denote set $\{l_j\text{is not a leaf for }G_j\}$ as $A_j$ where $G_j$ is the DAG with leaves found in former $j-1$ loops removed and $G_0\coloneqq G$. Therefore we have 
    \begin{equation}
        \lim\limits_{n\to\infty}\mathbb{P}(A_j)=1
    \end{equation}
    and 
    \begin{equation}        
    \mathbb{P}(\mathop{\cap}\limits_{j=1}^dA_j)\geq 1-\sum\limits_{j=1}^d(1-\mathbb{P}(A_j))
    \end{equation}
    which follows that
    \begin{equation}        \lim\limits_{n\to\infty}\mathbb{P}(\mathop{\cap}\limits_{j=1}^dA_j)=1
    \end{equation}
    and therefore,
        \begin{equation}
        \lim\limits_{n\to \infty}\mathbb{P}(\hat{\pi}\in \Pi)=1
    \end{equation}
    Denoted the intra-snapshot and inter-snapshot matrix from Algorithm~\ref{alg:PICK-t} as $\hat{W}_t$ and $\hat{P}_t$. Similar to the proof of Corollary~\ref{cor:static consistency}, the corresponded estimation of topological order for intra-snapshot graph $\hat{\pi}$ converges to ground truth with probability 1.
    For any leaf node $l$ at time step $t$, denote its parent node set at current and former time steps as $\hat{\pa }_l^{(t)}$. As $G^{(t)}\xrightarrow{\mathbb{P}}S^{(t)}$ where $S^{(t)}$ is the score function, we have $\lim\limits_{n\to\infty}\mathrm{Var}(G^{(t)}_i)=\mathrm{var}(S^{(t)}_i)$ for any node $i$ at current and former time steps. Hence we have
 \begin{equation}
     \lim\limits_{n\to\infty}\mathbb{P}(\hat{\pa }_l^{(t)}=\pa (l))=1
 \end{equation}
 where $\pa (l)$ stands for the parent nodes at current and former time steps which means that the predicted parent nodes converges to the ground truth with probability 1. It follows that after pruning,
 \begin{equation}
        \lim\limits_{n\to \infty}\mathbb{P}(\hat{W}_t=W)=1
 \end{equation}
  \begin{equation}
        \lim\limits_{n\to \infty}\mathbb{P}(\hat{P}_t=P)=1
 \end{equation}
Then we could know that the average predicted results also converge to the ground truth which ends the proof.
\end{proof}

\subsection{Theoretical analysis of Algorithm~\ref{alg:dynamic-pruning}}
\begin{theorem}\label{thm:pruning-converge}
    Denote the edge set selected by Algorithm~\ref{alg:dynamic-pruning} as $\hat{E}_\alpha$, and the ground truth edge set as $E$. Then we have
    \begin{equation}
      \mathbb{P}(\hat{E}_{\alpha}\subseteq E)=1-\alpha
    \end{equation}
    where $\alpha$ is a positive hyperparameter.
\end{theorem}

\begin{proof}[Proof of Theorem~\ref{thm:pruning-converge}]
    For simplicity, we only prove the situation for $p=1$ and without loss of generality, we assume that $\Tilde{X}^{(t)}\in \mathbb{R}^{n\times k}$ where $k$ is the number of covariates.
    The general additive model could be expressed as 
    \begin{equation}
        Y=\sum\limits_{i=1}^kf_i(X_i)
    \end{equation}
    Plugging in data with basis function, we have
    \begin{equation}
        Y_i = \beta^\top s(X_i)+\epsilon_i
    \end{equation}
    where $s(x)$ is the basis function and we have $\mathbb{E}(\epsilon)=0$, $\mathrm{Cov}(\epsilon)=\Sigma$ and $X$ and $\epsilon$ are independent. For simplicity, we could use $X_i$ to substitute $s(X_i)$
    It follows that 
    \begin{equation}
        \hat{\beta}=(X^\top X)^{-1}X^\top Y
    \end{equation}
    where $X=(X_1.\dots,X_n)^\top$ and $Y=(Y_1,\dots,Y_n)^\top$.
    Then we have 
    \begin{equation}
        \mathbb{E}(\hat{\beta}|X)=\beta
    \end{equation}
    and
    \begin{equation}
        \mathrm{Cov}(\hat{\beta}|X)=(X^\top X)^{-1}X^\top X\beta\beta^\top X^\top X(X^\top X)^{-1}-\beta\beta^\top+\sigma_{\epsilon}^2I_d
    \end{equation}

    Then under null hypothesis $H^{i}_0: \beta_i=0$, for any $i\in\{1,2,\dots,k\}$, we have
     \begin{equation}
        \mathbb{E}(\hat{\beta}_i|X)=0
    \end{equation}
    and
    \begin{equation}
        \Var(\hat{\beta}_i|X)=\sigma_{\epsilon}.
    \end{equation}
    By Chebyshev's Inequality, we have
        \begin{equation}
        \mathbb{P}(|\beta_i-\mathbb{E}\beta_i|\geq \alpha)\leq \frac{\Var\beta_i}{\alpha^2}
    \end{equation}
    Then under $H_0$, we could obtain that
    \begin{equation}
        \mathbb{P}\left(\frac{\hat{\beta}_i^2}{\Var \hat{\beta}_i}\geq \frac{1}{\alpha}\right)\leq \alpha.
    \end{equation}
    Since $\hat{\Var}\hat{\beta}\xrightarrow{P}\Var\hat{\beta}$ when sample number $n$ tends to infinity, we could have that $\frac{\hat{\beta}_i^2}{\hat{\Var} \hat{\beta}_i}\xrightarrow{P}\frac{\hat{\beta}_i^2}{\Var \hat{\beta}_i}$ which follows that
    \begin{equation}
        \lim\limits_{n\to\infty}\mathbb{P}\left(\left|\frac{\hat{\beta}_i^2}{\hat{\Var} \hat{\beta}_i}-\frac{\hat{\beta}_i^2}{{\Var} \hat{\beta}_i}\right|\leq \varepsilon\right)=1
    \end{equation}
    for any positive $\varepsilon$. Denote $\left|\frac{\hat{\beta}_i^2}{\hat{\Var} \hat{\beta}_i}-\frac{\hat{\beta}_i^2}{{\Var} \hat{\beta}_i}\right|\geq \varepsilon$ as $A_n$, and $\frac{\hat{\beta}_i^2}{\hat{\Var} \hat{\beta}_i}\geq \frac{1}{\alpha}$ as $B_n$ and then we could deduce that
    \begin{equation}
        \begin{split}
            \mathbb{P}(B_n)&=\mathbb{P}(B_n\cap A_n)+\mathbb{P}(B_n\setminus A_n)\\
           & \leq \mathbb{P}(A_n)+\frac{1}{1/\alpha+\varepsilon}.
        \end{split}
    \end{equation}
    As $\varepsilon$ is any positive constant and $\mathbb{P}(A_n)\to 0$ when $\varepsilon\to 0$, we have
       \begin{equation}
        \mathbb{P}\left(\frac{\hat{\beta}_i^2}{\hat{\Var} \hat{\beta}_i}\geq \frac{1}{\alpha}\right)\leq \alpha.
    \end{equation}
Therefore, the selected edge set is a subset of the ground truth edge set at a confidence level of $1-\alpha$.
\end{proof}

\section*{Additional Experimental Results}
In this section, we present the FDR and TPR results in the synthetic data experiments and results for different nonlinear function types and graph generating types.

\subsubsection{Deffered FDR and TPR  and other function type results numerical experiments}
\label{sec:defer results for temporal data}
\begin{figure}[htbp]
  \centering           
  \subfloat[ER1]   
  {      \label{fig:dy-sin-fdr-w-subfig1}\includegraphics[width=0.3\linewidth]{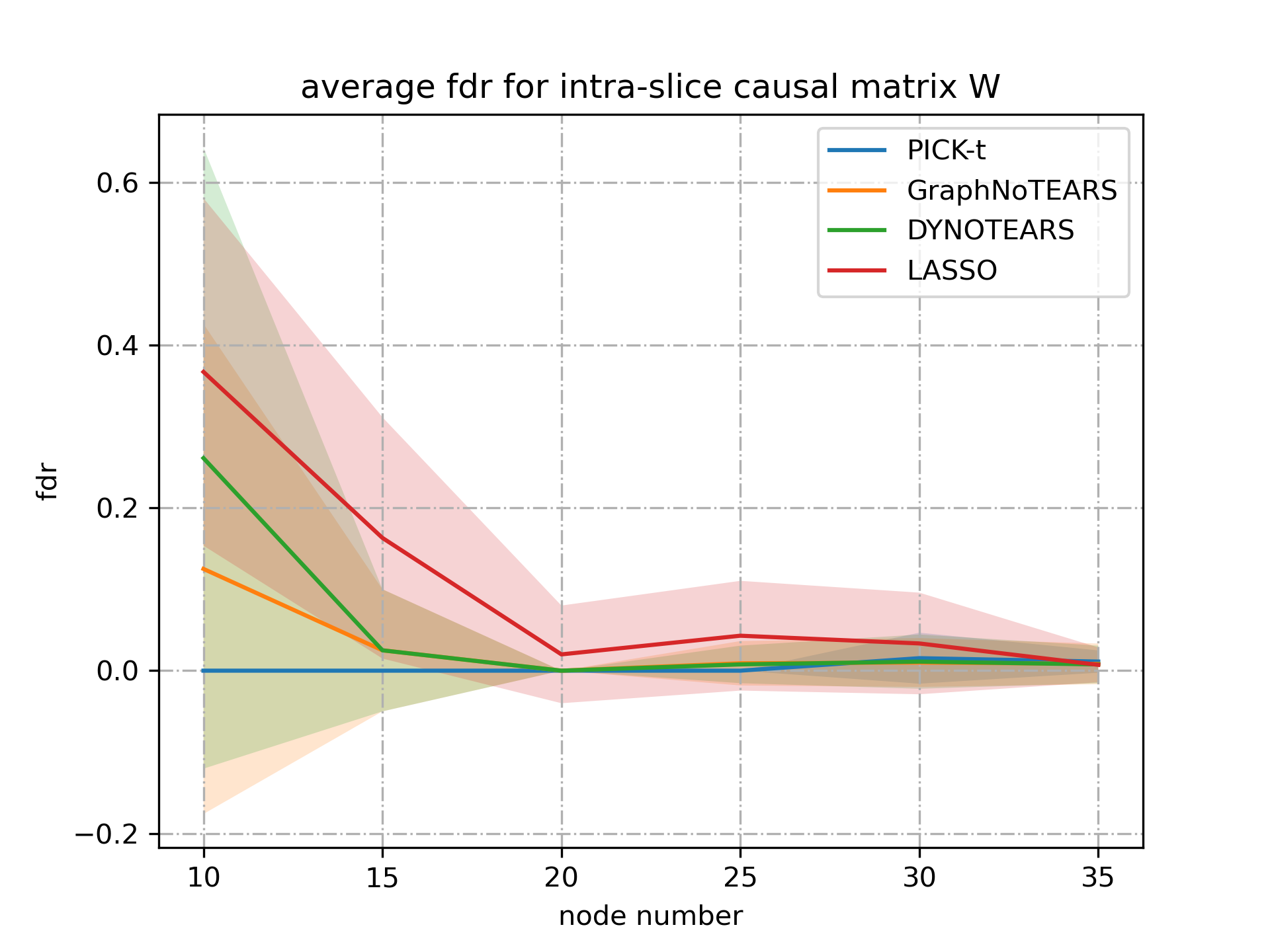}
  }
  \subfloat[ER2]
  {      \label{fig:dy-sin-fdr-w-subfig2}\includegraphics[width=0.3\linewidth]{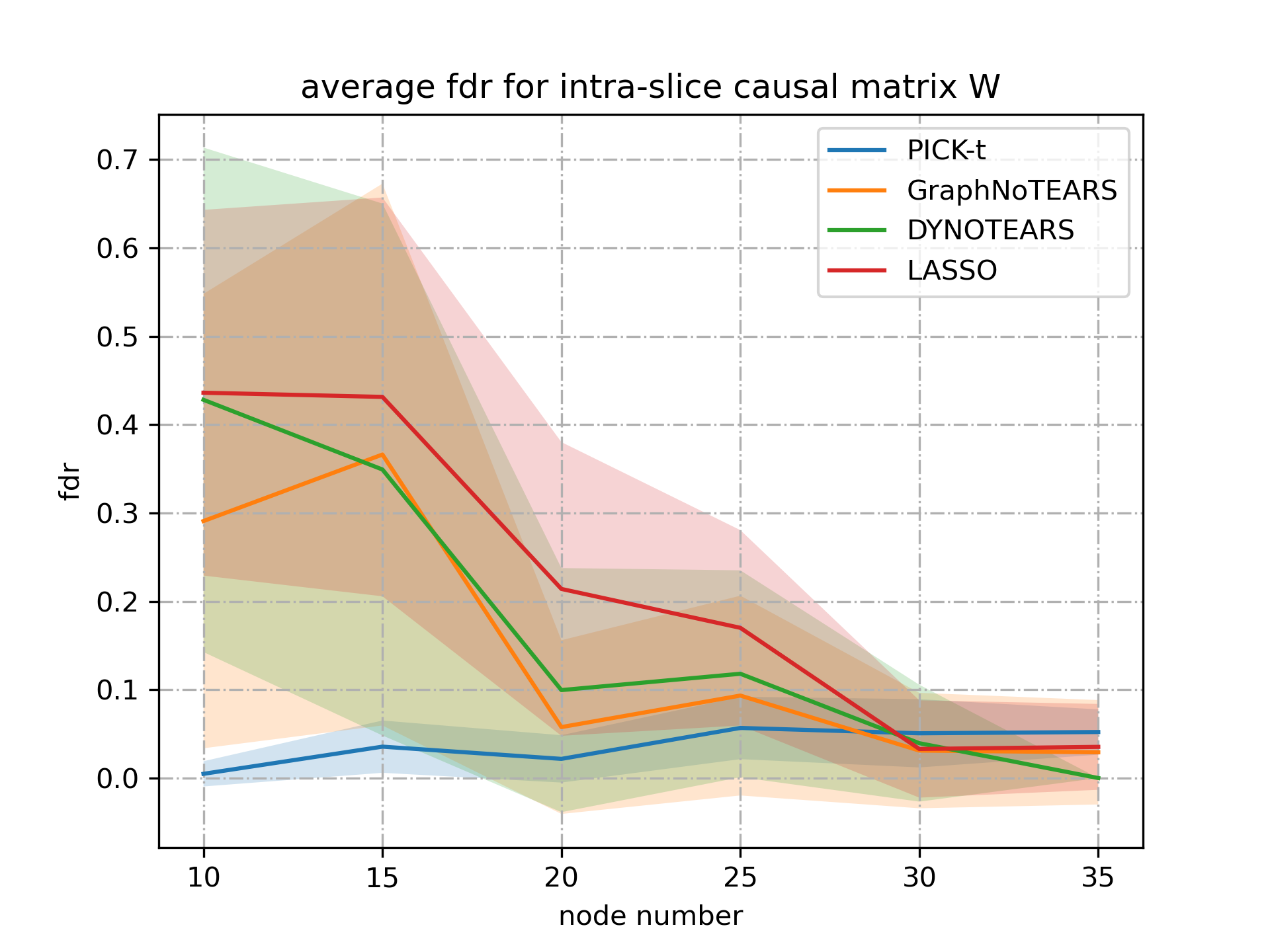}
  }
    \subfloat[ER4]
  {      \label{fig:dy-sin-fdr-w-subfig3}\includegraphics[width=0.3\linewidth]{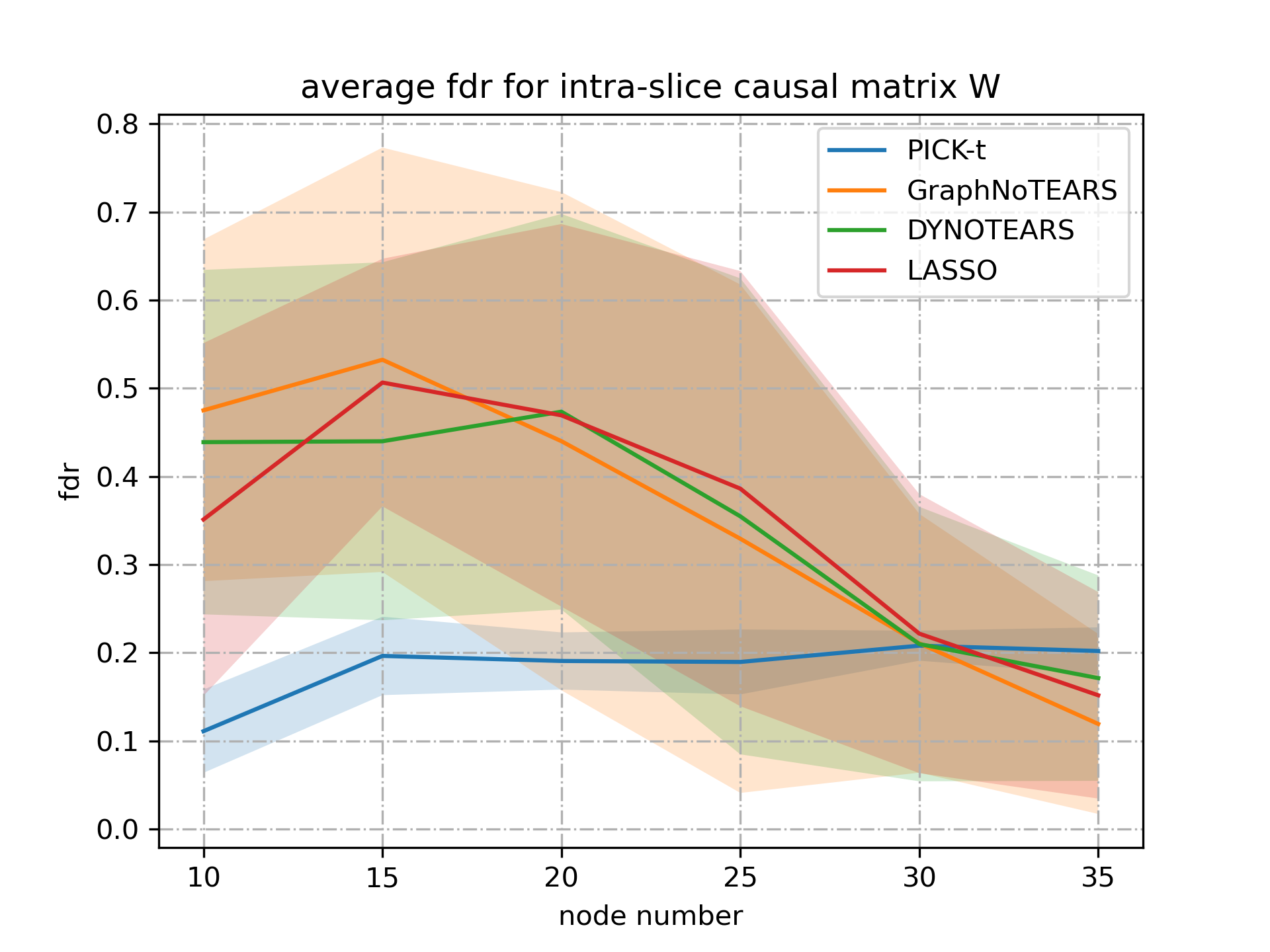}
  }
    \caption{FDR for predicted inter-snapshot causal graph and ground truth inter-snapshot causal graph with link function $f_{i}^{(t)}(x_i)=\sum\limits_{j\in\pa(i)}\sin{x_j}$.}   
  \label{fig:dynamic-fdr-intra-snapshot-sin}          
\end{figure}

\begin{figure}[htbp]
  \centering           
  \subfloat[ER1]   
  {      \label{fig:dy-sin-fdr-p-subfig1}\includegraphics[width=0.3\linewidth]{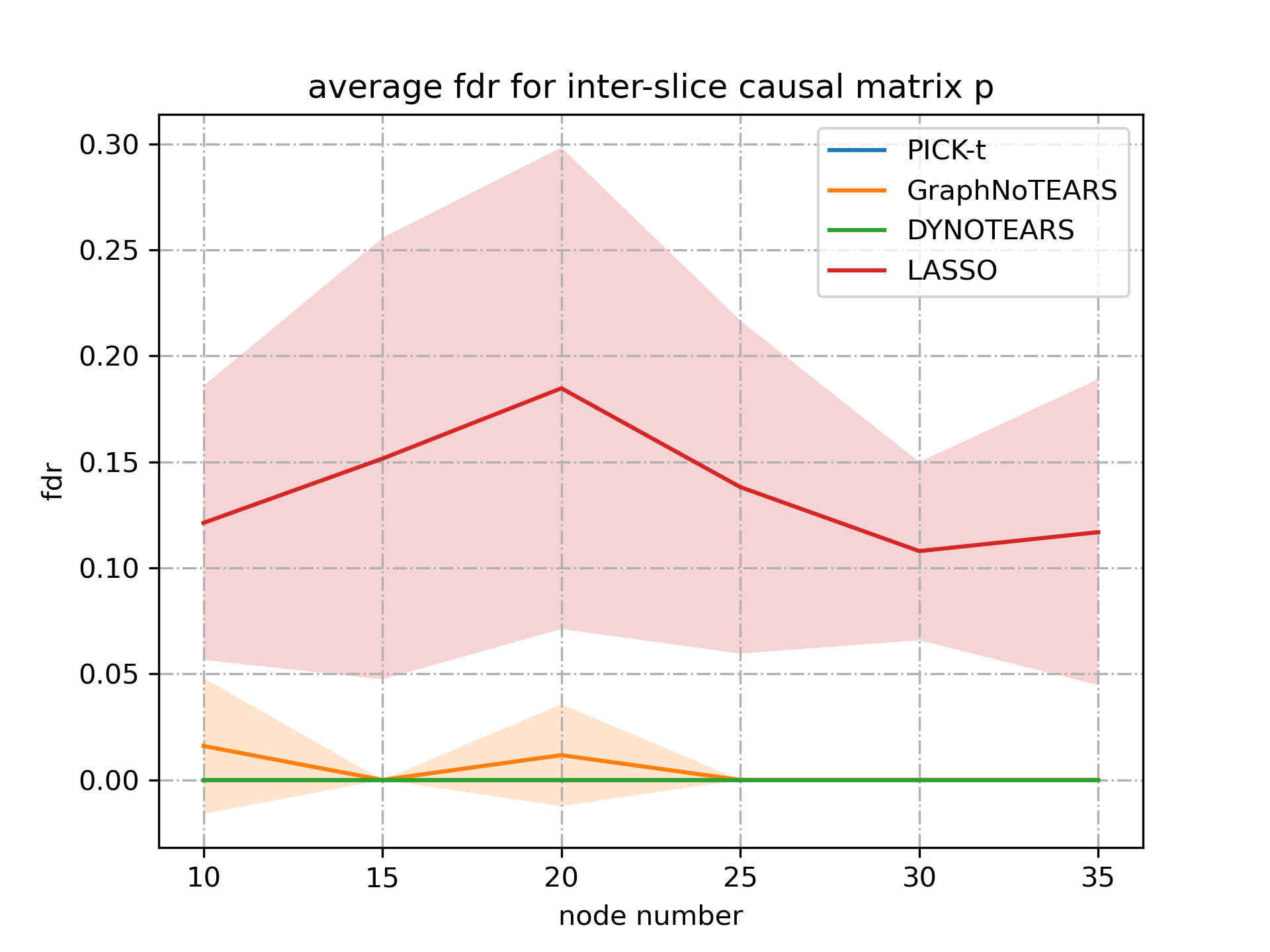}
  }
  \subfloat[ER2]
  {      \label{fig:dy-sin-fdr-p-subfig2}\includegraphics[width=0.3\linewidth]{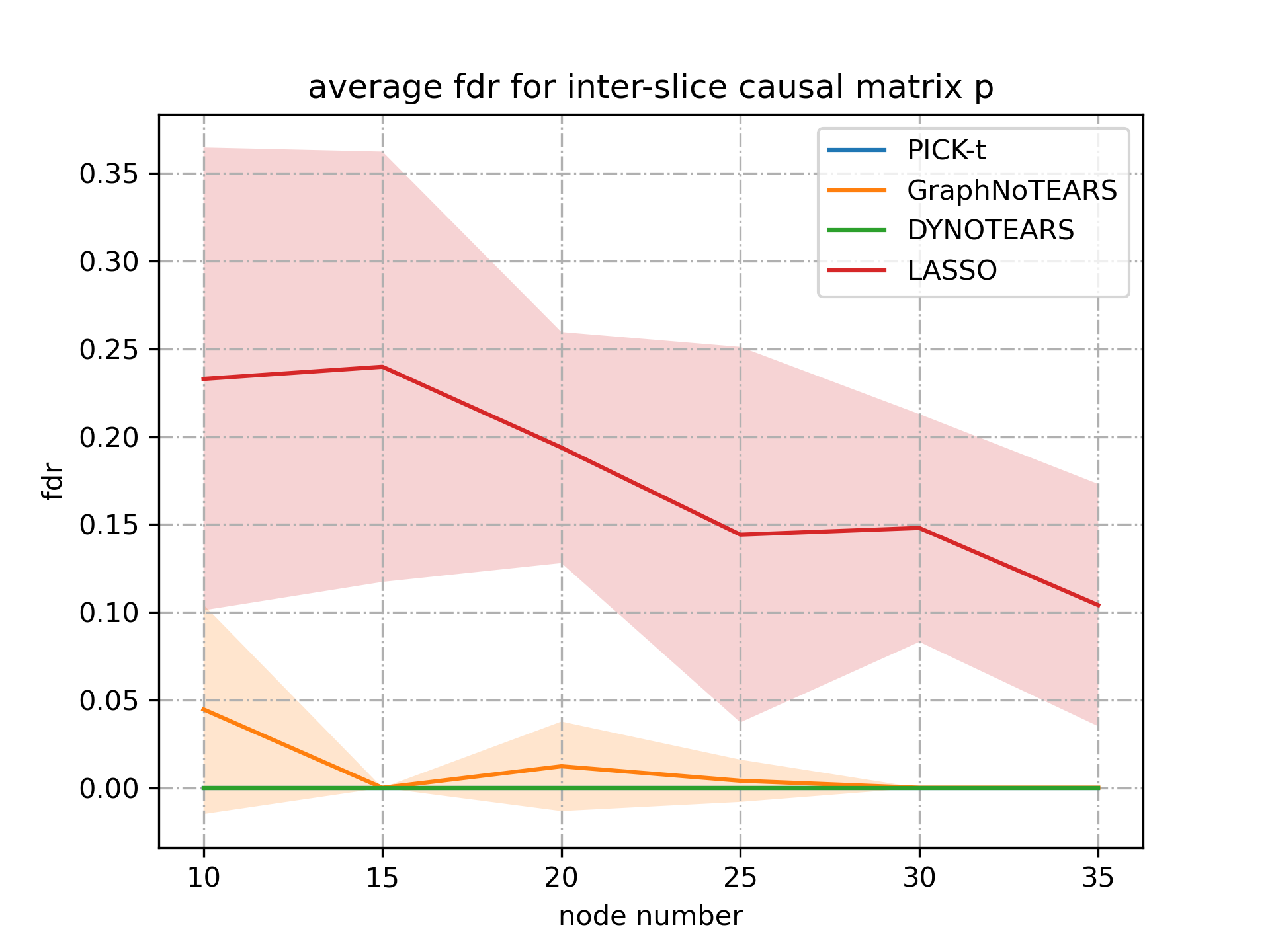}
  }
    \subfloat[ER4]
  {      \label{fig:dy-sin-fdr-p-subfig3}\includegraphics[width=0.3\linewidth]{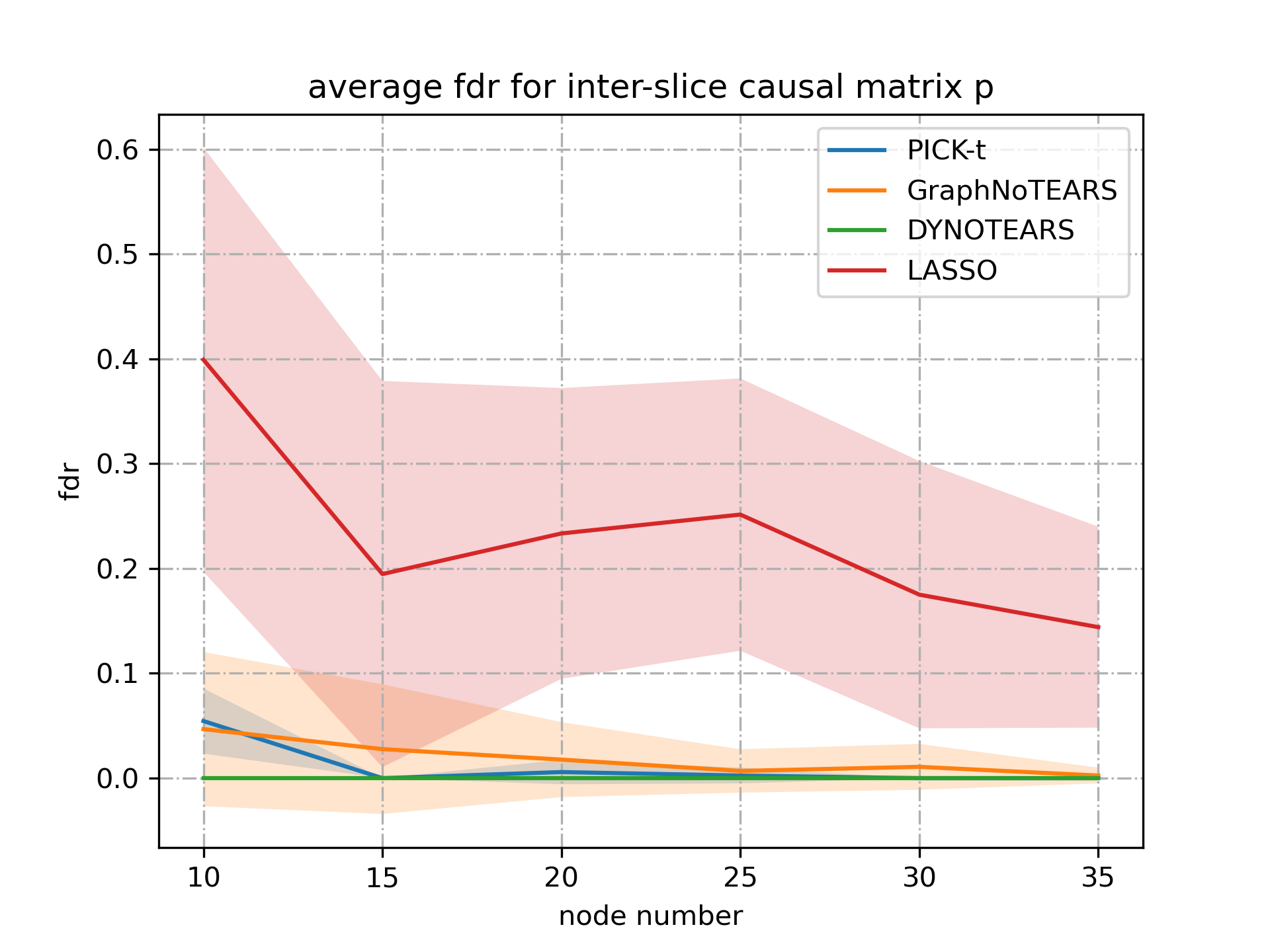}
  }
    \caption{FDR for predicted inter-snapshot causal graph and ground truth inter-snapshot causal graph with link function $f_{i}^{(t)}(x_i)=\sum\limits_{j\in\pa(i)}\sin{x_j}$.}   
  \label{fig:dynamic-fdr-inter-snapshot-sin}          
\end{figure}

\begin{figure}[htbp]
  \centering           
  \subfloat[ER1]   
  {      \label{fig:dy-sin-tpr-w-subfig1}\includegraphics[width=0.3\linewidth]{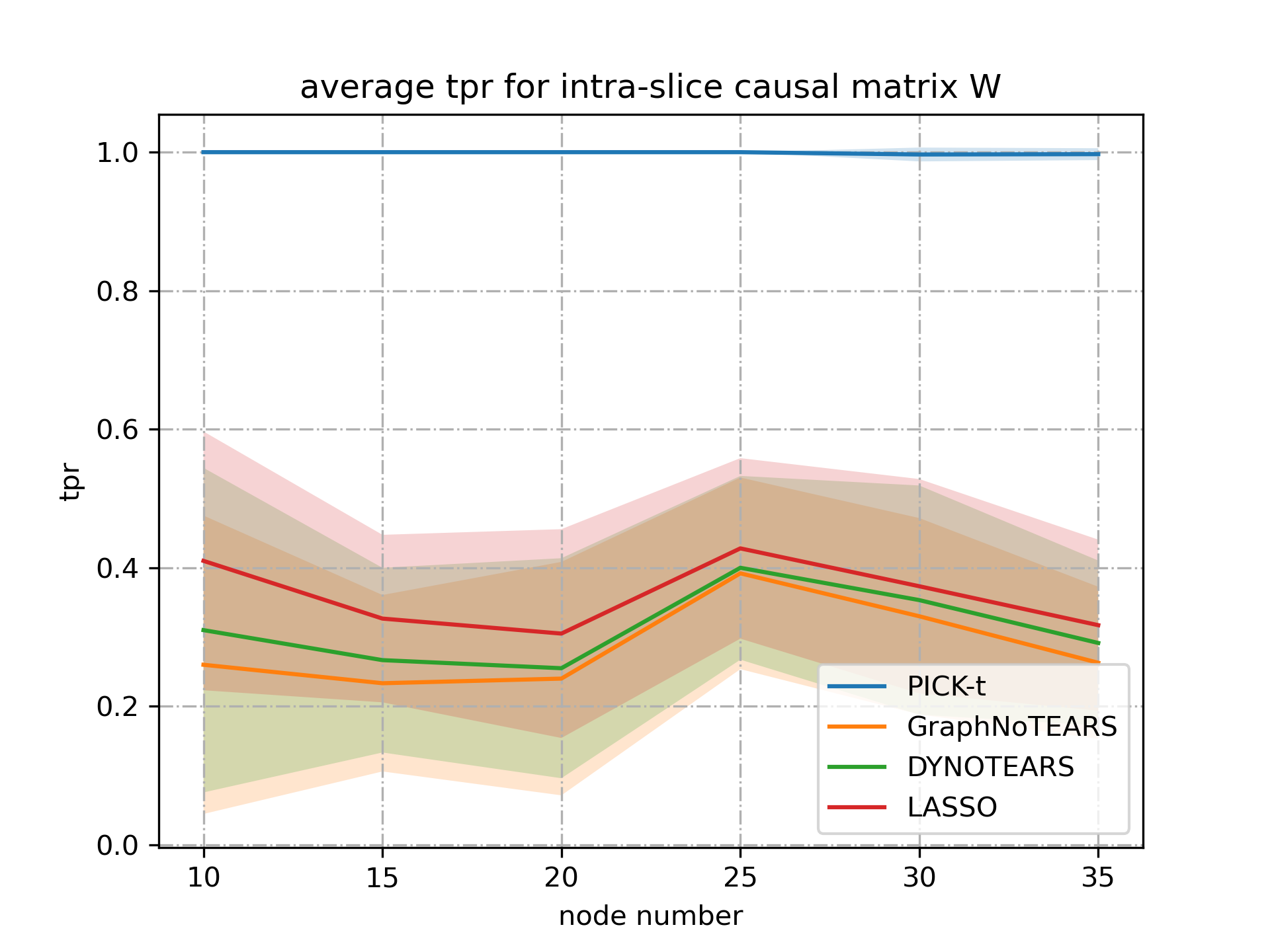}
  }
  \subfloat[ER2]
  {      \label{fig:dy-sin-tpr-w-subfig2}\includegraphics[width=0.3\linewidth]{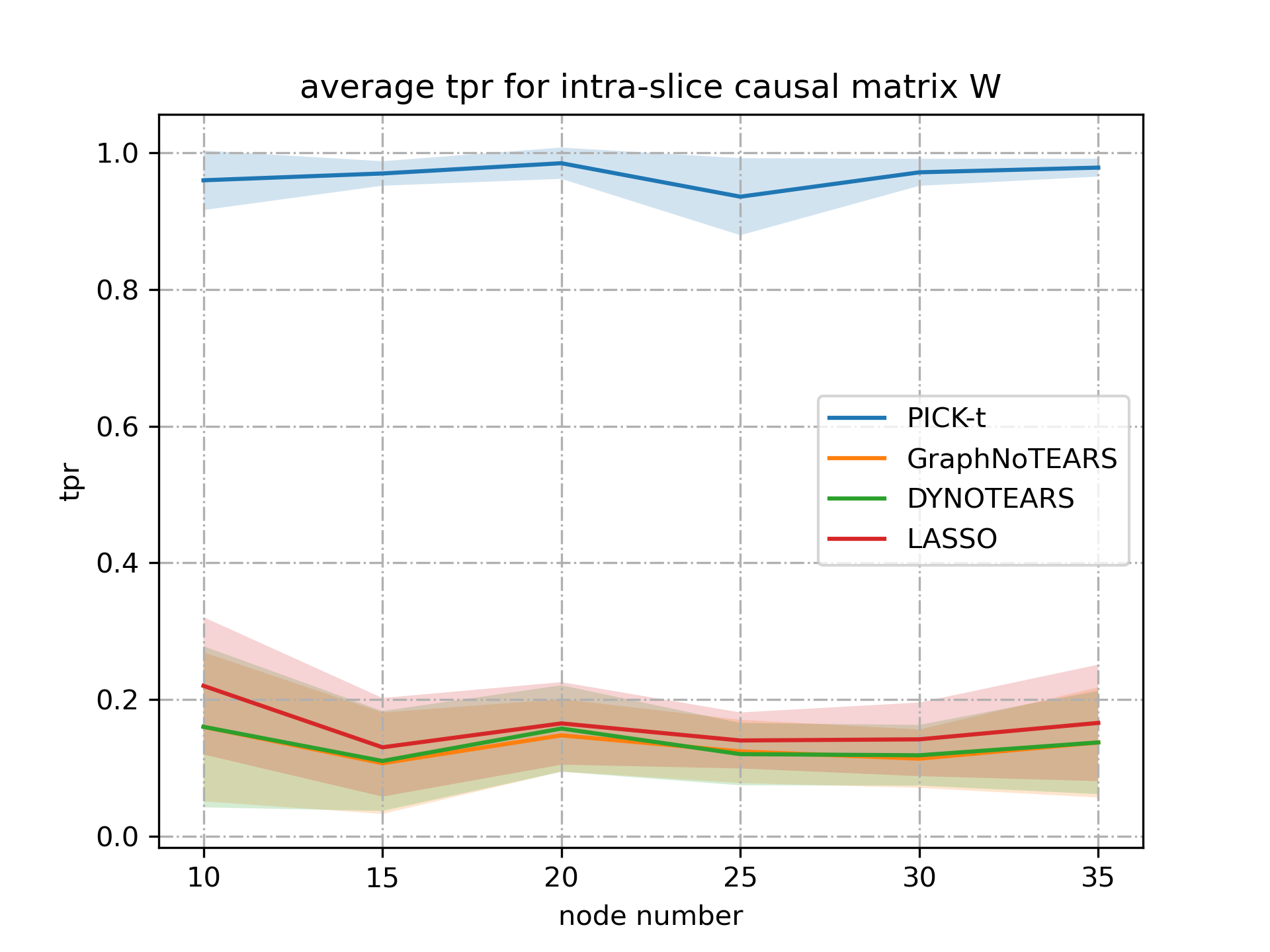}
  }
    \subfloat[ER4]
  {      \label{fig:dy-sin-tpr-w-subfig3}\includegraphics[width=0.3\linewidth]{new_figures/temporal/fdr-for-W-p=1-s0=4d-noisetype=sin-adddag=ERlag_type=ERd=10-35.png}
  }
    \caption{TPR for predicted inter-snapshot causal graph and ground truth intra-snapshot causal graph with link function $f_{i}^{(t)}(x_i)=\sum\limits_{j\in\pa(i)}\sin{x_j}$.}   
  \label{fig:dynamic-tpr-intra-snapshot-sin}          
\end{figure}

\begin{figure}[htbp]
  \centering           
  \subfloat[ER1]   
  {      \label{fig:dy-sin-tpr-p-subfig1}\includegraphics[width=0.3\linewidth]{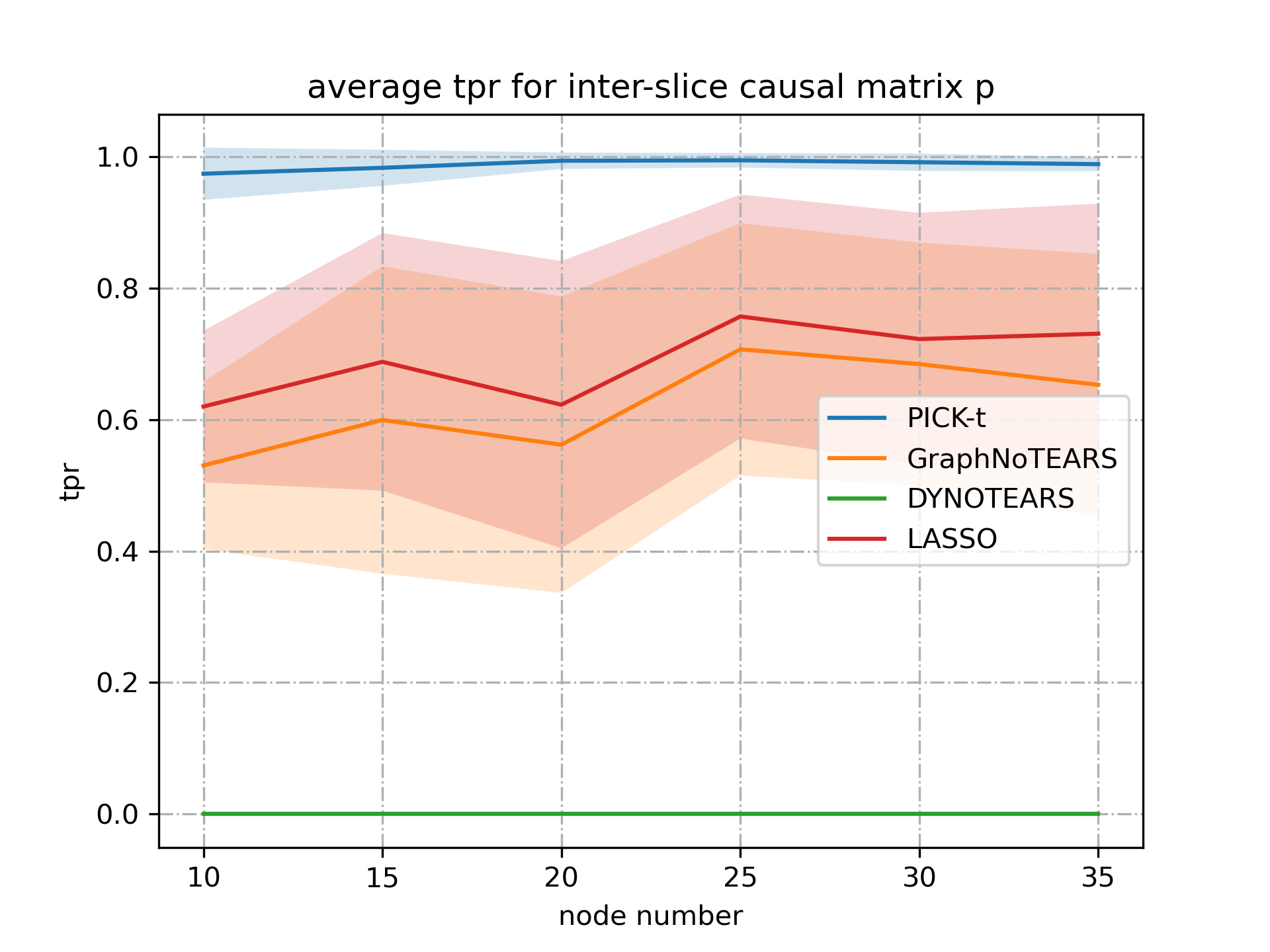}
  }
  \subfloat[ER2]
  {      \label{fig:dy-sin-tpr-p-subfig2}\includegraphics[width=0.3\linewidth]{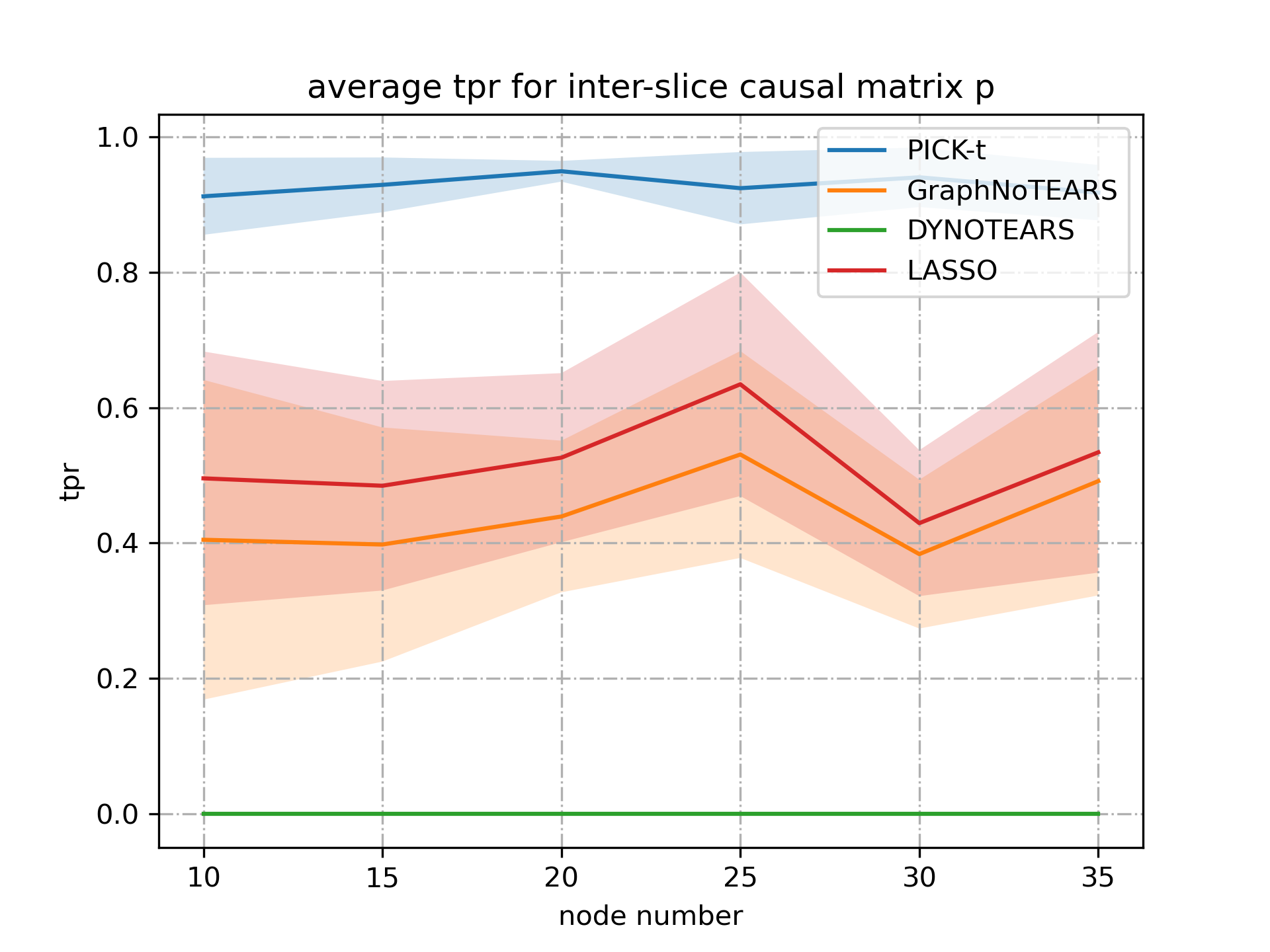}
  }
    \subfloat[ER4]
  {      \label{fig:dy-sin-tpr-p-subfig3}\includegraphics[width=0.3\linewidth]{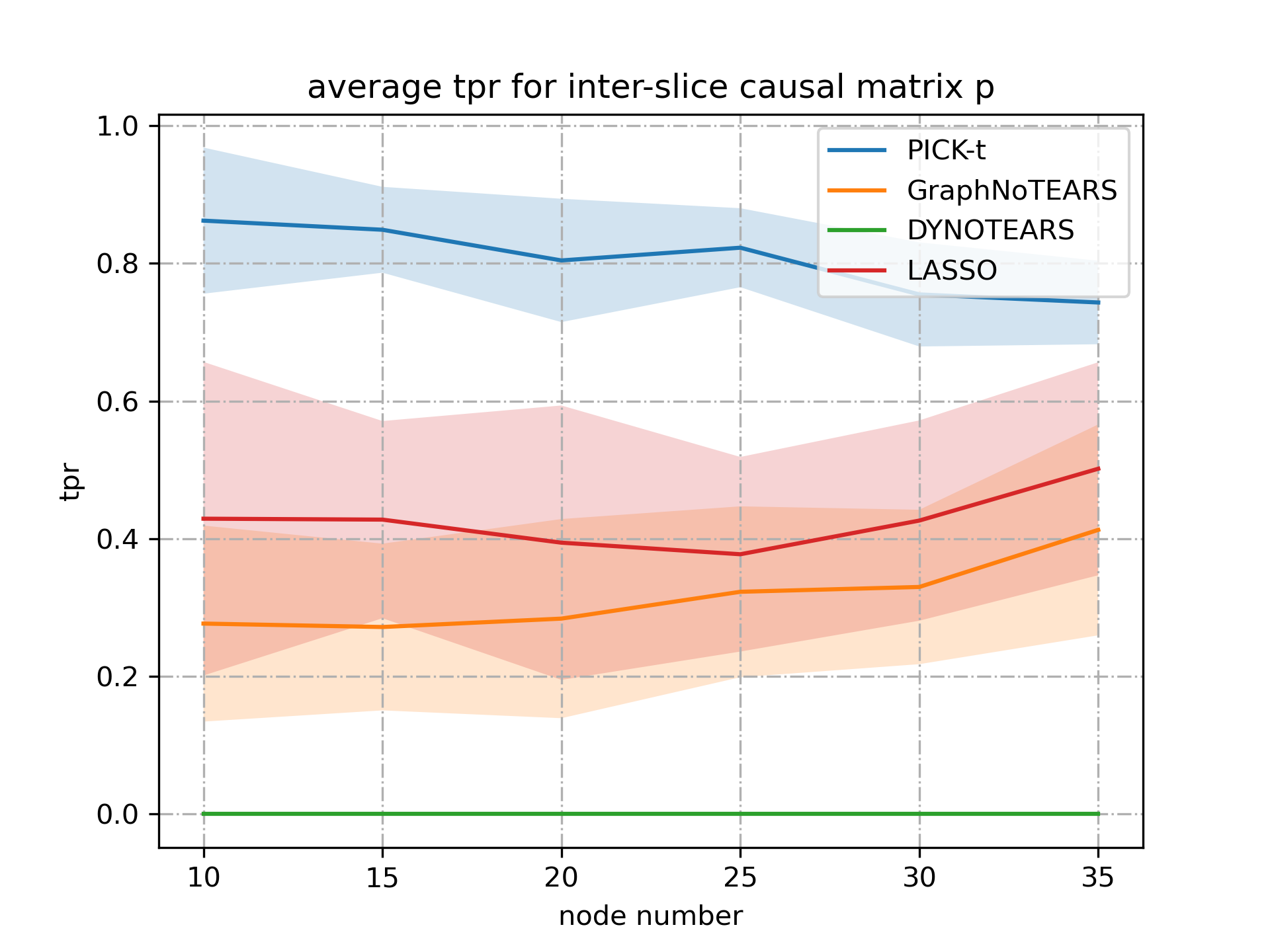}
  }
    \caption{TPR for predicted inter-snapshot causal graph and ground truth inter-snapshot causal graph with link function $f_{i}^{(t)}(x_i)=\sum\limits_{j\in\pa(i)}\sin{x_j}$.}   
  \label{fig:dynamic-tpr-inter-snapshot-sin}          
\end{figure}


\subsubsection{Additional results for synthetic data in different DGP settings}
In this part, we would present the evaluation results for different link function such that $f_{i}^{(t)}$ generated by sampling Gaussian process with a unit bandwidth RBF kernel. Similarly, we would still compare the performance with SHD, FDR and TPR.

\begin{figure}[htbp] 
  \centering           
  \subfloat[ER1]   
  {      \label{fig:dy-gp-shd-w-subfig1}\includegraphics[width=0.3\linewidth]{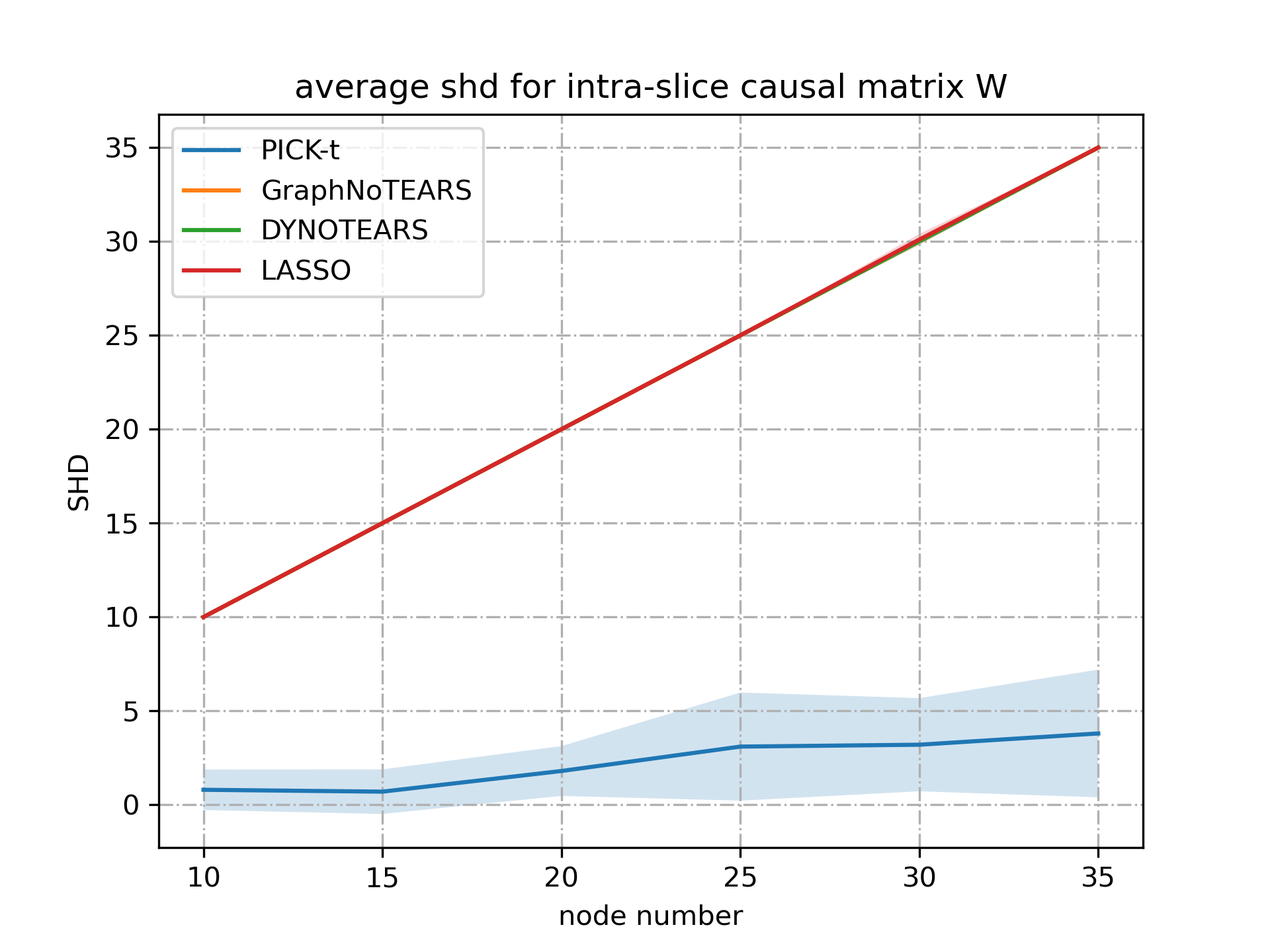}
  }
  \subfloat[ER2]
  {      \label{fig:dy-gp-shd-w-subfig2}\includegraphics[width=0.3\linewidth]{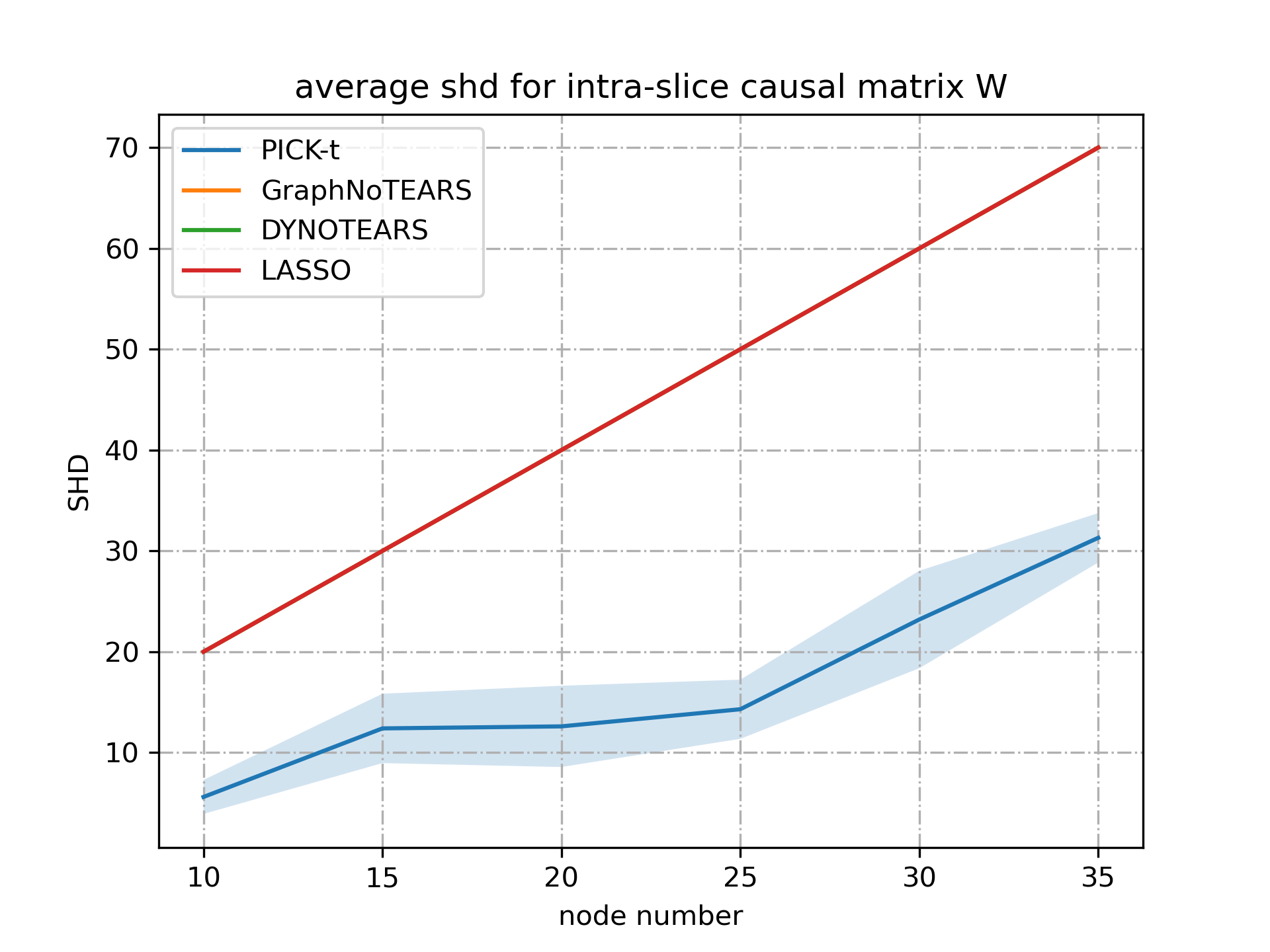}
  }
    \subfloat[ER4]
  {      \label{fig:dy-gp-shd-w-subfig3}\includegraphics[width=0.3\linewidth]{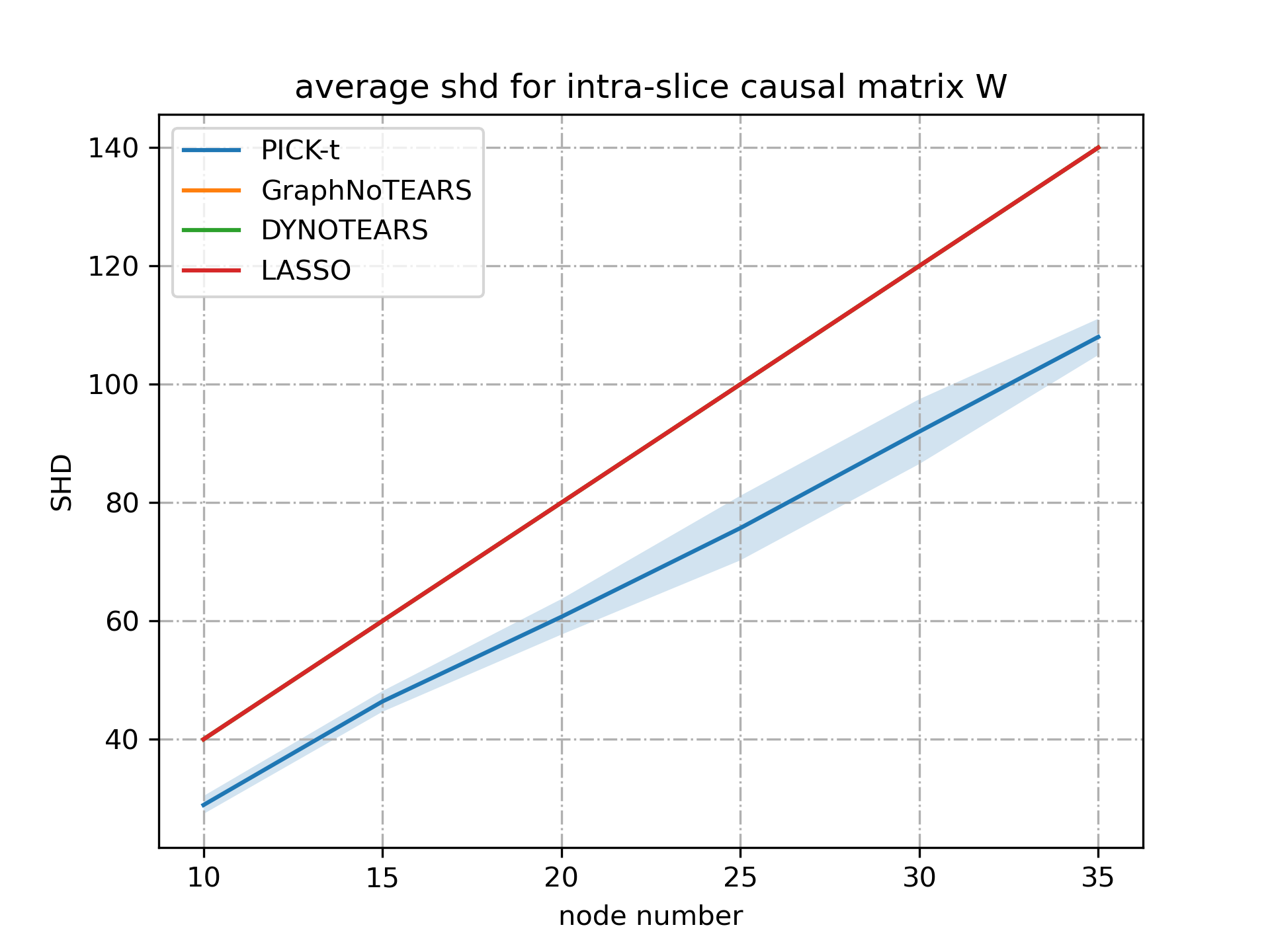}
  }
  \caption{SHD for predicted intra-snapshot causal graph and ground truth intra-snapshot causal graph with link function generated by sampling Gaussian process with a unit bandwidth RBF kernel.}   
  \label{fig:dynamic-shd-intra-snapshot-gp}          
\end{figure}

\begin{figure}[htbp] 
  \centering           
  \subfloat[ER1]   
  {      \label{fig:dy-gp-shd-p-subfig1}\includegraphics[width=0.3\linewidth]{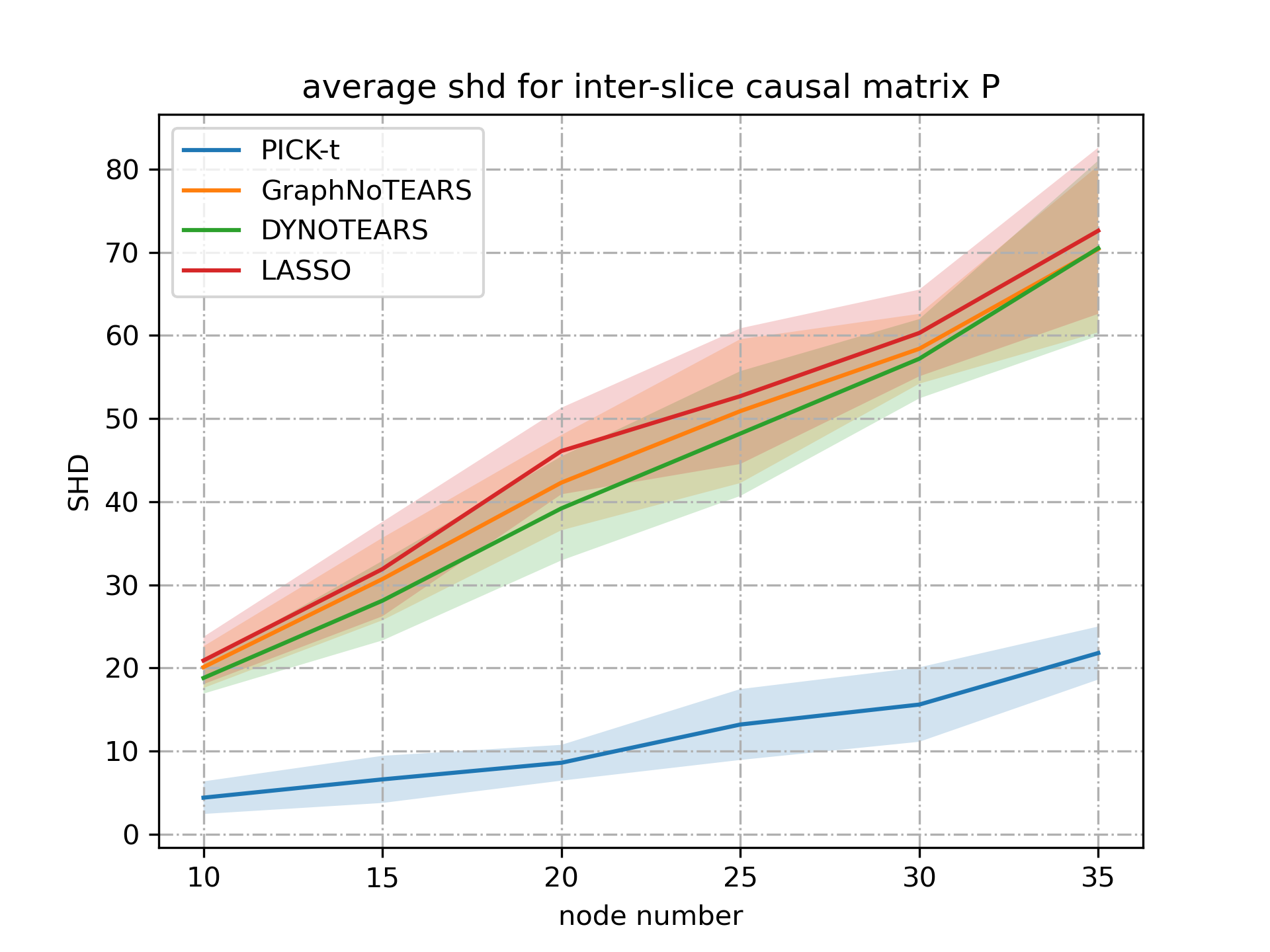}
  }
  \subfloat[ER2]
  {      \label{fig:dy-gp-shd-p-subfig2}\includegraphics[width=0.3\linewidth]{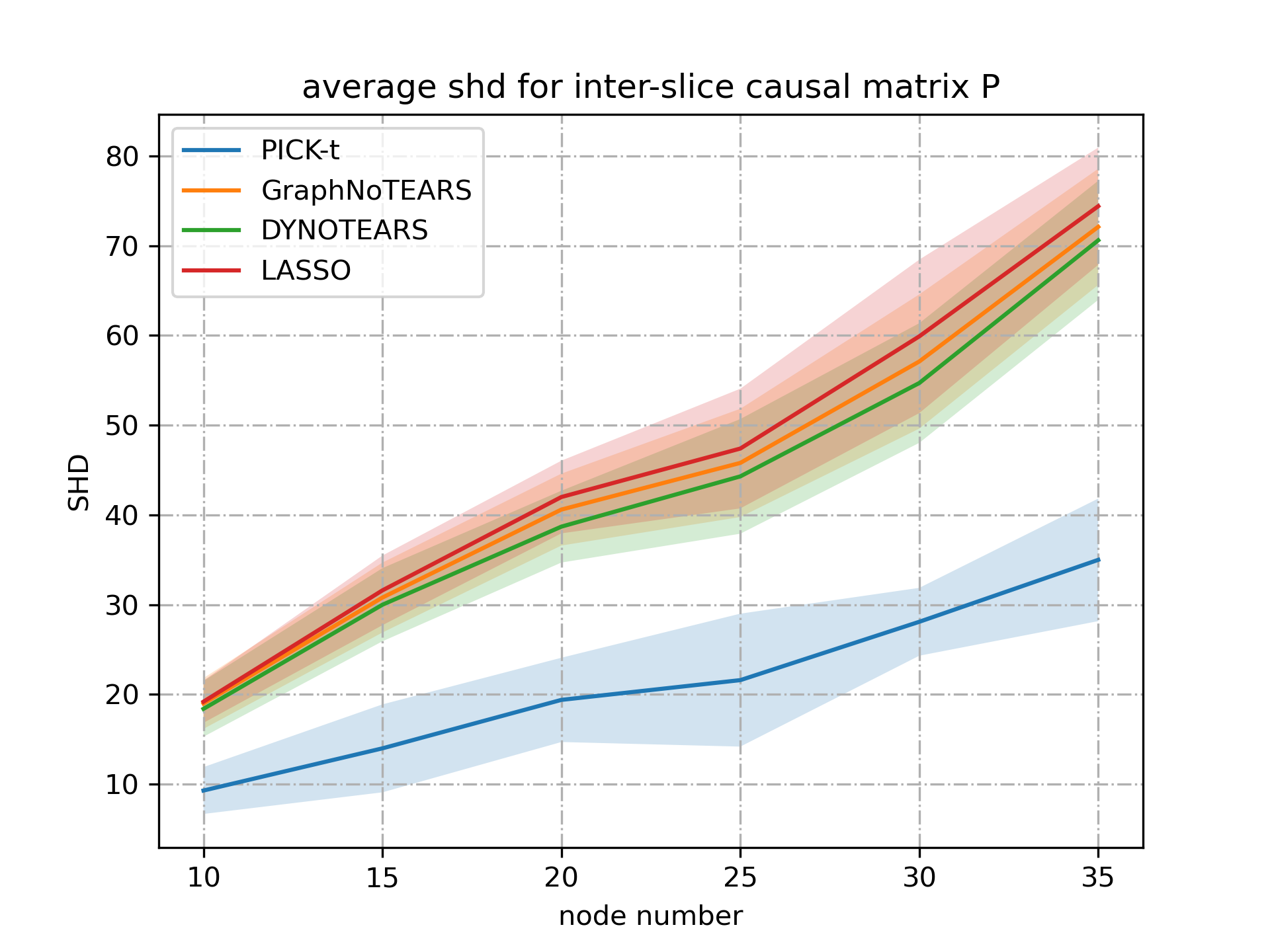}
  }
    \subfloat[ER4]
  {      \label{fig:dy-gp-shd-p-subfig3}\includegraphics[width=0.3\linewidth]{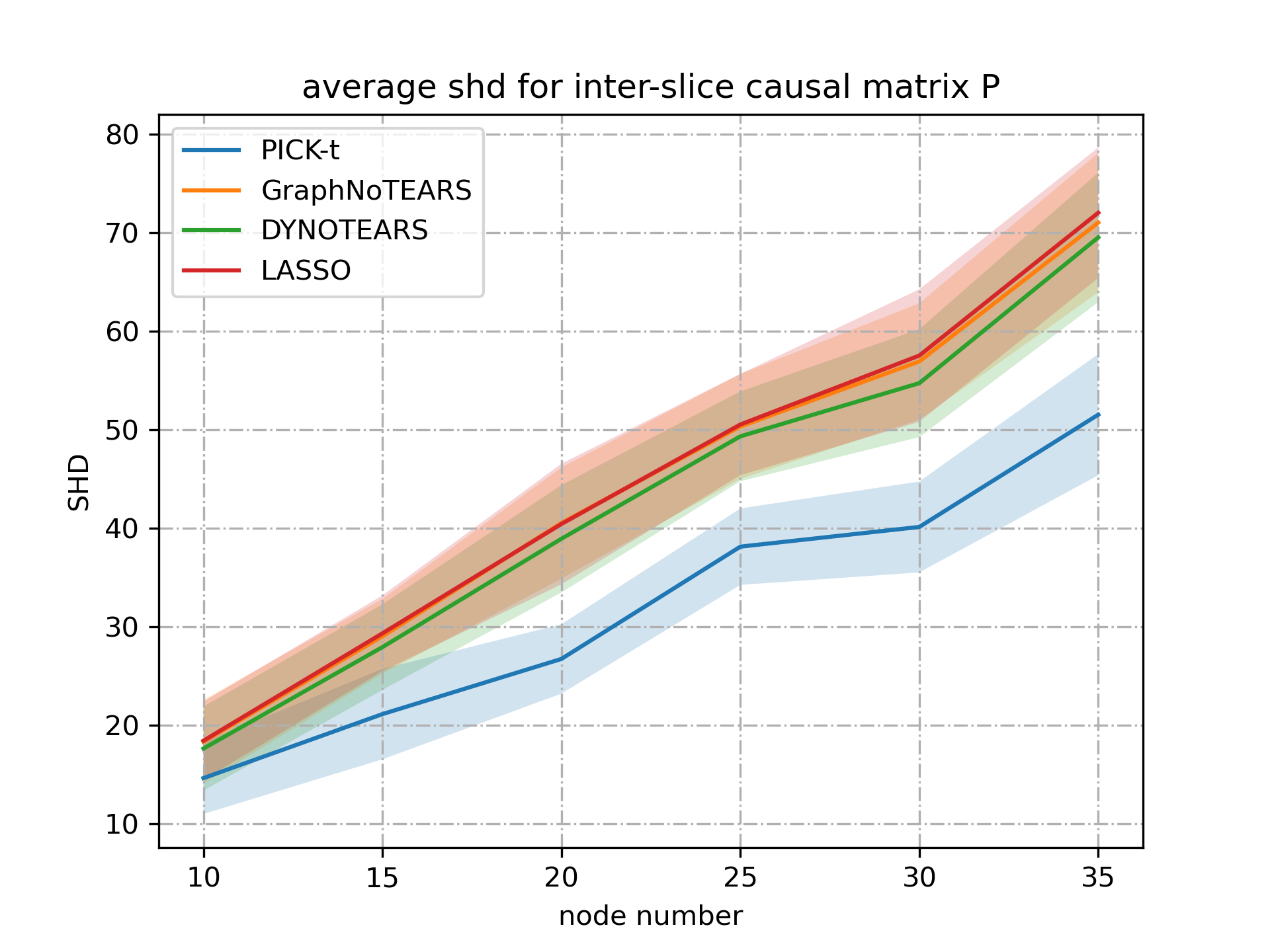}
  }
  \caption{SHD for predicted inter-snapshot causal graph and ground truth inter-snapshot causal graph with link function generated by sampling Gaussian process with a unit bandwidth RBF kernel.}   
  \label{fig:dynamic-shd-inter-snapshot-gp}          
\end{figure}

\begin{figure}[htbp]
  \centering           
  \subfloat[ER1]   
  {      \label{fig:dy-gp-fdr-w-subfig1}\includegraphics[width=0.3\linewidth]{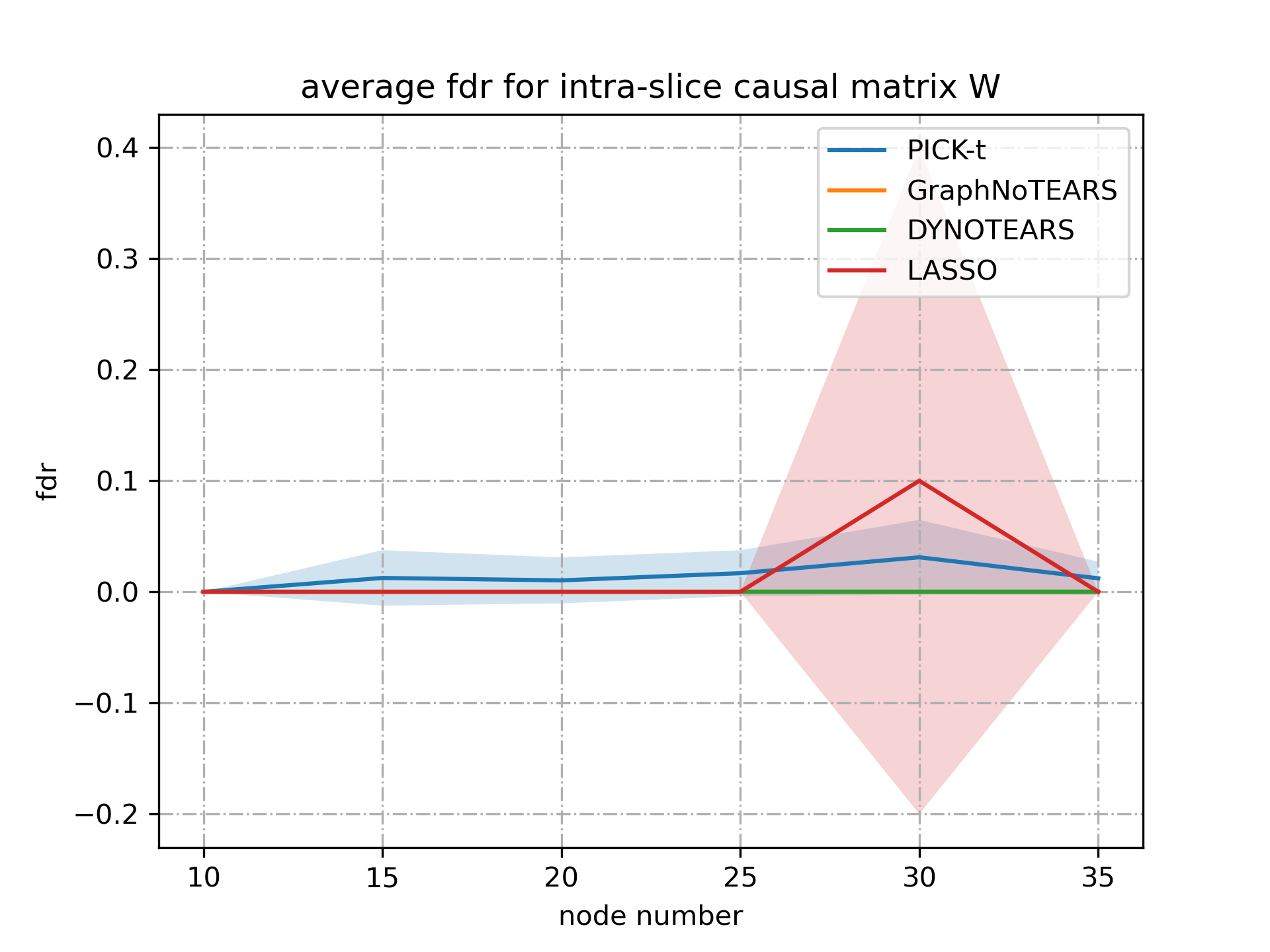}
  }
  \subfloat[ER2]
  {      \label{fig:dy-gp-fdr-w-subfig2}\includegraphics[width=0.3\linewidth]{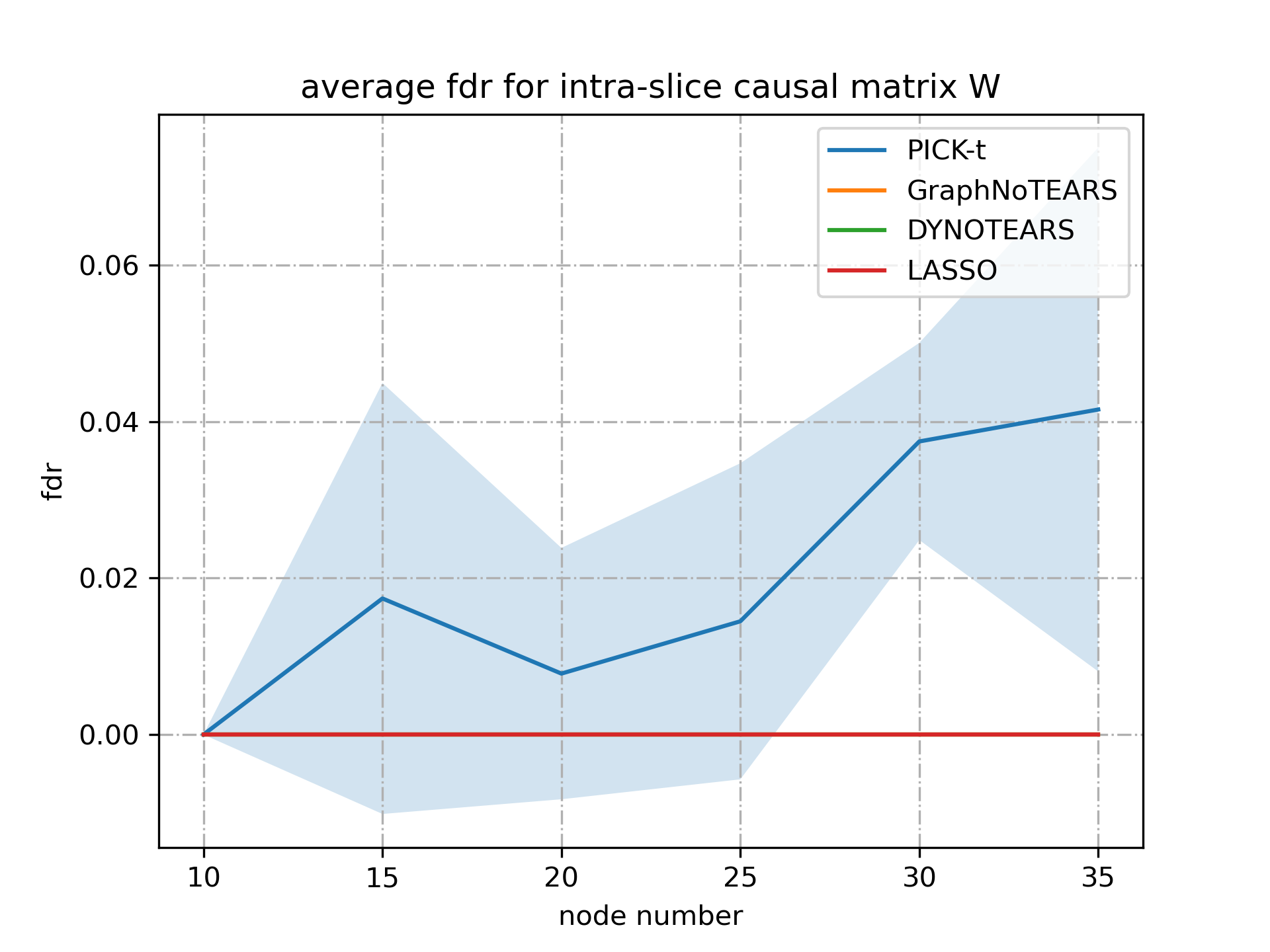}
  }
    \subfloat[ER4]
  {      \label{fig:dy-gp-fdr-w-subfig3}\includegraphics[width=0.3\linewidth]{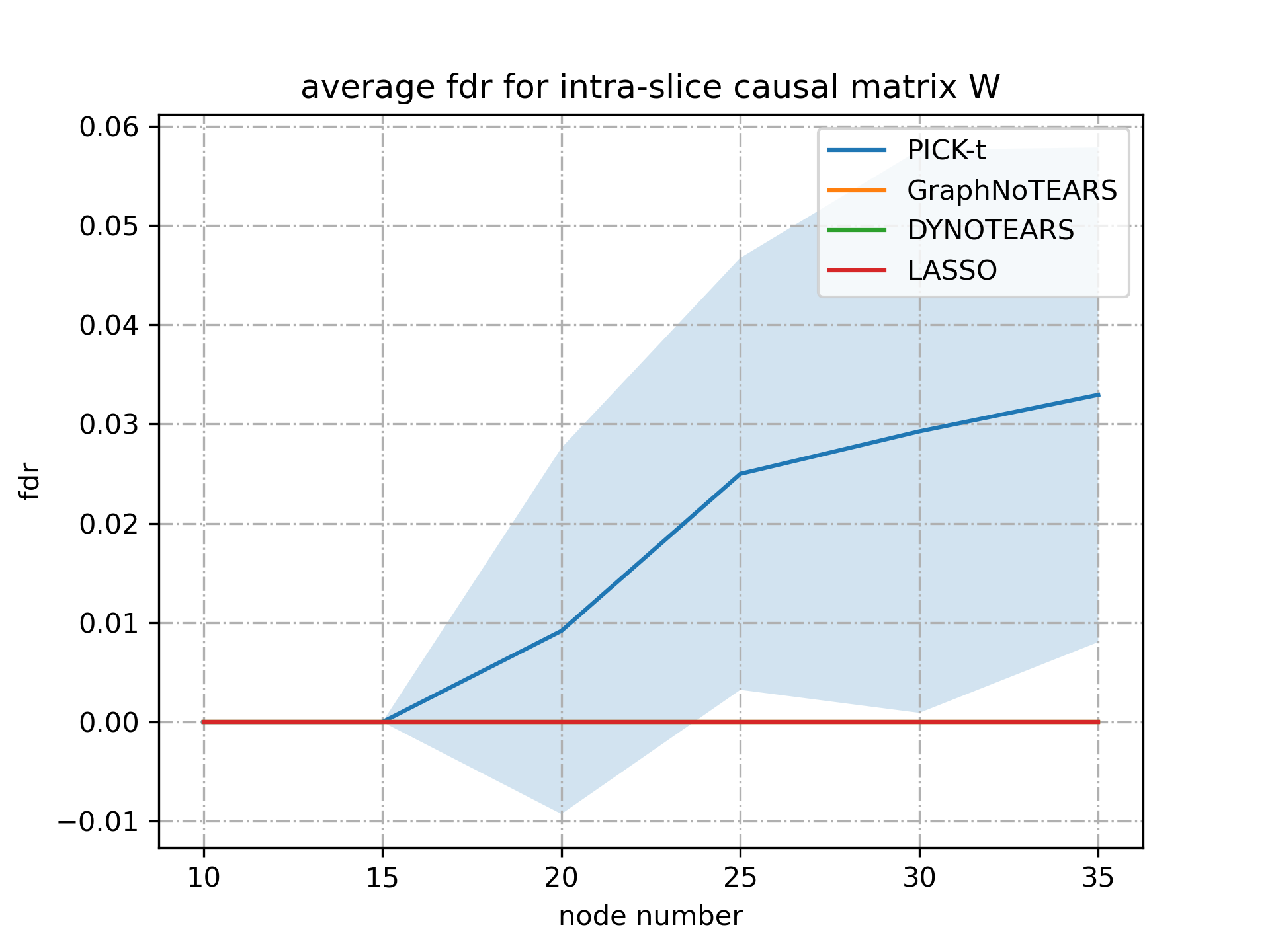}
  }
    \caption{FDR for predicted inter-snapshot causal graph and ground truth inter-snapshot causal graph with link function generated by sampling Gaussian process with a unit bandwidth RBF kernel.}   
  \label{fig:dynamic-fdr-intra-snapshot-gp}          
\end{figure}

\begin{figure}[htbp]
  \centering           
  \subfloat[ER1]   
  {      \label{fig:dy-gp-fdr-p-subfig1}\includegraphics[width=0.3\linewidth]{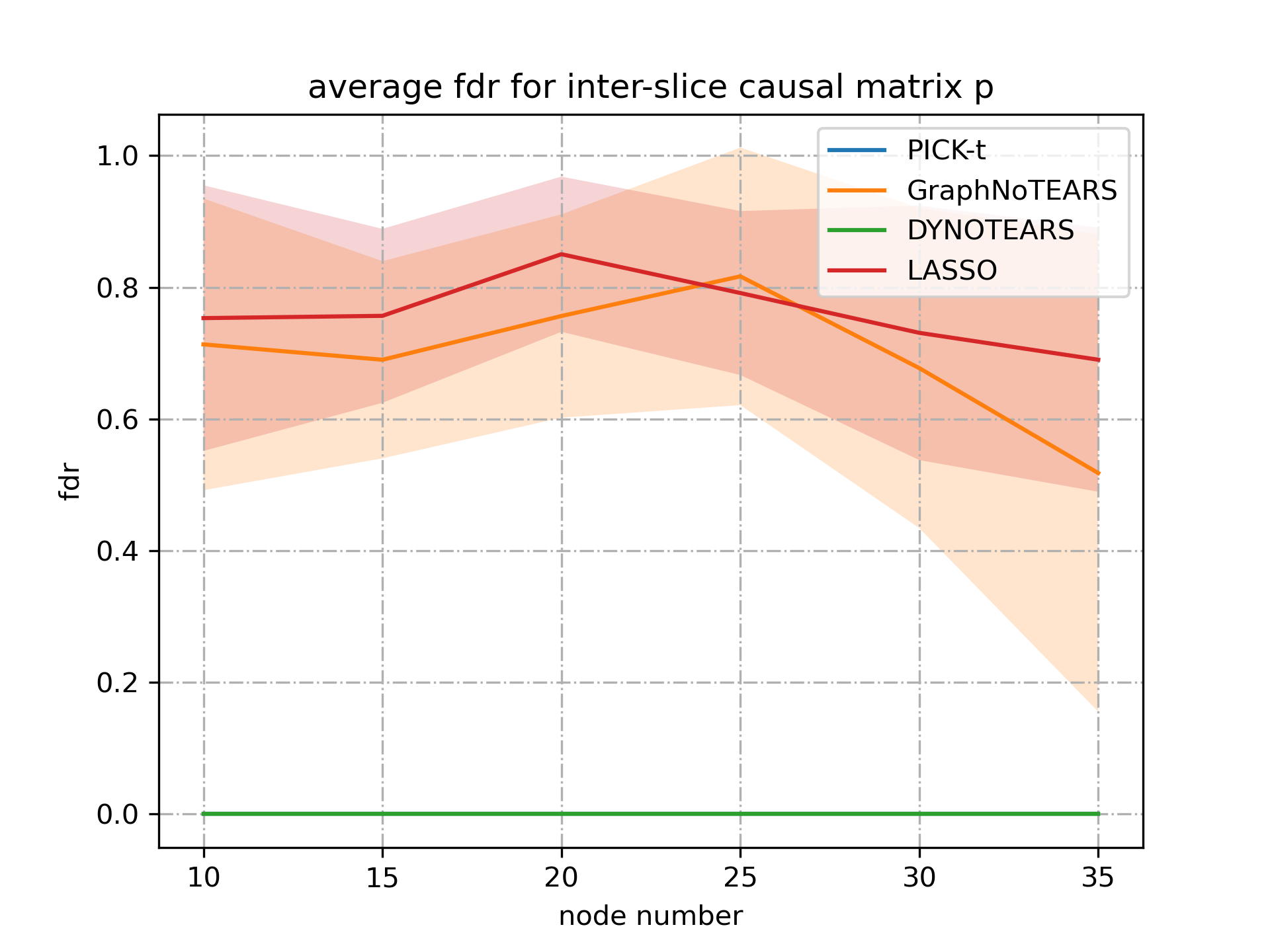}
  }
  \subfloat[ER2]
  {      \label{fig:dy-gp-fdr-p-subfig2}\includegraphics[width=0.3\linewidth]{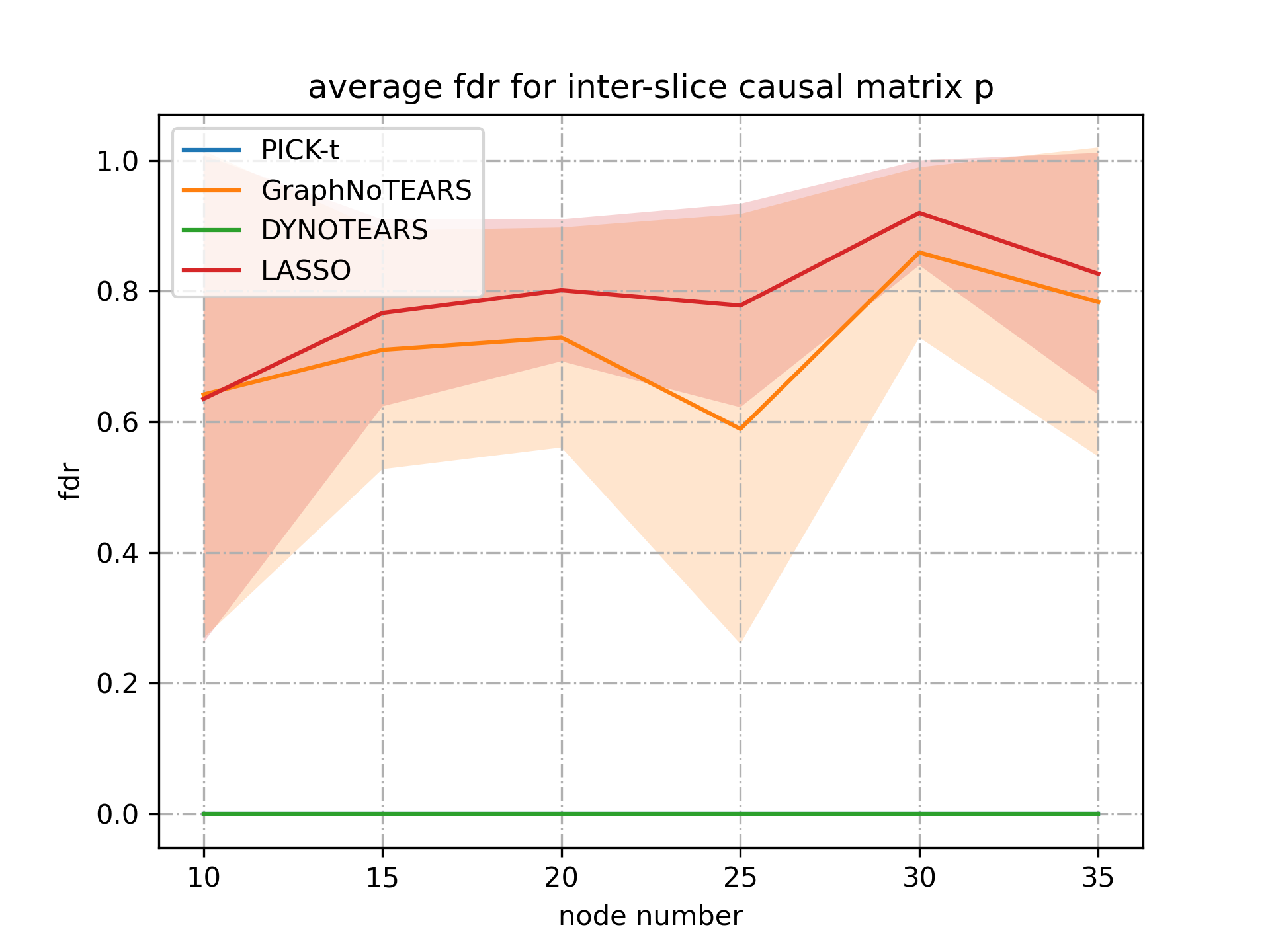}
  }
    \subfloat[ER4]
  {      \label{fig:dy-gp-fdr-p-subfig3}\includegraphics[width=0.3\linewidth]{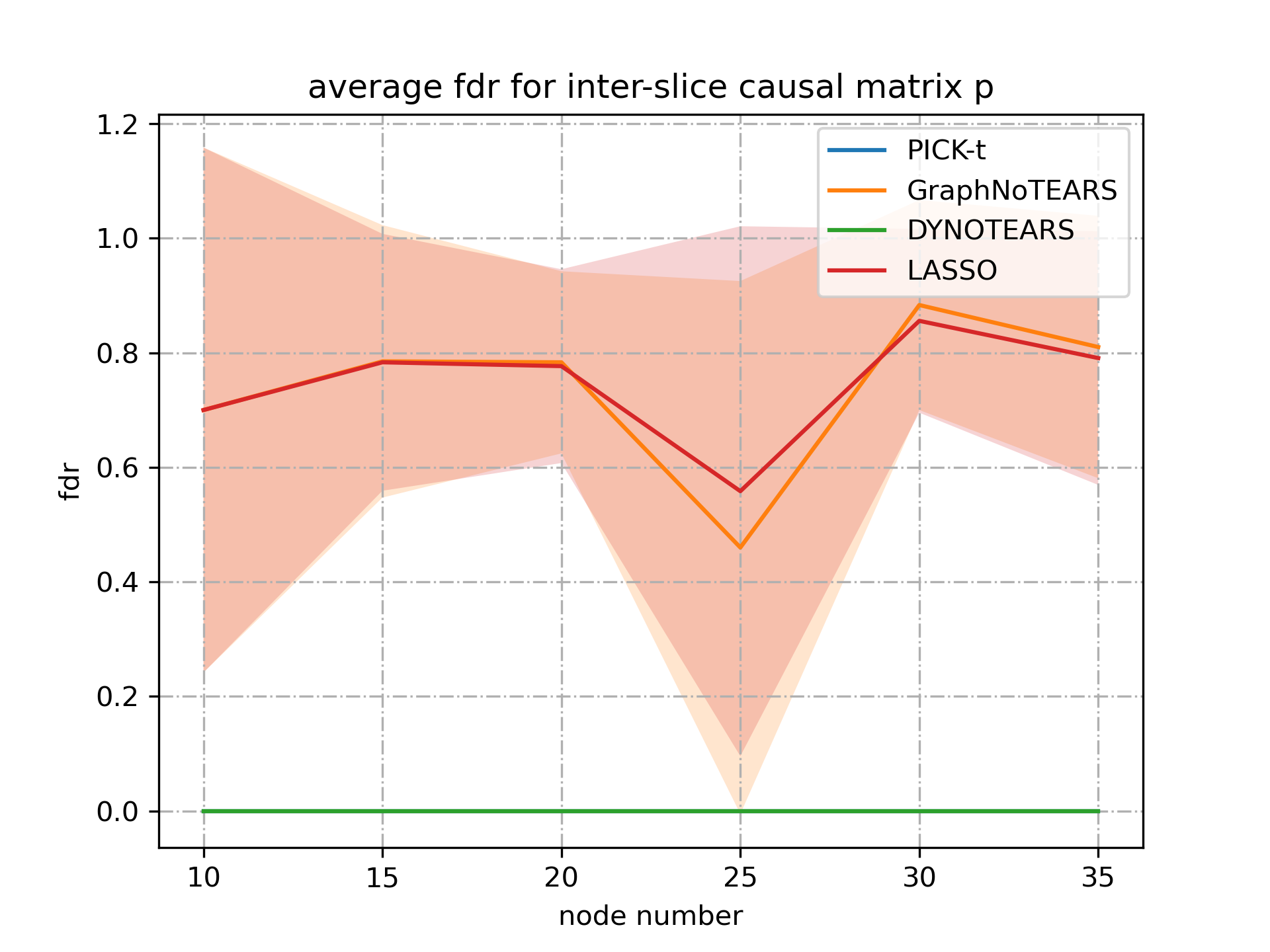}
  }
    \caption{FDR for predicted inter-snapshot causal graph and ground truth inter-snapshot causal graph with link function generated by sampling Gaussian process with a unit bandwidth RBF kernel.}   
  \label{fig:dynamic-fdr-inter-snapshot-gp}          
\end{figure}

\begin{figure}[htbp]
  \centering           
  \subfloat[ER1]   
  {      \label{fig:dy-gp-tpr-w-subfig1}\includegraphics[width=0.3\linewidth]{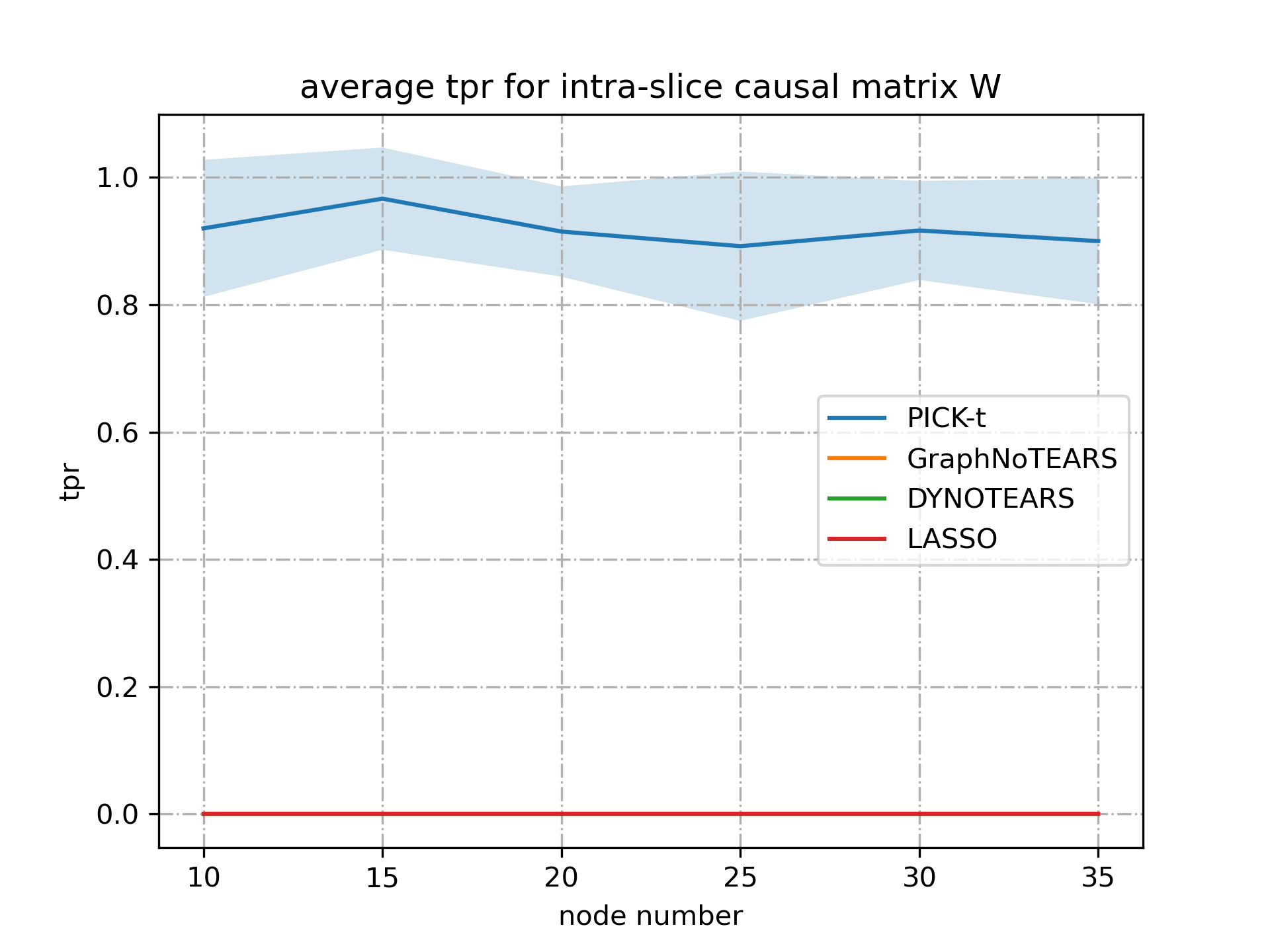}
  }
  \subfloat[ER2]
  {      \label{fig:dy-gp-tpr-w-subfig2}\includegraphics[width=0.3\linewidth]{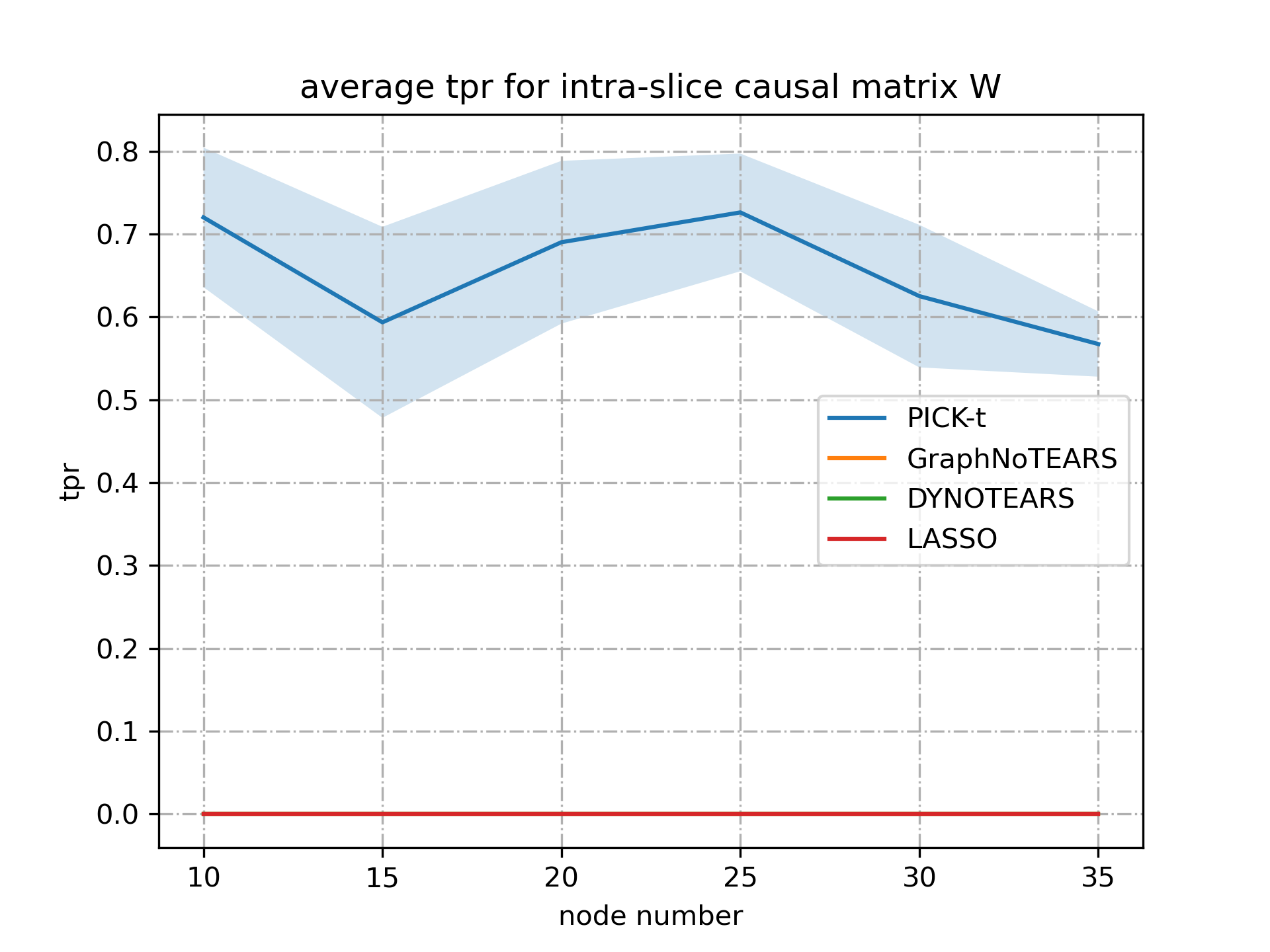}
  }
    \subfloat[ER4]
  {      \label{fig:dy-gp-tpr-w-subfig3}\includegraphics[width=0.3\linewidth]{new_figures/temporal/gp/fdr-for-W-p=1-s0=4d-noisetype=gpdag=ERlag_type=ERd=10-35.png}
  }
    \caption{TPR for predicted inter-snapshot causal graph and ground truth intra-snapshot causal graph with link function generated by sampling Gaussian process with a unit bandwidth RBF kernel.}   
  \label{fig:dynamic-tpr-intra-snapshot-gp}          
\end{figure}

\begin{figure}[htbp]
  \centering           
  \subfloat[ER1]   
  {      \label{fig:dy-gp-tpr-p-subfig1}\includegraphics[width=0.3\linewidth]{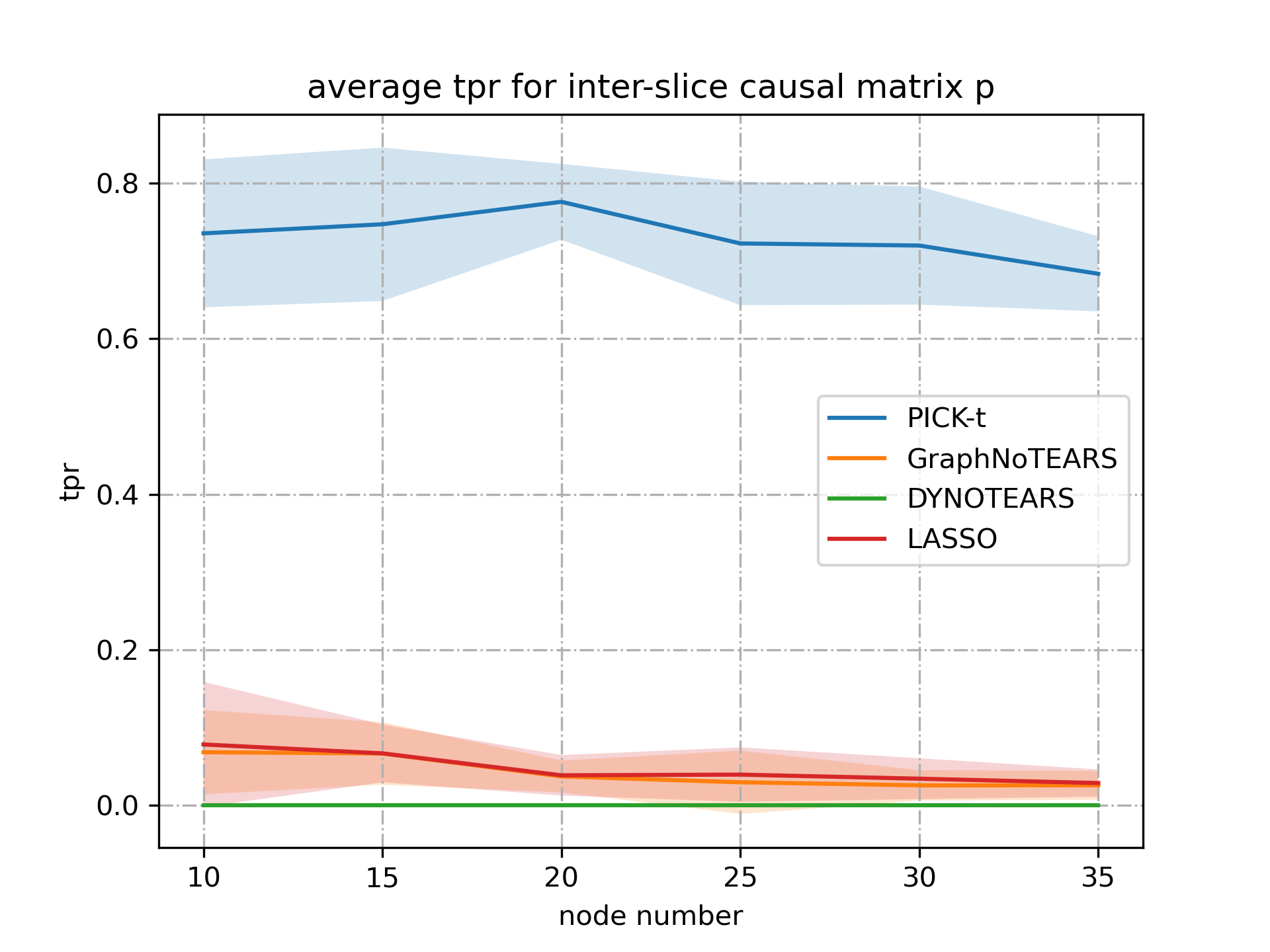}
  }
  \subfloat[ER2]
  {      \label{fig:dy-gp-tpr-p-subfig2}\includegraphics[width=0.3\linewidth]{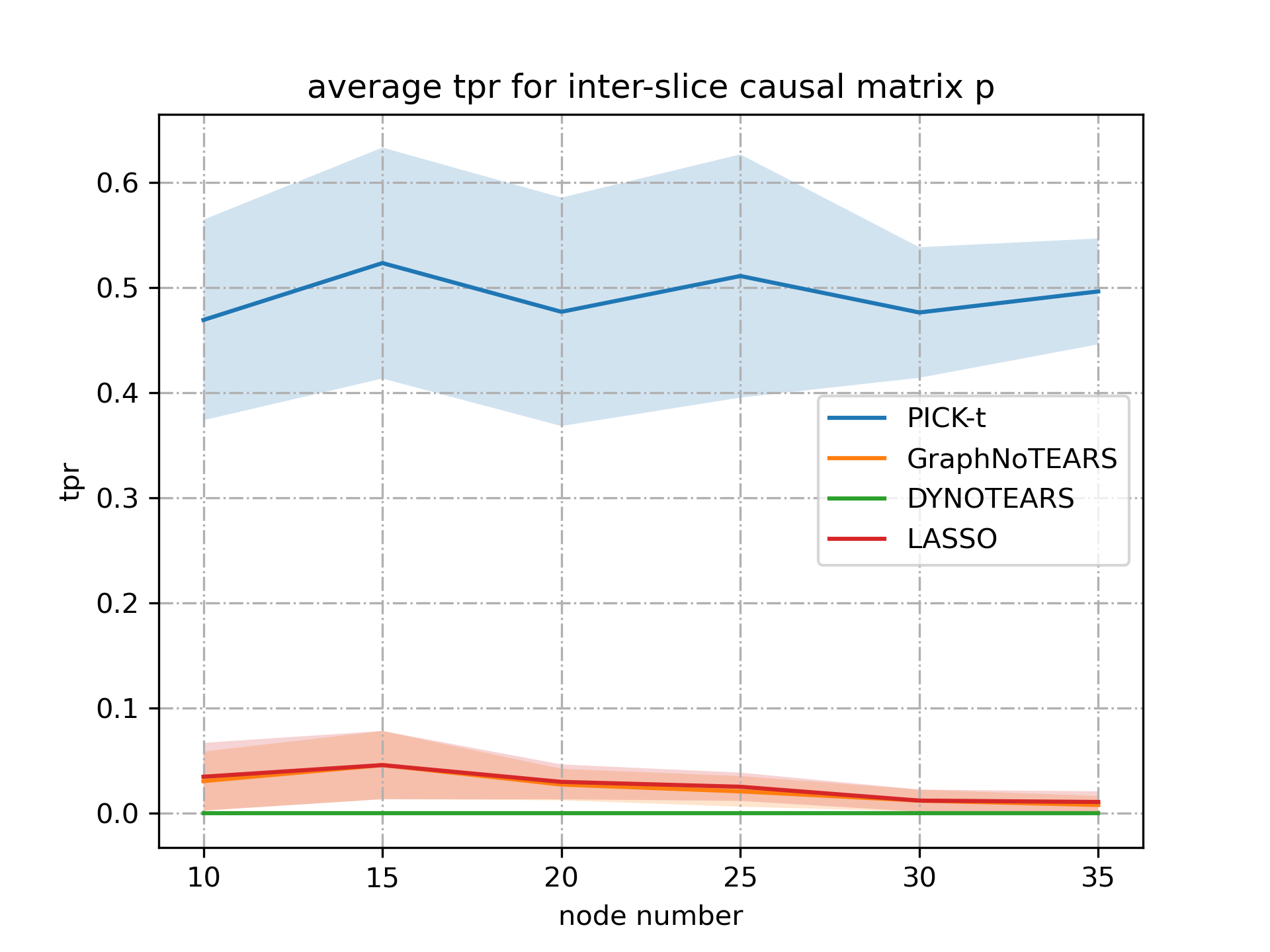}
  }
    \subfloat[ER4]
  {      \label{fig:dy-gp-tpr-p-subfig3}\includegraphics[width=0.3\linewidth]{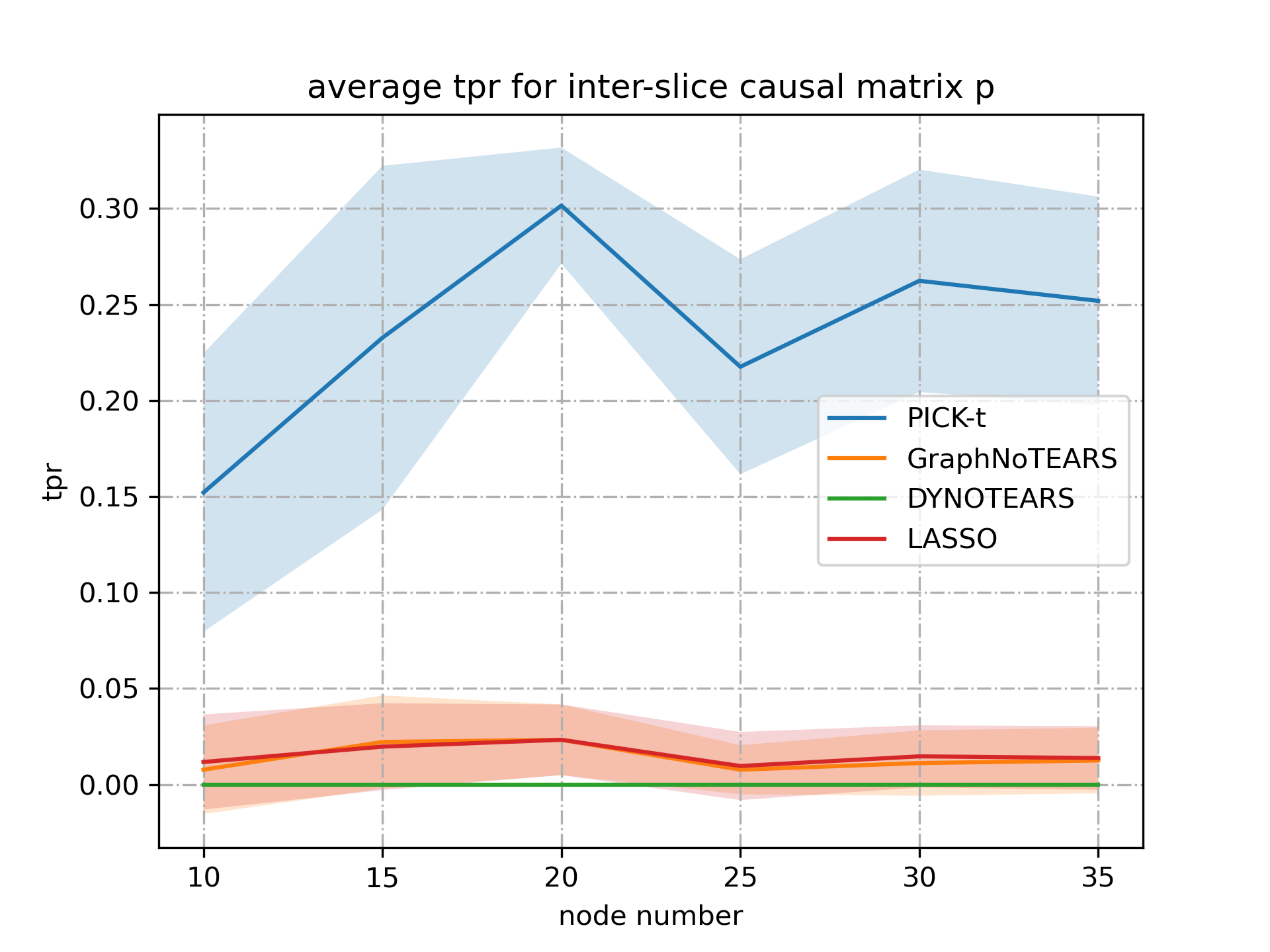}
  }
    \caption{TPR for predicted inter-snapshot causal graph and ground truth inter-snapshot causal graph with link function generated by sampling Gaussian process with a unit bandwidth RBF kernel.}   
  \label{fig:dynamic-tpr-inter-snapshot-gp}          
\end{figure}

\end{document}

%% file: math_commands.tex
\usepackage{amsmath,amsfonts,bm}

\def\eqref#1{equation~\ref{#1}}

\def\1{\bm{1}}

\DeclareMathAlphabet{\mathsfit}{\encodingdefault}{\sfdefault}{m}{sl}
\SetMathAlphabet{\mathsfit}{bold}{\encodingdefault}{\sfdefault}{bx}{n}

\newcommand{\Var}{\mathrm{Var}}

\def\bbR{\mathbb{R}}

\def\bfO{\mathbf{O}}

\def\pa{\mathsf{pa}}
\def\ch{\mathsf{ch}}

\def\hat{\widehat}
\def\tilde{\widetilde}
